\documentclass{article} 
\usepackage{iclr2027_conference,times}

\usepackage{amsmath,amsfonts,bm}

\def\eqref#1{equation~\ref{#1}}

\def\1{\bm{1}}

\DeclareMathAlphabet{\mathsfit}{\encodingdefault}{\sfdefault}{m}{sl}
\SetMathAlphabet{\mathsfit}{bold}{\encodingdefault}{\sfdefault}{bx}{n}

\newcommand{\E}{\mathbb{E}}

\newcommand{\R}{\mathbb{R}}

\newcommand{\KL}{D_{\mathrm{KL}}}

\usepackage{titletoc}
\usepackage{amsmath, amssymb, amsfonts, amsthm}
\usepackage{nicefrac, xfrac}
\usepackage{thmtools, thm-restate}
\usepackage[dvipsnames,svgnames]{xcolor}
\usepackage{graphicx}
\usepackage{caption, subcaption}
\usepackage{wrapfig}
\usepackage{booktabs, longtable, multirow, makecell, tabularx, tabularray}
\usepackage{tikz}
\usetikzlibrary{calc,decorations.pathreplacing,arrows.meta}
\usepackage{algorithm}
\usepackage{float}
\usepackage{algpseudocode}
\usepackage{enumitem}
\usepackage{bbm}
\usepackage[most]{tcolorbox}
\definecolor{DarkKlein}{HTML}{122A82}
\definecolor{LinkBurgundy}{HTML}{8A1538}
\usepackage[
    colorlinks=true,
    linkcolor=LinkBurgundy,
    citecolor=DarkKlein,
    urlcolor=Black
]{hyperref}
\usepackage{xurl}
\usepackage[capitalize]{cleveref}
\usepackage[normalem]{ulem}

\definecolor{accentline}{HTML}{436998}
\definecolor{boxbg}{HTML}{F2F6FB}
\newtcolorbox{formalbox}[1][]{%
  enhanced,
  breakable,
  pad at break=0pt,                
  sharp corners,
  boxrule=0pt,
  frame hidden,                    
  borderline west={1.75pt}{0pt}{accentline}, 
  colback=boxbg,                   
  left=7pt,
  right=4pt,
  top=3pt,
  bottom=3pt,
  before skip=6pt,
  after skip=6pt,
  #1                               
}

\let\oldcomplement\complement 
\renewcommand{\complement}{\text{\scalebox{0.8}{$\oldcomplement$}}}

\declaretheorem[name=Theorem,numberwithin=section]{theorem}
\declaretheorem[name=Definition,style=definition,numberwithin=section]{definition}
\declaretheorem[name=Proposition,numberwithin=section]{proposition}
\declaretheorem[name=Assumption,numberwithin=section]{assumption}
\declaretheorem[name=Corollary,numberwithin=section]{corollary}
\declaretheorem[name=Lemma,numberwithin=section]{lemma}

\declaretheorem[name=Remark,style=remark,numberwithin=section]{remark}

\crefname{equation}{Eq.}{Eqs.}
\crefname{section}{Sec.}{Secs.}
\crefname{appendix}{Appx.}{Appxs.} 
\crefname{theorem}{Thm.}{Thms.}
\crefname{corollary}{Cor.}{Cors.}
\crefname{proposition}{Prop.}{Props.}
\crefname{assumption}{Asm.}{Asms.}
\crefname{definition}{Defn.}{Defns.}
\crefname{example}{Ex.}{Exs.}
\crefname{figure}{Fig.}{Figs.}
\crefname{table}{Tab.}{Tabs.}
\crefname{algorithm}{Alg.}{Algs.}

\newcommand{\cA}{\mathcal{A}}   
\newcommand{\cX}{\mathcal{X}}
\newcommand{\cQ}{\mathcal{Q}}   

\newcommand{\cE}{\mathcal{E}}
\newcommand{\cZ}{\mathcal{Z}}

\newcommand{\Pp}{\operatorname{\mathbb{P}}}    

\newcommand{\DTV}{D_{\mathrm{TV}}}     
\newcommand{\TV}{\mathrm{TV}}          
\newcommand{\AT}{\mathrm{AT}}          
\newcommand{\supp}{\operatorname{supp}}

\newcommand{\modl}{\mathfrak{A}}       
\newcommand{\dd}{\mathop{}\!\mathrm{d}}
\newcommand{\diam}{\operatorname{diam}}
\newcommand{\kl}{\mathrm{kl}}          
\newcommand{\logit}{\operatorname{\mathrm{logit}}}
\newcommand{\mmse}{\operatorname{\mathrm{mmse}}}
\newcommand{\Ent}{\operatorname{\mathrm{Ent}}}
\newcommand{\indic}{\mathbbm{1}}

\title{An Identifiability Theory of Masked Prediction: Mode Blindness and Mask Schedules}

\author{Yichao Cai\textsuperscript{*} \;\&\; Javen Qinfeng Shi\textsuperscript{*, \dag} \\
\textsuperscript{*}~Australian Institute for Machine Learning, Adelaide University\\
\textsuperscript{\dag}~Responsible AI Research Centre\\
\texttt{\{yichao.cai,javen.shi\}@adelaide.edu.au}
}

\iclrfinalcopy 
\begin{document}

\maketitle

\begin{abstract}
Masked prediction learns to infer missing variables from visible context selected by a mask schedule. When does small excess risk guarantee recovery of the true joint distribution? We study this question on finite product spaces under masked-block log loss, with conditionals induced by a single joint distribution. To quantify recovery, we introduce an $\varepsilon$-identifiability modulus measuring the worst-case error among joint distributions with population excess risk at most $\varepsilon$. For data with separated modes pinned down by sufficiently large visible contexts, schedules that always retain such contexts can permit nonvanishing mode-weight errors while incurring exponentially small excess risk. An information decomposition explains why: the loss captures mode-weight mismatch only where the visible context leaves the mode uncertain. We prove two-sided bounds showing that, over a fixed range of mode weights, sensitivity to mode reweighting is governed by residual mode uncertainty averaged over the mask schedule. Assigning schedule mass to low-visibility masks that retain this uncertainty yields recovery bounds within the reweighting family. Beyond this family, positive full-mask probability characterizes uniform control of joint KL divergence by excess risk. Exact calculations and controlled gradient optimization validate these predictions.
\end{abstract}

\section{Introduction}
\label{sec:introduction}

Masked prediction replaces direct joint modeling by predicting missing variables from partial observations, with a \emph{mask schedule} governing the distribution of revealed coordinates. This training paradigm underlies BERT-style pretraining \citep{devlin2019bert,liu2019roberta}, masked autoencoding \citep{he2022masked}, absorbing-state discrete diffusion \citep{austin2021structured,lou2024discrete,sahoo2024simple,ou2025absorbing}, and masked latent prediction \citep{assran2023self,bardes2024revisiting,nam2026cjepa}. Despite architectural differences, these methods raise a fundamental question: when does small masked-prediction excess risk guarantee recovery of the true joint distribution?

A simple thought experiment shows this cannot hold unconditionally. Suppose the data contain two well-separated regimes, such as code and prose, and we perturb only their relative mixture weights while keeping distributions fixed within each regime. With little visible context, a model must hedge its predictions across regimes, scaling completion probabilities directly by global mode weights. Once the context pins down the active regime, however, that regime's global weight has little effect on the masked conditional. Increasing visible context thus makes conditional inference easier while eroding the signal required to recover global weights. Hence, a model can achieve near-optimal loss while incurring substantial errors in mode weights---a pathology we term \emph{mode blindness}.

To quantify this gap, we introduce an \emph{$\varepsilon$-identifiability modulus}, which captures the worst-case recovery error over joint distributions with excess risk at most $\varepsilon$. Specifically, the modulus vanishes at $\varepsilon=0$ when the objective identifies the true distribution; for $\varepsilon>0$, it measures the stability of recovery. Prior work derives related recovery guarantees via functional inequalities, translating training objective discrepancies into joint-distribution error. \citet{koehler2023statistical} analyze score matching and pseudolikelihood from this perspective, showing that score matching fails to recover mixture weights. For masked prediction, \citet{li2024promises} show that larger masked blocks improve asymptotic statistical efficiency and derive joint recovery bounds through block approximate tensorization. Complementing these upper bounds, we show how the mask schedule governs the excess-risk scale at which nonvanishing mode-weight errors can persist.

We characterize this schedule dependence under masked-block log loss, with model conditionals induced by a single joint distribution. We focus on the mode-reweighting family, which changes only the mixture weights while keeping the distribution within each mode fixed. By the KL chain rule, the joint mismatch separates into a marginal term for the visible variables and a conditional term for the masked variables; only the latter contributes to the learning objective. The strength of this signal is governed by the residual mode uncertainty, which is defined as the expected posterior variance of the mode indicator. We prove that the excess risk is bounded above and below by the product of the mode-weight KL divergence and the residual mode uncertainty averaged over the mask schedule. These bounds hold for arbitrary mask schedules and uniformly over a compact interval of mode weights containing the truth, with constants depending only on that interval.

\begin{figure}[t]
    \centering
    \includegraphics[width=0.95\linewidth]{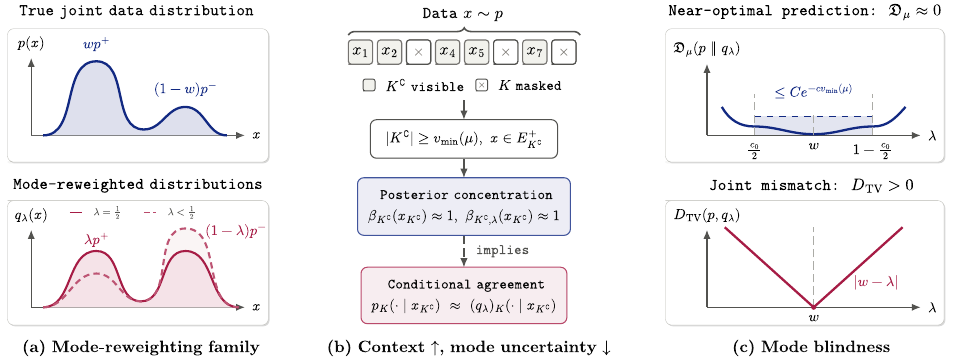}
    \caption{Mode blindness in masked prediction. (a) The true data distribution $p$ and a mode-reweighting distribution $q_\lambda$ share mode conditionals but differ in their global mode weights. (b) Once the visible context leaves little uncertainty about the mode, the masked conditionals under $p$ and $q_\lambda$ become nearly indistinguishable. (c) The masked discrepancy can therefore decay exponentially with visible-context size while the joint distributions remain a constant total-variation distance apart. See \Cref{thm:mode-blind} for the statement and \Cref{app:M-modes} for the extension to $M\ge2$ modes.}
\label{fig:mode-blindness}
\end{figure}

Under a mode-pinning condition, sufficiently large visible contexts leave exponentially small residual mode uncertainty. Mask schedules that always retain such large contexts admit nonvanishing mode-weight errors at excess risk decaying exponentially with the minimum visible-context size (\Cref{fig:mode-blindness}). This forces the associated approximate-tensorization constants to grow at least exponentially. Conversely, when low-visibility masks---those leaving few coordinates visible---retain nontrivial mode uncertainty, assigning them positive schedule mass yields explicit mode-weight recovery bounds within the mode-reweighting family. Beyond this family, full masking directly scores the joint distribution, and positive full-mask probability characterizes uniform control of joint KL divergence by excess risk. Mask schedules thus govern both prediction difficulty and the stability of joint-distribution recovery. Our main contributions are: 
\begin{enumerate}[label=(\roman*), topsep=0pt, itemsep=0pt]
\item We introduce an $\varepsilon$-identifiability modulus that quantifies the worst-case distribution error at an excess-risk tolerance $\varepsilon$, distinguishing exact identification from stable recovery (\Cref{def:modulus}).

\item We establish exponential mode blindness under mode pinning for schedules retaining only large visible contexts (\Cref{thm:mode-blind}), yielding exponential lower bounds on block approximate-tensorization constants in the bimodal setting (\Cref{cor:at-blowup}).

\item We prove two-sided mode-weight sensitivity bounds governed by schedule-averaged residual mode uncertainty, yielding low-visibility recovery within the bimodal reweighting family (\Cref{thm:mode-sensitivity}). Positive full-mask probability further characterizes uniform joint-KL control under masked-block log loss (\Cref{prop:full-mask}).

\item We validate mode blindness, its recovery under low-visibility masking, and predictions for recovery speed through exact calculations and controlled gradient descent within the mode-reweighting family (\Cref{sec:experiments}).
\end{enumerate}

\section{Related Work}
\label{app:related-work}

\paragraph{Masked-pretraining objectives and theory.}
Estimating a joint distribution through conditional likelihoods underlies pseudolikelihood and conditional composite likelihood \citep{besag1974spatial,varin2011overview}, as well as modern masked objectives including masked language modeling \citep{devlin2019bert,liu2019roberta}, masked autoencoding \citep{he2022masked}, masked latent prediction \citep{assran2023self,bardes2024revisiting,nam2026cjepa}, and absorbing-state discrete diffusion \citep{austin2021structured,sahoo2024simple,shi2024simplified,ou2025absorbing}. In masked diffusion, ELBO formulations can be written as weighted masked-prediction losses, including mixtures over mask counts \citep{shi2024simplified,zheng2025masked}. Several works study how masking affects representations or latent recovery. \citet{kong2023understanding} analyze hierarchical latent-variable identifiability in masked autoencoders, characterizing how masking ratio and patch size affect which latent variables are recovered. \citet{meng2024representation} identify representation deficiencies induced by \texttt{[MASK]} dimensions, while \citet{zhang2026theoretical} connect masked reconstruction to implicit alignment, dimensional collapse, and downstream guarantees under different masking strategies. \citet{pynadath2026candi} study token identifiability under Gaussian corruption. Our focus is on how the mask schedule controls recovery of the underlying joint distribution from near-optimal conditional prediction.

\paragraph{Joint-distribution recovery.}
\citet{liu2022masked} establish exact identification for latent-variable models. Functional-inequality approaches provide quantitative distributional recovery guarantees: \citet{koehler2023statistical} connect score matching and pseudolikelihood to such inequalities, including mixture-weight failures in score matching. \citet{li2024promises} derive finite-sample joint recovery bounds for masked prediction through block AT~\citep{caputo2021block}. We complement these results by characterizing the population excess-risk geometry of mode reweighting: sensitivity to mode reweighting is controlled by schedule-averaged residual mode uncertainty. This yields mode blindness under mode pinning and explicit recovery from low-visibility mass.

\paragraph{Metastability and conditional compatibility.}
Our formulation is related to phase coexistence and slow mixing in Curie--Weiss models \citep{ellis1978statistics,levin2010glauber,ding2009mixing,gheissari2022low,kim2026sharp,samanta2024mixing}, and to mixture descriptions of text \citep{hofmann1999probabilistic,blei2003latent}. Separately, conditionals learned by masked models need not be compatible with a single joint distribution \citep{wang2019bert,young2022inconsistencies,goyal2022exposing,hennigen2023deriving}. Recent work on iterative masked-token sampling also proves exponentially slow escape from semantic basins under a local margin condition at low temperatures \citep{sana2026mixing}. Mode blindness is distinct from conditional incompatibility, in that it occurs even when every model conditional is induced by a single coherent joint distribution.

\paragraph{Population objectives and identifiability.}
Related work asks what population objectives determine about learned representations. For contrastive learning, this includes latent-class structure, spectral structure, alignment and uniformity, collapse, and multimodal feature geometry \citep{saunshi2019theoretical,haochen2021provable,wang2020understanding,jing2022understanding,cai2024clap,liang2022mind,cai2025value,cai2026geometric}. For next-token prediction, recent work characterizes representations shared by distributionally equivalent predictors \citep{marconato2025all} and conditions for recovering human-interpretable latent structure \citep{liu2026i}; \citet{nelson2026statistical} further study representation-level near-identifiability at exact population-risk minimizers. Our $\varepsilon$-identifiability modulus measures the largest joint-distribution error compatible with a given excess-risk budget, characterizing the stability of distributional recovery without positing internal representations or transformation symmetries.

\section{Problem Setup}
\label{sec:setup}

This section formulates the masked-prediction objective on a finite product space and introduces our theoretical framework. \Cref{app:notation} summarizes the notation and mathematical conventions used throughout. Auxiliary lemmas and complete proofs of the main-text results are deferred to \Cref{app:aux-lemmas} and \Cref{app:proofs}, respectively.

\paragraph{Data space and mask schedule.}
Let \(\cA\) be a finite alphabet with $|\cA|\ge2$, and let \(\cX\!:=\!\cA^N\) be the space of length-\(N\) sequences. Let \(p\in\Delta(\cX)\) denote the true data distribution, and write \(X=(X_1,\ldots,X_N)\sim p\). A \emph{mask schedule} is a probability distribution \(\mu\) over subsets \(K\subseteq[N]\), where \(K\) indexes the masked coordinates and \(K^{\complement}:=[N]\setminus K\) indexes the visible coordinates.\footnote{Notation: while $K \subseteq [N]$ and $K^\complement$ strictly denote a mask drawn from $\mu$ and its visible coordinates, we use $V$ and $V^\complement=[N]\setminus V$ to denote general, schedule-independent conditioning sets.} For a \emph{visibility threshold} \(s\in\{0,\ldots,N\}\), define
\begin{equation}
v_{\mathrm{min}}(\mu)
:=
\min\nolimits_{K\in\supp(\mu)}
|K^{\complement}|,
\qquad
\pi_s(\mu)
:=
\mu\big(
|K^{\complement}|\le s
\big).
\label{eq:low-visibility-mass}
\end{equation}
Here, \(v_{\mathrm{min}}(\mu)\) is the minimum visible-context size under schedule $\mu$, and \(\pi_s(\mu)\) is the probability of observing at most \(s\) visible coordinates. We refer to \(\pi_s(\mu)\) as the \emph{low-visibility mass}; in particular, \(\pi_0(\mu)=\mu(K=[N])\) is the \emph{full-mask probability}, corresponding to an empty visible context.

\paragraph{Admissible model distributions.}
A model is represented by a joint distribution \(q\in\Delta(\cX)\) through its masked-block conditionals \(q_K(\cdot\mid x_{K^{\complement}})\). We call models covering the data support \emph{admissible} and define
\begin{equation}
\cQ(p)
:=
\left\{
q\in\Delta(\cX):
\supp(p)\subseteq\supp(q)
\right\}.
\label{eq:model-class}
\end{equation}
For every \(q\in\cQ(p)\), every mask \(K\), and each context \(x_{K^{\complement}}\) with \(p_{K^{\complement}}(x_{K^{\complement}})>0\), the model conditional \(q_K(\cdot\mid x_{K^{\complement}})\) is well defined and satisfies \(p_K(\cdot\mid x_{K^{\complement}})\ll q_K(\cdot\mid x_{K^{\complement}})\). Requiring model conditionals to be induced by a joint distribution isolates schedule-induced information loss from conditional incompatibility. This strengthens our negative results because the mode-blind constructions are themselves valid joint distributions, while making the scope of our positive guarantees explicit.

\begin{definition}[Masked log discrepancy]
\label{def:masked-log-discrepancy}
For \(q\in\cQ(p)\) and a fixed mask \(K\subseteq[N]\), define the log discrepancy
\begin{equation}
\mathfrak D_K(p\| q)
:=
\E_{X_{K^{\complement}}\sim p_{K^{\complement}}}
\left[
\KL\!\left(
p_K(\cdot\mid X_{K^{\complement}})
\,\middle\|\,
q_K(\cdot\mid X_{K^{\complement}})
\right)
\right].
\label{eq:fixed-mask-discrepancy}
\end{equation}
For a mask schedule \(\mu\), define the schedule-averaged masked log discrepancy
\begin{equation}
\mathfrak D_\mu(p\| q)
:=
\E_{K\sim\mu}
\left[\mathfrak D_K(p\| q)\right].
\label{eq:masked-discrepancy}
\end{equation}
\end{definition}

\begin{remark}[Excess-risk interpretation]
We evaluate masked prediction using the joint masked-block log loss $-\log q_K\!\left(X_K\mid X_{K^{\complement}}\right)$. For every \(q\in\cQ(p)\),
\begin{equation}
\mathfrak D_\mu(p\| q)
=
\E_{K\sim\mu}
\E_{X\sim p}
\left[
-\log q_K(X_K\mid X_{K^{\complement}})
+
\log p_K(X_K\mid X_{K^{\complement}})
\right].
\label{eq:excess-risk-identity}
\end{equation}
Hence, \(\mathfrak D_\mu(p\| q)\) is the \emph{population masked-prediction excess risk} of model distribution \(q\) relative to data distribution \(p\): the mask $K$ is sampled from $\mu$, $X\sim p$ supplies the visible context and masked target, and $q$ is scored on the joint conditional distribution of the masked block. Our primary setting scores the entire masked block jointly; see \Cref{app:factorized} for extensions to factorized per-token losses.
\end{remark}

\paragraph{$\varepsilon$-Identifiability modulus.} 
The discrepancy $\mathfrak D_{\mu}(p\| q)$ measures conditional agreement on visible contexts drawn from $p$, but near-zero excess risk need not imply that the joint distributions are close. We therefore introduce an identifiability modulus that quantifies the worst-case joint error compatible with an excess-risk tolerance $\varepsilon$.

\begin{definition}[$\varepsilon$-identifiability modulus]
\label{def:modulus}
Fix a schedule $\mu$ and a metric $d$ continuous on $\Delta(\cX)\times\Delta(\cX)$. For a data distribution $p\in\Delta(\cX)$ and risk budget $\varepsilon\ge0$, the $\varepsilon$-identifiability modulus is
\begin{equation}
\modl_d^{\mu}(p,\varepsilon)
:=
\sup\left\{
d(p,q):
q\in\cQ(p),\;
\mathfrak D_{\mu}(p\| q)\le\varepsilon
\right\}.
\label{eq:modulus}
\end{equation}
\end{definition}

Because $\Delta(\cX)$ is compact and $d$ is continuous, $\modl_d^{\mu}(p,\varepsilon)$ is finite and nondecreasing in $\varepsilon$; see \Cref{app:prop-modulus} for basic properties. The modulus distinguishes two types of guarantees:
\begin{itemize}
    \item A \emph{lower bound} on $\modl_d^{\mu}(p,\varepsilon)$ requires only one candidate distribution $q\in\cQ(p)$ that lies within the excess-risk budget while remaining distant from $p$ under $d$.
    \item An \emph{upper bound} must control \emph{every} admissible model within that budget.
\end{itemize}
At $\varepsilon=0$, \Cref{def:modulus} reduces to exact objective-level identifiability: $\modl_d^{\mu}(p,0)=0$ if and only if $p$ is uniquely identified by zero masked discrepancy. For $\varepsilon>0$, it characterizes the stability of joint-distribution recovery under near-optimal prediction. Although defined here for masked log loss and joint-distribution recovery, the modulus extends naturally to other objectives and targets.

\section{From Approximate Tensorization to the Bimodal Regime}
\label{sec:mixing-and-twomode}

We first examine how block approximate tensorization controls the identifiability modulus by converting masked discrepancy into joint-distribution error. We then introduce a bimodal data geometry in which this local-to-global control weakens, leading to residual mode uncertainty and a mode-reweighting family with an exact information decomposition.

\subsection{Quantitative Identifiability via Block Approximate Tensorization}
\label{sec:mixing}

\paragraph{Block heat-bath dynamics.} Given a distribution \(r\in\Delta(\cX)\), a configuration \(x\in\supp(r)\), and a mask \(K\), a block heat-bath update holds the visible coordinates fixed and redraws the masked block according to
\begin{equation}
\widetilde X = (\widetilde X_{K^{\complement}}, \widetilde X_{K}),
\qquad
\widetilde X_{K^{\complement}}
=
x_{K^{\complement}},
\qquad
\widetilde X_K
\sim
r_K(\cdot\mid x_{K^{\complement}}).
\label{eq:block-update}
\end{equation}
Sampling \(K\sim\mu\) at each step defines the block dynamics induced by \((r,\mu)\). Because each fixed-mask update is reversible with respect to \(r\), their mixture under $\mu$ is also reversible, so \(r\) is stationary for the resulting dynamics. Taking \(r=p\) or \(r=q\) yields the data and model block dynamics, respectively.

\paragraph{Approximate tensorization.} For $r\in\Delta(\cX)$, define the block \emph{approximate-tensorization (AT)} constant $\bar C_{\AT}(r,\mu)\in[0,\infty]$ as the smallest constant such that
\begin{equation}
\KL(p'\,\|\, r)
\;\le\;
\bar C_{\AT}(r,\mu)\,
\mathfrak D_{\mu}(p'\,\|\, r)\qquad
\text{for every } p'\in\Delta(\cX)\ \text{with } \supp(p')\subseteq\supp(r).
\label{eq:at-def}
\end{equation}
This is the density-ratio form of the block factorization of relative entropy. The size of $\bar C_{\AT}(r,\mu)$ determines how strongly conditional discrepancies control joint KL divergence \citep{caputo2021block,li2024promises}. For an admissible model $q$, applying \Cref{eq:at-def} with $r=q$ and $p'=p$, followed by Pinsker's inequality, yields\footnote{Hereafter, we specialize to total variation distance $d=\DTV$ and write $\modl_{\TV}^{\mu}:=\modl_{\DTV}^{\mu}$.}
\begin{equation}
\modl_{\TV}^{\mu}(p,\varepsilon)
\le
\sqrt{\tfrac12\,\bar C_{\AT}^{\varepsilon}(p,\mu)\,\varepsilon},
\quad
\bar C_{\AT}^{\varepsilon}(p,\mu)
:=
\sup\big\{\bar C_{\AT}(q,\mu):
q\in\cQ(p),\ \mathfrak D_{\mu}(p\| q)\le\varepsilon\big\}.
\label{eq:at-upper}
\end{equation}
The AT constant is evaluated at the model distribution $q$ because \Cref{eq:at-def} uses $q$ as the reference measure. See \Cref{app:at-transfer} for the formal transfer and comparisons to classical functional inequalities.

For any sequence $(p_N,\mu_N)$ and excess-risk budgets $\varepsilon_N$, \Cref{eq:at-upper} guarantees vanishing total-variation error whenever $\bar C_{\AT}^{\varepsilon_N}(p_N,\mu_N)\varepsilon_N\to0$. We next introduce a bimodal data geometry that isolates a direction along which masked discrepancy can be much smaller than joint-distribution error.

\subsection{Bimodal Geometry and the Mode-Reweighting Direction}
\label{sec:twomode}

Suppose the data distribution has two well-separated global regimes, represented by a partition $\cX = \cX_+ \sqcup \cX_-$ with $p(\cX_\tau)>0$ for each $\tau \in \{+,-\}$. This partition is a property of $p$ and independent of any mask schedule. Let $w:=p(\cX_+)\in(0,1)$ denote the true global mode weight. For each mode $\tau\in\{+,-\}$, we write $-\tau$ for the complementary mode, define the within-mode conditional distribution $p^\tau(\cdot):=p(\cdot\mid\cX_\tau)$, and let $p^\tau_V$ denote its marginal over coordinates $V\subseteq[N]$.

\paragraph{Residual mode uncertainty.}
To quantify uncertainty about the active mode after observing $X_V$, define the binary mode indicator $Z:=\indic\{X\in\cX_+\}$ and its posterior probability
\begin{equation}
\beta_V(x_V):=p(Z=1\mid X_V=x_V).
\label{eq:mode-posterior}
\end{equation}
We measure \emph{residual mode uncertainty} by the conditional minimum mean-square error (MMSE):
\begin{equation}
\mmse_p(Z\mid X_V)
:=
\E_p\big[(Z-\E_p[Z\mid X_V])^2\big]
=
\E_{X_V\sim p_V}\big[\beta_V(X_V)(1-\beta_V(X_V))\big].
\label{eq:mode-mmse}
\end{equation}
This variance vanishes if and only if the visible context determines the mode almost surely. For a fixed mode weight $w$, it is maximized at $w(1-w)$ when $X_V$ carries no information about $Z$.

\paragraph{Mode-weight perturbations.} For any $\lambda\in(0,1)$, define the mode-reweighted distribution
\begin{equation}
q_\lambda
:=
\lambda\,p^+ + (1-\lambda)\,p^-.
\label{eq:reweighted-law}
\end{equation}
Because the mode supports are disjoint, $q_\lambda$ preserves the within-mode distributions of $p$ exactly and changes only the global mixture weight to $q_\lambda(\cX_+)=\lambda$. Since $\lambda\in(0,1)$, $\supp(q_\lambda)=\supp(p)$, so $q_\lambda\in\cQ(p)$. As in \Cref{eq:mode-posterior}, let $\beta_{\lambda,V}(x_V):=q_\lambda(Z=1\mid X_V=x_V)$ denote the posterior mode probability under $q_\lambda$.

Mode reweighting acts as an exact logit tilt on each visible context and reduces the masked-block conditional KL to the binary KL between mode posteriors:

\begin{lemma}[Posterior shift under mode reweighting]
\label{lem:logit-shift}
Let $q_\lambda$ be defined by \Cref{eq:reweighted-law} for $\lambda\in(0,1)$. Fix $V\subseteq[N]$, let $V^\complement=[N]\setminus V$, and take $x_V\in\supp(p_V)$. The following properties hold:
\begin{enumerate}
\item[\textup{(i)}] The posterior mode probability under $q_\lambda$ satisfies
\begin{equation}
\beta_{\lambda,V}(x_V)
=
\frac{\frac{\lambda}{w}\,\beta_V(x_V)}
{\frac{\lambda}{w}\,\beta_V(x_V)
+
\frac{1-\lambda}{1-w}\,\big(1-\beta_V(x_V)\big)}.
\label{eq:posterior-reweighting}
\end{equation}
\item[\textup{(ii)}] If $\beta_V(x_V)\in(0,1)$, reweighting induces a logit tilt:
\begin{equation}
\logit\big(\beta_{\lambda,V}(x_V)\big)
=
\logit\big(\beta_V(x_V)\big)
+
\logit(\lambda)
-
\logit(w).
\label{eq:posterior-logit-shift}
\end{equation}
\item[\textup{(iii)}] The conditional masked-block distributions satisfy
\begin{equation}
\KL\!\left(
p_{V^\complement}(\cdot\mid x_V)
\,\|\,
q_{\lambda,V^\complement}(\cdot\mid x_V)
\right)
=
\kl\big(\beta_V(x_V)\,\|\,\beta_{\lambda,V}(x_V)\big),
\label{eq:conditional-kl-reweighting}
\end{equation}
where $\kl(\cdot\|\cdot)$ is the binary KL divergence (\Cref{eq:binary-kl-defn}). In particular, whenever $\beta_V(x_V)\in\{0,1\}$, we have $\beta_{\lambda,V}(x_V)=\beta_V(x_V)$ and the divergence in \Cref{eq:conditional-kl-reweighting} vanishes identically.
\end{enumerate}
\end{lemma} 

\begin{proposition}[Mode-weight information decomposition]
\label{prop:info-decomp}
Let $q_\lambda$ be defined as in \Cref{eq:reweighted-law}. For every mask $K\subseteq[N]$ with visible context coordinates $K^{\complement}=[N]\setminus K$,
\begin{align}
\mathfrak D_K(p\| q_\lambda)
&\;=\;
\E_{X_{K^{\complement}}\sim p_{K^{\complement}}}\big[\kl\big(\beta_{K^{\complement}}(X_{K^{\complement}})\,\|\,\beta_{\lambda, K^{\complement}}(X_{K^{\complement}})\big)\big],
\label{eq:decomp-binary}\\
\kl(w\|\lambda)
&\;=\;
\KL(p_{K^{\complement}}\| q_{\lambda,K^{\complement}}) + \mathfrak D_K(p\| q_\lambda).
\label{eq:decomp-chain}
\end{align}
\end{proposition}

\begin{remark}[KL decomposition under conditioning]
\label{rem:info-allocation}
\Cref{eq:decomp-chain} specializes the KL chain rule to mode reweighting: the visible-marginal term accounts for the mismatch in the context distribution, while the masked objective scores only the conditional remainder. When the visible context nearly resolves the mode, almost all of the joint KL discrepancy is absorbed by the visible marginal, leaving little signal for the masked objective. Conversely, at the full-mask endpoint ($K=[N]$), the visible context is empty and absorbs no information, yielding $\mathfrak D_{[N]}(p\|q_\lambda)=\kl(w\|\lambda)$. Thus, mode pinning diminishes the residual training signal, while preserving mode uncertainty restores sensitivity to global mode weights.
\end{remark}

\section{Mode Blindness and Schedule-Controlled Recovery}
\label{sec:main}

We now characterize how masking controls the visibility of global mode weights. Large visible contexts make mode reweighting exponentially hard to detect, low-visibility masks restore this sensitivity, and positive full-mask probability yields distribution-free control of joint KL divergence.

\subsection{Mode Blindness of Masked Prediction}
\label{sec:mode-blindness}

Mode blindness arises when large visible contexts leave exponentially little uncertainty about the active global mode. We formalize this geometry through the following mode-pinning condition.

\begin{assumption}[Large-context mode pinning]
\label{ass:mode-pinning}
There exist constants $c_0\in(0,1/2]$, $c_1,\kappa>0$, and a visibility threshold $s_{\mathrm{pin}}\in[N]$ such that two conditions hold. First, the global mode weights are strictly bounded away from zero: $p(\cX_\tau)\ge c_0$ for both modes $\tau\in\{+,-\}$. Second, for every visible set $V\subseteq[N]$ of size $|V|\ge s_{\mathrm{pin}}$ and each mode $\tau$, the typical set of visible contexts,
\begin{equation}
\cE^\tau_V
:=
\big\{
x_V\in\supp(p^\tau_V):
p(\cX_{-\tau}\mid X_V=x_V)
\le
e^{-\kappa|V|}
\big\},
\label{eq:typical-context-set}
\end{equation}
has probability at least $1-e^{-c_1|V|}$ under $p^\tau_V$.
\end{assumption}

For asymptotic statements under mode pinning, the constants $c_0,c_1,\kappa$ are independent of $N$. Throughout, we denote the combined decay rate by $c:=\min\{c_1,\kappa\}$, and note that \Cref{ass:mode-pinning} implies $w\in[c_0,1-c_0]$. Uniformity over all sufficiently large visible sets makes the assumption schedule-independent, though it can be restrictive when some such sets remain weakly informative about the mode.
See \Cref{app:modepinning-just} for further discussion of its interpretation.

\begin{remark}[Mode-pinning bottleneck]
\label{rem:mode-pinning}
A visible context typical of mode $\cX_\tau$ assigns exponentially small posterior probability to the opposite mode. A block heat-bath update conditioned on such a context therefore remains in $\cX_\tau$ with overwhelming probability. Low-visibility masks and atypical contexts contribute the remaining terms in the cross-mode transition bound. Concretely, letting $\widetilde X$ denote one step of the data block dynamics (\Cref{eq:block-update}), \Cref{lem:cross-mode-transition} yields
\begin{equation}
\Pp(
\widetilde{X}\in\cX_{-\tau}
|
X\in\cX_\tau
)
\le
\pi_s(\mu)+2e^{-cs},
\qquad\text{for every}\quad
s\ge s_{\mathrm{pin}}.
\label{eq:crossing-endpoint}
\end{equation}
When the schedule assigns little probability to low-visibility contexts, the partition $\cX_+\sqcup\cX_-$ forms a low-flow bottleneck. This metastable reading motivates the slow-mixing interpretation. Our objective-level arguments below use only the static geometry.
\end{remark}

\paragraph{Mode reweighting as a blind direction.}
By \Cref{prop:info-decomp}, the discrepancy of a reweighted distribution corresponds to the expected binary KL divergence between the posterior mode probabilities under $p$ and $q_\lambda$. The family $\{q_\lambda\}$ therefore isolates the global mode-weight direction: every member preserves the within-mode distributions of $p$ and perturbs only their relative proportion. Under \Cref{ass:mode-pinning}, large visible contexts leave exponentially small uncertainty regarding the mode, so the posterior mode probability is nearly invariant to perturbations in the prior weight. Averaging over contexts and masks gives the mode-blindness result below; \Cref{app:M-modes} further extends it to $M\ge2$ modes.

\begin{theorem}[Mode blindness of masked prediction]
\label{thm:mode-blind}
Suppose \Cref{ass:mode-pinning} holds with combined rate $c=\min\{c_1,\kappa\}$. Consider the mode-reweighting family $\{q_\lambda\in \cQ(p):\lambda\in I_\mathrm{bl}\}$, where $q_\lambda$ is defined in \Cref{eq:reweighted-law} and $I_\mathrm{bl}:=[c_0/2,\,1-c_0/2]$. There exists a constant $C\le 6c_0^{-2}$, depending only on $c_0$, such that every mask schedule $\mu$ with $v_{\mathrm{min}}(\mu)\ge s_{\mathrm{pin}}$ satisfies
\begin{equation}
\sup\nolimits_{\lambda\in I_\mathrm{bl}}
\mathfrak D_\mu(p\| q_\lambda)
\;\le\;
C\,\E_{K\sim\mu}\big[e^{-c|K^{\complement}|}\big]
\;\le\;
C e^{-c\,v_{\mathrm{min}}(\mu)}.
\label{eq:mode-blindness}
\end{equation}
Hence, the modulus (\Cref{def:modulus}) remains bounded away from zero at an exponentially small excess-risk tolerance:
\begin{equation}
\modl_{\TV}^{\mu}(p,\varepsilon)
\;\ge\;
\frac{1-c_0}{2},
\qquad
\text{for every}\quad
\varepsilon\ge C\,\E_{K\sim\mu}\big[e^{-c|K^{\complement}|}\big].
\label{eq:mode-blindness-modulus}
\end{equation}
\end{theorem}

Together, \Cref{eq:mode-blindness,eq:mode-blindness-modulus} show that a constant mode-weight error, and hence constant total-variation error, can persist at masked excess risk exponentially small in the minimum visible-context size. Because the reweighting family preserves all within-mode conditionals, this failure isolates a global blind direction rather than an optimization failure. This construction also demonstrates why the AT recovery guarantee in \Cref{sec:mixing} must deteriorate under this well-separated data regime:

\begin{corollary}[Exponential lower bounds on block AT constants]
\label{cor:at-blowup}
Suppose \Cref{ass:mode-pinning} holds with combined rate $c=\min\{c_1,\kappa\}$. Let $\mu$ be a mask schedule with $v_{\mathrm{min}}(\mu)\ge s_{\mathrm{pin}}$. Then there exists $\lambda\in I_{\mathrm{bl}}$ such that
\begin{equation}
\bar C_{\AT}(q_\lambda,\mu)
\;\ge\;
\frac{c_0^2(1-c_0)^2}{12}\; e^{c\, v_{\mathrm{min}}(\mu)}
\quad\text{and}\quad
\bar C_{\AT}(p,\mu)
\;\ge\;
\frac{c_0^{4}(1-c_0)^2}{32}\; e^{c\, v_{\mathrm{min}}(\mu)}.
\label{eq:at-blowup-bounds}
\end{equation}
\end{corollary}

\begin{remark}[Interpretation of the AT lower bounds]
The two bounds in \Cref{cor:at-blowup} play distinct roles. Set
\(
\varepsilon_\mu:=C e^{-c\,v_{\mathrm{min}}(\mu)},
\)
with $C$ from \Cref{thm:mode-blind}. Since the reweighted distribution $q_\lambda$ lies in the $\varepsilon_\mu$-discrepancy ball, the first bound yields
\(
\bar C_{\AT}^{\varepsilon_\mu}(p,\mu)\,\varepsilon_\mu
\ge
\bar C_{\AT}(q_\lambda,\mu)\,\varepsilon_\mu
\ge
{C c_0^2(1-c_0)^2}/{12}.
\)
Thus, \Cref{eq:at-upper} cannot certify vanishing joint-distribution error at the mode-blindness scale. The second bound uses \Cref{lem:reverse-blindness} to show that this blowup also occurs for the data distribution itself. Along any sequence with $v_{\mathrm{min}}(\mu_N)=\Omega(N)$, the data-side AT constants grow at least as $e^{\Omega(N)}$. Hence AT-based guarantees requiring constants uniformly bounded in $N$ cannot cover this family.
\end{remark}

\subsection{Residual Mode Uncertainty Governs Mode-Weight Sensitivity}
\label{sec:mode-sensitivity}

Mode blindness shows that large visible contexts suppress sensitivity to global mode weights. The information decomposition in \Cref{prop:info-decomp} identifies the remaining signal with residual uncertainty about the mode after conditioning on the visible context.

\begin{lemma}[Reweighting sensitivity bounds]
\label{lem:sensitivity}
Fix a compact interval $I=[a,b]\subset(0,1)$ and define
\[
\delta_I:=\logit(b)-\logit(a),
\qquad
m_I:=\min\nolimits_{u\in I} u(1-u),
\qquad
c_I:=4e^{-\delta_I},
\qquad
C_I:=\frac{e^{\delta_I}}{m_I}.
\]
Whenever $w,\lambda\in I$, every mask $K\subseteq[N]$ with visible context $K^{\complement}$ satisfies
\begin{equation}
c_I\mmse_p(Z\mid X_{K^{\complement}})\,\kl(w\|\lambda)
\;\le\;
\mathfrak D_K(p\| q_\lambda)
\;\le\;
C_I\mmse_p(Z\mid X_{K^{\complement}})\,\kl(w\|\lambda).
\label{eq:sensitivity}
\end{equation}
\end{lemma}

\begin{remark}[Exact local sensitivity coefficient]
\label{rem:sensitivity-exact}
The sensitivity bounds in \Cref{eq:sensitivity} require no mixing or visibility assumptions; they follow entirely from the bimodal mixture structure. For every mask $K$ with $\mmse_p(Z\mid X_{K^{\complement}})>0$, a second-order expansion as $\lambda\to w$ yields
\begin{equation}
\lim_{\lambda\to w}\frac{\mathfrak D_K(p\| q_\lambda)}{\mmse_p(Z\mid X_{K^{\complement}})\,\kl(w\|\lambda)}
=
\frac{1}{w(1-w)},
\label{eq:exact-sensitivity}
\end{equation}
independently of the mask $K$. The quantity $\mathcal I(w):=(w(1-w))^{-1}$ is the Fisher information of the Bernoulli mode weight. Thus, at fixed $w$, visible contexts with smaller residual mode MMSE have weaker local sensitivity to changes in the global mode weight. See \Cref{rem:exactly-solvable} for the interval dependence of \Cref{eq:sensitivity} and the underlying reduction.
\end{remark}

\paragraph{High- and low-visibility regimes.} Mode pinning controls residual mode uncertainty in the large-context regime:

\begin{corollary}[MMSE upper bound under mode pinning]
\label{cor:mmse-upper}
Under \Cref{ass:mode-pinning}, every visible set $V\subseteq[N]$ satisfying $|V|\ge s_{\mathrm{pin}}$ obeys $\mmse_p(Z\mid X_V) \le \frac{5}{4}e^{-c|V|}$.
\end{corollary}

At low visibility, recovery requires the complementary behavior: small contexts must retain nontrivial mode uncertainty.

\begin{assumption}[Low-visibility residual mode uncertainty]
\label{ass:low-vis-uncertainty}
There exist constants $u_0,L>0$ and a visibility scale $s_{\mathrm{ov}}\ge s_{\mathrm{pin}}$ such that every $V\subseteq[N]$ with $|V|\le s_{\mathrm{ov}}$ satisfies 
\begin{equation}
\mmse_p(Z\mid X_V)
\ge
u_0\, e^{-L|V|}.
\label{eq:low-vis-uncertainty}
\end{equation}
\end{assumption}

\Cref{ass:low-vis-uncertainty} requires residual mode uncertainty not to decay too rapidly in the low-visibility regime. This condition is compatible with large-context mode pinning: \Cref{app:both-assumptions} constructs an explicit distribution satisfying both assumptions with constants independent of $N$. \Cref{app:low-vis-sufficient,app:low-visibility-just} give further interpretations and sufficient conditions. 

Averaging the fixed-mask bounds in \Cref{eq:sensitivity} over the mask schedule yields the corresponding schedule-level sensitivity bounds:

\begin{theorem}[Schedule-controlled mode-weight sensitivity]
\label{thm:mode-sensitivity}
Fix a compact interval $I\subset(0,1)$ containing $w$, and let $c_I,C_I$ be the constants from \Cref{lem:sensitivity}. For every mask schedule $\mu$ and every parameter $\lambda\in I$,
\begin{equation}
c_I\,\kl(w\|\lambda)\,
\E_{K\sim\mu}\big[\mmse_p(Z\mid X_{K^{\complement}})\big]
\;\le\;
\mathfrak D_\mu(p\| q_\lambda)
\;\le\;
C_I\,\kl(w\|\lambda)\,
\E_{K\sim\mu}\big[\mmse_p(Z\mid X_{K^{\complement}})\big].
\label{eq:schedule-mmse-equivalence}
\end{equation}
Under \Cref{ass:low-vis-uncertainty}, every threshold $0\le s\le s_{\mathrm{ov}}$ satisfies
\begin{equation}
\mathfrak D_\mu(p\| q_\lambda)
\;\ge\;
c_I\,u_0\,e^{-Ls}\,\pi_s(\mu)\,\kl(w\|\lambda).
\label{eq:schedule-lower}
\end{equation}
Under \Cref{ass:mode-pinning}, every threshold $s\ge s_{\mathrm{pin}}$ satisfies
\begin{equation}
\mathfrak D_\mu(p\| q_\lambda)
\;\le\;
\frac{5C_I}{4}\Big(
\pi_s(\mu)
+
\E_{K\sim\mu}\big[e^{-c|K^{\complement}|}\,\indic\{|K^{\complement}|>s\}\big]
\Big)\,
\kl(w\|\lambda).
\label{eq:schedule-upper}
\end{equation}
\end{theorem}

\Cref{eq:schedule-mmse-equivalence} gives two-sided sensitivity control for two-mode reweightings under arbitrary schedules, \emph{without} mode-pinning or visibility assumptions. Under the respective assumptions, \Cref{eq:schedule-lower,eq:schedule-upper} further express the sensitivity scale through low-visibility mass \(\pi_s(\mu)\) and a high-visibility remainder. The multimode distinction is discussed in \Cref{sec:multimode-sensitivity}. For fixed $s$, common mask schedules can assign polynomially small, exponentially small, or zero low-visibility mass; see \Cref{app:instantiations} for schedule-specific calculations.

\begin{corollary}[Mode-weight recovery from low-visibility mass]
\label{cor:recovery}
Under the setting of \Cref{thm:mode-sensitivity}, let $c_I$, $C_I$, and $m_I$ be the constants from \Cref{lem:sensitivity}. Fix a mask schedule $\mu$, a weight $\lambda\in I$, and an excess-risk tolerance $\varepsilon\ge0$. Under \Cref{ass:low-vis-uncertainty}, for every $0\le s\le s_{\mathrm{ov}}$ with $\pi_s(\mu)>0$, the condition $\mathfrak D_\mu(p\| q_\lambda)\le\varepsilon$ guarantees
\begin{equation}
|w-\lambda|
\;\le\;
\frac{e^{Ls/2}}{\sqrt{2c_I u_0}}
\sqrt{\frac{\varepsilon}{\pi_s(\mu)}}.
\label{eq:recovery}
\end{equation}
Conversely, under \Cref{ass:mode-pinning}, for every $s\ge s_{\mathrm{pin}}$, the bound $\mathfrak D_\mu(p\| q_\lambda)\le\varepsilon$ holds whenever
\begin{equation}
\frac{5C_I}{4m_I}
\Big(
\pi_s(\mu)
+
\E_{K\sim\mu}\big[
e^{-c|K^{\complement}|}
\indic\{|K^{\complement}|>s\}
\big]
\Big)
|w-\lambda|^2
\le \varepsilon.
\label{eq:recovery-converse}
\end{equation}
\end{corollary}

\paragraph{Interpretation of mode-weight recovery.}
For fixed $s$, $I$, $u_0$, and $L$, \Cref{eq:recovery} bounds the mode-weight error by $|w-\lambda|=O\big(\sqrt{\varepsilon/\pi_s(\mu)}\big)$. Thus, mode-weight errors within $\{q_\lambda:\lambda\in I\}$ vanish whenever $\varepsilon/\pi_s(\mu)\to0$. Conversely, \Cref{eq:recovery-converse} gives a sufficient condition for a mode-weight perturbation to remain within the $\varepsilon$-discrepancy ball. When both bounds apply at the same $s$ and
\[
\E_{K\sim\mu}\big[
e^{-c|K^{\complement}|}\,
\indic\{|K^{\complement}|>s\}
\big]
=
O\big(\pi_s(\mu)\big),
\]
they have matching square-root dependence on $\varepsilon/\pi_s(\mu)$. 

These bounds concern the one-dimensional reweighting family ${q_\lambda}$, with the within-mode distributions held fixed. We next ask when the masked objective controls joint-distribution error uniformly over all $q\in\cQ(p)$.

\subsection{Full-Mask Boundary: Distribution-Free KL Control}
\label{sec:full-mask}

At the full-mask endpoint $K=[N]$, the visible context is empty and $\mathfrak D_{[N]}(p\|q)=\KL(p\|q)$. Hence, positive full-mask probability suffices to control joint KL divergence by the masked discrepancy. The following result shows that it is also necessary for distribution-free control over all admissible models.

\begin{proposition}[Full-mask characterization]
\label{prop:full-mask}
A schedule $\mu$ admits a constant $C_\mu < \infty$ such that
\begin{equation}
\KL(p\| q)
\;\le\;
C_\mu\,\mathfrak D_\mu(p\| q)
\qquad
\text{for all } p\in\Delta(\cX) \text{ and all } q\in\cQ(p)
\label{eq:full-mask-coercivity}
\end{equation}
if and only if $\pi_0(\mu)>0$. When $\pi_0(\mu)>0$, the constant $C_\mu=1/\pi_0(\mu)$ is valid, and consequently
\begin{equation}
\modl_{\TV}^\mu(p,\varepsilon)
\;\le\;
\sqrt{{\varepsilon}/(2\,\pi_0(\mu))}.
\label{eq:full-mask-modulus}
\end{equation}
\end{proposition}

Unlike \Cref{cor:recovery}, \Cref{prop:full-mask} provides KL control uniformly over both data distributions and admissible models. Particular distributions may admit such control without full masks; the necessity statement concerns the distribution-free bound above. This characterization does not extend to per-token log loss, whose full-mask term scores only coordinate marginals rather than the joint distribution (\Cref{app:factorized}). See \Cref{rem:full-mask-sharp} for further discussion of the sharpness of this guarantee.

\section{Experiments}
\label{sec:experiments}

We test two predictions in the conditioned-product mixture of \Cref{prop:product-instantiation}: visibility controls the mode-weight signal, and residual uncertainty predicts gradient-descent recovery speed. Holding the within-mode distributions fixed isolates the schedule's effect, while all conditionals remain induced by the joint distribution $q_\lambda$. Both exact calculation and optimization are performed within this family. \Cref{app:experiments} provides protocols, robustness checks, and interacting and unimodal controls.

\begin{figure}[t]
    \centering
    \begin{subfigure}[t]{0.32\linewidth}
        \centering
        \includegraphics[width=\linewidth]{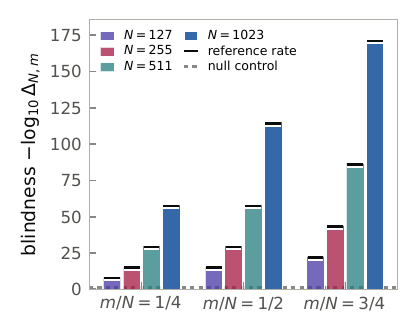}
        \vspace{-1.5em}
        \subcaption{Exponential mode blindness.}
        \label{fig:exp-blind}
    \end{subfigure}\hfill
    \begin{subfigure}[t]{0.32\linewidth}
        \centering
        \includegraphics[width=\linewidth]{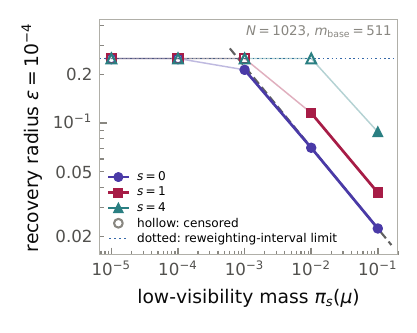}
        \vspace{-1.5em}
        \subcaption{Recovery under low visibility.}
        \label{fig:exp-recovery}
    \end{subfigure}\hfill
    \begin{subfigure}[t]{0.32\linewidth}
        \centering
        \includegraphics[width=\linewidth]{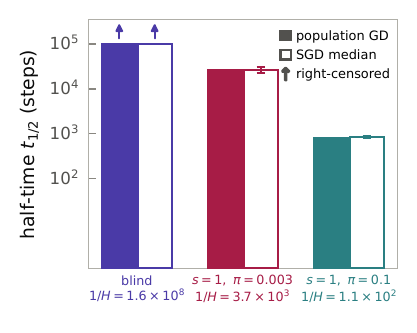}
        \vspace{-1.5em}
        \subcaption{Half-time vs. inverse curvature.}
        \label{fig:exp-mode-halftime}
    \end{subfigure}\\[5pt]
    \begin{subfigure}[t]{\linewidth}
        \centering
        \includegraphics[width=\linewidth]{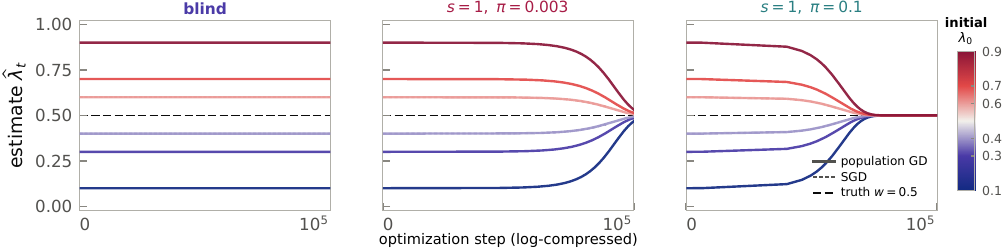}
        \subcaption{Mode-weight optimization trajectories across mask schedules.}
        \label{fig:exp-mode-trajectories}   
    \end{subfigure}
    \caption{Population geometry and optimization within the reweighting family $q_\lambda$. (a) Mode blindness at fixed joint TV error $1/4$; dotted curves show the unimodal control. (b) Largest admissible mode-weight error versus low-visibility mass at $\varepsilon=10^{-4}$; hollow markers indicate interval saturation. (c) Population and median stochastic half-times compared with the inverse-curvature prediction; the blind case is right-censored. (d) Solid curves show population gradient descent toward $w=1/2$. Uncensored SGD traces and bands end after sustained half-error crossings.}
    \label{fig:experiments}
\end{figure}

\paragraph{Mask visibility controls mode-weight sensitivity.}
At joint TV error $1/4$, masked discrepancy decreases exponentially with visible-context size $m:=|K^\complement|$ (\Cref{fig:exp-blind}; \Cref{thm:mode-blind}). Across the evaluated grid with $w=1/2$, the ratio of masked discrepancy to the product of mode-weight KL and residual uncertainty remains in $[4.000, 4.408]$ (\Cref{app:exp-e1}), consistent with the sensitivity scale in \Cref{thm:mode-sensitivity}. Assigning mass $\pi_s(\mu)$ to masks leaving $s\in\{0,1,4\}$ coordinates visible shrinks the largest admissible mode-weight error at $\varepsilon=10^{-4}$ (\Cref{fig:exp-recovery}). Away from interval saturation, the full-mask curve has log--log slope $-0.49$ against $\pi_0(\mu)$, consistent with $\pi_0(\mu)^{-1/2}$ scaling. The $s=1$ and $s=4$ interventions restore mode-weight sensitivity without using full masks (\Cref{cor:recovery}). Thus, low-visibility masks restore the objective signal predicted by the theory even without full masking, while the full-mask case provides the stronger distribution-free endpoint.

\paragraph{Residual uncertainty predicts optimization speed.}
We parameterize $\lambda=(1+e^{-a})^{-1}$ and optimize conditional mode-label cross-entropy, which is exactly the $\lambda$-dependent part of the joint masked-block negative log-likelihood in this model family. Writing $R_\mu(a)$ for the population risk, its curvature is
\(
H_\mu
:=
R''_\mu(\logit(w))
=
\E_{K\sim\mu}\!\left[
\mmse_p(Z\mid X_{K^\complement})
\right].
\)
At the fixed learning rate $\eta=0.1$, local linearization of gradient descent predicts $t_{1/2}\simeq \log 2/(\eta H_\mu)$ for halving the mode-weight error $|\lambda-w|$. For $N=63$, $m_{\mathrm{base}}=31$, and $\lambda_0=0.4$, the blind schedule has a predicted local half-time of $1.11\times10^9$ steps; neither population gradient descent nor any of the five stochastic runs reaches the half-error threshold within $10^5$ steps. Assigning mass $\pi_1(\mu)=0.1$ to masks leaving one coordinate visible reduces the predicted half-time to $770$ steps. The observed values are $776$ steps for population gradient descent and a median of $900$ steps for stochastic gradient descent. Across uncensored schedules, both population and median stochastic half-times follow the predicted ordering (\Cref{fig:exp-mode-halftime}), while the trajectories show recovery under non-full-mask interventions (\Cref{fig:exp-mode-trajectories}). These results show that the same residual mode uncertainty governing population identifiability also controls the local optimization timescale within the reweighting family.

\section{Conclusion}
\label{sec:conclusion}

Mask schedules should be judged by the conditional prediction tasks they create and by the joint-distribution errors they expose. When visible context nearly determines the active mode, masked prediction can remain near-optimal despite substantial errors in global mode weights. Small masked-prediction excess risk must therefore be interpreted relative to both the schedule and the underlying data geometry. When global frequencies matter, schedule design should retain supervision that constrains them while supporting effective learning within modes. Low-visibility masks that preserve mode uncertainty yield recovery bounds within the mode-reweighting family, whereas positive full-mask probability characterizes distribution-free joint-KL control under masked-block log loss. For broader representation learning, these findings motivate asking how strongly an objective constrains the target of interest and at what excess-risk scale its recovery remains stable. 
See \Cref{app:discussion} for broader implications, limitations, and open directions.

\subsection*{AI use statement}

In this work, we used generative AI tools to assist with providing feedback on interpreting experimental results. We did not use generative AI tools to generate synthetic datasets, and other required-disclosure tasks not relevant to this work are not applicable. Additionally, we used generative AI tools for language editing, drafting and revising portions of the manuscript, literature search and related-work discovery, and suggestions on presentation. All AI-assisted experimental interpretations were checked against our experimental results, literature references were verified against primary sources, and AI-assisted manuscript text was reviewed and revised by the authors for accuracy and consistency. We take responsibility for the final content of this work, including text, claims, and artifacts produced with the aid of generative AI.

\subsection*{Ethics statement}

This work studies population-level properties of masked-prediction objectives and does not involve human participants or the collection of new data. The natural-text experiments use existing publicly available corpora in accordance with their documented licenses and report only aggregate measurements of residual mode uncertainty. We do not release derived text data or attempt to identify individuals. We are not aware of additional material ethical risks arising from this work.

\subsection*{Reproducibility statement}

Assumptions and formal statements are provided in the main text, with complete proofs in \Cref{app:proofs}. The appendices provide auxiliary results, the extension to multiple modes, mask-schedule calculations, and further interpretations. \Cref{app:experiments} reports the full experimental protocols, hyperparameters, numerical audits, uncertainty-estimation procedures, and additional results. All sources of randomness are documented for the stochastic experiments. Our code is available at \url{https://github.com/YichaoCai1/Masked-Prediction-Identifiability}.


\bibliography{iclr2027_conference}

@inproceedings{koehler2023statistical,
title={Statistical Efficiency of Score Matching: The View from Isoperimetry},
author={Frederic Koehler and Alexander Heckett and Andrej Risteski},
booktitle={The Eleventh International Conference on Learning Representations},
year={2023},
url={https://openreview.net/forum?id=TD7AnQjNzR6}
}

@inproceedings{li2024promises,
title={Promises and Pitfalls of Generative Masked Language Modeling: Theoretical Framework and Practical Guidelines},
author={Yuchen Li and Alexandre Kirchmeyer and Aashay Mehta and Yilong Qin and Boris Dadachev and Kishore Papineni and Sanjiv Kumar and Andrej Risteski},
booktitle={Forty-first International Conference on Machine Learning},
year={2024},
url={https://openreview.net/forum?id=BTkaKA74mS}
}

@inproceedings{liu2022masked,
author = {Liu, Bingbin and Hsu, Daniel J and Ravikumar, Pradeep and Risteski, Andrej},
booktitle = {Advances in Neural Information Processing Systems},
editor = {S. Koyejo and S. Mohamed and A. Agarwal and D. Belgrave and K. Cho and A. Oh},
pages = {21241--21254},
publisher = {Curran Associates, Inc.},
title = {Masked Prediction: A Parameter Identifiability View},
url = {https://proceedings.neurips.cc/paper_files/paper/2022/file/85dd09d356ca561169b2c03e43cf305e-Paper-Conference.pdf},
volume = {35},
year = {2022}
}

@article{caputo2021block,
  title={Block factorization of the relative entropy via spatial mixing},
  author={Caputo, Pietro and Parisi, Daniel},
  journal={Communications in Mathematical Physics},
  volume={388},
  number={2},
  pages={793--818},
  year={2021},
  publisher={Springer},
  url={https://doi.org/10.1007/s00220-021-04237-1}
}

@inproceedings{wang2019bert,
  title={{BERT} has a mouth, and it must speak: {BERT} as a {Markov} random field language model},
  author={Wang, Alex and Cho, Kyunghyun},
  booktitle={Proceedings of the Workshop on Methods for Optimizing and Evaluating Neural Language Generation},
  pages={30--36},
  year={2019},
  url={https://aclanthology.org/W19-2304.pdf}
}

@inproceedings{
goyal2022exposing,
title={Exposing the Implicit Energy Networks behind Masked Language Models via {Metropolis}--{Hastings}},
author={Kartik Goyal and Chris Dyer and Taylor Berg-Kirkpatrick},
booktitle={International Conference on Learning Representations},
year={2022},
url={https://openreview.net/forum?id=6PvWo1kEvlT}
}

@article{young2022inconsistencies,
  title={Inconsistencies in masked language models},
  author={Young, Tom and Chen, Yunan and You, Yang},
  journal={arXiv preprint arXiv:2301.00068},
  year={2022},
  url={https://arxiv.org/abs/2301.00068}
}

@inproceedings{hennigen2023deriving,
  title={Deriving language models from masked language models},
  author={Hennigen, Lucas Torroba and Kim, Yoon},
  booktitle={Proceedings of the 61st Annual Meeting of the Association for Computational Linguistics (Volume 2: Short Papers)},
  pages={1149--1159},
  year={2023},
  url={https://aclanthology.org/2023.acl-short.99/}
}

@inproceedings{austin2021structured,
title={Structured Denoising Diffusion Models in Discrete State-Spaces},
author={Jacob Austin and Daniel D. Johnson and Jonathan Ho and Daniel Tarlow and Rianne van den Berg},
booktitle={Advances in Neural Information Processing Systems},
editor={A. Beygelzimer and Y. Dauphin and P. Liang and J. Wortman Vaughan},
year={2021},
url={https://openreview.net/forum?id=h7-XixPCAL}
}

@inproceedings{ou2025absorbing,
title={Your Absorbing Discrete Diffusion Secretly Models the Conditional Distributions of Clean Data},
author={Jingyang Ou and Shen Nie and Kaiwen Xue and Fengqi Zhu and Jiacheng Sun and Zhenguo Li and Chongxuan Li},
booktitle={The Thirteenth International Conference on Learning Representations},
year={2025},
url={https://openreview.net/forum?id=sMyXP8Tanm}
}

@inproceedings{devlin2019bert,
  title={{BERT}: Pre-training of deep bidirectional transformers for language understanding},
  author={Devlin, Jacob and Chang, Ming-Wei and Lee, Kenton and Toutanova, Kristina},
  booktitle={Proceedings of the 2019 Conference of the North American Chapter of the Association for Computational Linguistics: Human Language Technologies, Volume 1 (Long and Short Papers)},
  pages={4171--4186},
  year={2019},
  url={https://aclanthology.org/N19-1423/}
}

@inproceedings{nam2026cjepa,
title={Causal-{JEPA}: Learning World Models through Object-Level Latent Masking},
author={Heejeong Nam and Quentin Le Lidec and Lucas Maes and Yann LeCun and Randall Balestriero},
booktitle={Forty-third International Conference on Machine Learning},
year={2026},
url={https://openreview.net/forum?id=VMAHQDOtjp}
}

@article{besag1974spatial,
  title={Spatial interaction and the Statistical analysis of lattice systems},
  author={Besag, Julian},
  journal={Journal of the Royal Statistical Society: Series B (Methodological)},
  volume={36},
  number={2},
  pages={192--225},
  year={1974},
  publisher={Wiley Online Library},
  url={https://doi.org/10.1111/j.2517-6161.1974.tb00999.x}
}

@inproceedings{sahoo2024simple,
 author = {Sahoo, Subham Sekhar and Arriola, Marianne and Schiff, Yair and Gokaslan, Aaron and Marroquin, Edgar and Chiu, Justin T and Rush, Alexander and Kuleshov, Volodymyr},
 booktitle = {Advances in Neural Information Processing Systems},
 doi = {10.52202/079017-4135},
 editor = {A. Globerson and L. Mackey and D. Belgrave and A. Fan and U. Paquet and J. Tomczak and C. Zhang},
 pages = {130136--130184},
 publisher = {Curran Associates, Inc.},
 title = {Simple and Effective Masked Diffusion Language Models},
 url = {https://proceedings.neurips.cc/paper_files/paper/2024/file/eb0b13cc515724ab8015bc978fdde0ad-Paper-Conference.pdf},
 volume = {37},
 year = {2024}
}

@inproceedings{lou2024discrete,
title={Discrete Diffusion Modeling by Estimating the Ratios of the Data Distribution},
author={Aaron Lou and Chenlin Meng and Stefano Ermon},
booktitle={Forty-first International Conference on Machine Learning},
year={2024},
url={https://openreview.net/forum?id=CNicRIVIPA}
}

@inproceedings{gheissari2022low,
  title={Low-temperature {Ising} dynamics with random initializations},
  author={Gheissari, Reza and Sinclair, Alistair},
  booktitle={Proceedings of the 54th Annual ACM SIGACT Symposium on Theory of Computing},
  pages={1445--1458},
  year={2022},
  url={https://doi.org/10.1145/3519935.3519964}
}

@article{varin2011overview,
  title={An overview of composite likelihood methods},
  author={Varin, Cristiano and Reid, Nancy and Firth, David},
  journal={Statistica Sinica},
  pages={5--42},
  year={2011},
  publisher={JSTOR},
  url={https://www.jstor.org/stable/24309261}
}

@article{liu2019roberta,
  title={{RoBERTa}: A robustly optimized {BERT} pretraining approach},
  author={Liu, Yinhan and Ott, Myle and Goyal, Naman and Du, Jingfei and Joshi, Mandar and Chen, Danqi and Levy, Omer and Lewis, Mike and Zettlemoyer, Luke and Stoyanov, Veselin},
  journal={arXiv preprint arXiv:1907.11692},
  year={2019},
  url={https://arxiv.org/abs/1907.11692}
}

@inproceedings{he2022masked,
  title={Masked autoencoders are scalable vision learners},
  author={He, Kaiming and Chen, Xinlei and Xie, Saining and Li, Yanghao and Doll{\'a}r, Piotr and Girshick, Ross},
  booktitle={Proceedings of the IEEE/CVF Conference on Computer Vision and Pattern Recognition},
  pages={16000--16009},
  year={2022},
  url={https://doi.org/10.1109/CVPR52688.2022.01553}
}

@article{levin2010glauber,
  title={Glauber dynamics for the mean-field {Ising} model: cut-off, critical power law, and metastability},
  author={Levin, David A and Luczak, Malwina J and Peres, Yuval},
  journal={Probability Theory and Related Fields},
  volume={146},
  number={1},
  pages={223},
  year={2010},
  publisher={Springer},
  url={https://doi.org/10.1007/s00440-008-0189-z}
}

@article{ding2009mixing,
  title={The mixing time evolution of {Glauber} dynamics for the mean-field {Ising} model},
  author={Ding, Jian and Lubetzky, Eyal and Peres, Yuval},
  journal={Communications in Mathematical Physics},
  volume={289},
  number={2},
  pages={725--764},
  year={2009},
  publisher={Springer},
  url={https://doi.org/10.1007/s00220-009-0781-9}
}

@article{samanta2024mixing,
  title={Mixing phases and metastability for the {Glauber} dynamics on the $p$-spin {Curie}-{Weiss} model},
  author={Samanta, Ramkrishna Jyoti and Mukherjee, Somabha and Zhang, Jiang},
  journal={arXiv preprint arXiv:2412.16952},
  year={2024},
  url={https://arxiv.org/abs/2412.16952}
}

@article{kim2026sharp,
  title={Sharp mixing time asymptotics of {Glauber} dynamics for the {Curie}-{Weiss}-{Potts} model at low temperatures},
  author={Kim, Seonwoo and Lee, Jungkyoung},
  journal={arXiv preprint arXiv:2602.19545},
  year={2026},
  url={https://arxiv.org/abs/2602.19545}
}

@article{blei2003latent,
  title={Latent {Dirichlet} allocation},
  author={Blei, David M and Ng, Andrew Y and Jordan, Michael I},
  journal={Journal of Machine Learning Research},
  volume={3},
  number={Jan},
  pages={993--1022},
  year={2003},
  url={https://www.jmlr.org/papers/volume3/blei03a/blei03a.pdf}
}

@inproceedings{hofmann1999probabilistic,
  title={Probabilistic latent semantic indexing},
  author={Hofmann, Thomas},
  booktitle={Proceedings of the 22nd annual international ACM SIGIR conference on Research and development in information retrieval},
  pages={50--57},
  year={1999},
  url={https://sigir.org/wp-content/uploads/2017/06/p211.pdf}
}

@article{ellis1978statistics,
  title={The statistics of {Curie}-{Weiss} models},
  author={Ellis, Richard S and Newman, Charles M},
  journal={Journal of Statistical Physics},
  volume={19},
  number={2},
  pages={149--161},
  year={1978},
  publisher={Springer},
  url={https://doi.org/10.1007/BF01012508}
}

@article{sana2026mixing,
  title={Mixing Times of {Glauber} Dynamics on Masked Language Models},
  author={Sana, Suvadip and Wolf, Sami and Mehta, Neer and Shah, Alina and Shaikh, Aitzaz and Goodman, Janna and Levine, Lionel},
  journal={arXiv preprint arXiv:2605.16378},
  year={2026},
  url={https://arxiv.org/abs/2605.16378}
}

@inproceedings{assran2023self,
  title={Self-supervised learning from images with a joint-embedding predictive architecture},
  author={Assran, Mahmoud and Duval, Quentin and Misra, Ishan and Bojanowski, Piotr and Vincent, Pascal and Rabbat, Michael and LeCun, Yann and Ballas, Nicolas},
  booktitle={Proceedings of the IEEE/CVF Conference on Computer Vision and Pattern Recognition},
  pages={15619--15629},
  year={2023},
  url={https://openaccess.thecvf.com/content/CVPR2023/papers/Assran_Self-Supervised_Learning_From_Images_With_a_Joint-Embedding_Predictive_Architecture_CVPR_2023_paper.pdf}
}

@inproceedings{baldwin2010language,
  title={Language identification: The long and the short of the matter},
  author={Baldwin, Timothy and Lui, Marco},
  booktitle={Human Language Technologies: The 2010 Annual Conference of the North American Chapter of the Association for Computational Linguistics},
  pages={229--237},
  year={2010},
  url={https://aclanthology.org/N10-1027/}
}

@inproceedings{phan2008learning,
  title={Learning to classify short and sparse text \& web with hidden topics from large-scale data collections},
  author={Phan, Xuan-Hieu and Nguyen, Le-Minh and Horiguchi, Susumu},
  booktitle={Proceedings of the 17th International Conference on World Wide Web},
  pages={91--100},
  year={2008},
  url={https://doi.org/10.1145/1367497.1367510}
}

@inproceedings{
zheng2025masked,
title={Masked Diffusion Models are Secretly Time-Agnostic Masked Models and Exploit Inaccurate Categorical Sampling},
author={Kaiwen Zheng and Yongxin Chen and Hanzi Mao and Ming-Yu Liu and Jun Zhu and Qinsheng Zhang},
booktitle={The Thirteenth International Conference on Learning Representations},
year={2025},
url={https://openreview.net/forum?id=CTC7CmirNr}
}

@inproceedings{
shi2024simplified,
title={Simplified and Generalized Masked Diffusion for Discrete Data},
author={Jiaxin Shi and Kehang Han and Zhe Wang and Arnaud Doucet and Michalis Titsias},
booktitle={The Thirty-eighth Annual Conference on Neural Information Processing Systems},
year={2024},
url={https://openreview.net/forum?id=xcqSOfHt4g}
}

@InProceedings{saunshi2019theoretical,
  title = 	 {A Theoretical Analysis of Contrastive Unsupervised Representation Learning},
  author =       {Saunshi, Nikunj and Plevrakis, Orestis and Arora, Sanjeev and Khodak, Mikhail and Khandeparkar, Hrishikesh},
  booktitle = 	 {Proceedings of the 36th International Conference on Machine Learning},
  pages = 	 {5628--5637},
  year = 	 {2019},
  editor = 	 {Chaudhuri, Kamalika and Salakhutdinov, Ruslan},
  volume = 	 {97},
  series = 	 {Proceedings of Machine Learning Research},
  month = 	 {09--15 Jun},
  publisher =    {PMLR},
  url = 	 {https://proceedings.mlr.press/v97/saunshi19a.html}
}

@article{haochen2021provable,
  title={Provable guarantees for self-supervised deep learning with spectral contrastive loss},
  author={HaoChen, Jeff Z and Wei, Colin and Gaidon, Adrien and Ma, Tengyu},
  journal={Advances in Neural Information Processing Systems},
  volume={34},
  pages={5000--5011},
  year={2021},
  url={https://proceedings.neurips.cc/paper/2021/file/27debb435021eb68b3965290b5e24c49-Paper.pdf}
}

@inproceedings{wang2020understanding,
  title={Understanding contrastive representation learning through alignment and uniformity on the hypersphere},
  author={Wang, Tongzhou and Isola, Phillip},
  booktitle={International Conference on Machine Learning},
  pages={9929--9939},
  year={2020},
  organization={PMLR},
  url={https://proceedings.mlr.press/v119/wang20k/wang20k.pdf}
}

@inproceedings{cai2024clap,
  title={{CLAP}: Isolating content from style through contrastive learning with augmented prompts},
  author={Cai, Yichao and Liu, Yuhang and Zhang, Zhen and Shi, Javen Qinfeng},
  booktitle={European Conference on Computer Vision},
  pages={130--147},
  year={2024},
  organization={Springer},
  url={https://link.springer.com/chapter/10.1007/978-3-031-72664-4_8}
}

@inproceedings{cai2025value,
title={On the Value of Cross-Modal Misalignment in Multimodal Representation Learning},
author={Yichao Cai and Yuhang Liu and Erdun Gao and Tianjiao Jiang and Zhen Zhang and Anton van den Hengel and Javen Qinfeng Shi},
booktitle={The Thirty-ninth Annual Conference on Neural Information Processing Systems},
year={2025},
url={https://openreview.net/forum?id=3KtPujOw5z}
}

@inproceedings{cai2026geometric,
title={The Geometric Mechanics of Contrastive Representation Learning: Alignment Potentials, Entropic Dispersion, and Cross-Modal Divergence},
author={Yichao Cai and Zhen Zhang and Yuhang Liu and Javen Qinfeng Shi},
booktitle={Forty-third International Conference on Machine Learning},
year={2026},
url={https://openreview.net/forum?id=X9VtjdBlQw}
}

@inproceedings{
jing2022understanding,
title={Understanding Dimensional Collapse in Contrastive Self-supervised Learning},
author={Li Jing and Pascal Vincent and Yann LeCun and Yuandong Tian},
booktitle={International Conference on Learning Representations},
year={2022},
url={https://openreview.net/forum?id=YevsQ05DEN7}
}

@inproceedings{
liang2022mind,
title={Mind the Gap: Understanding the Modality Gap in Multi-modal Contrastive Representation Learning},
author={Weixin Liang and Yuhui Zhang and Yongchan Kwon and Serena Yeung and James Zou},
booktitle={Advances in Neural Information Processing Systems},
editor={Alice H. Oh and Alekh Agarwal and Danielle Belgrave and Kyunghyun Cho},
year={2022},
url={https://openreview.net/forum?id=S7Evzt9uit3}
}

@inproceedings{tiedemann-2012-parallel,
    title = "Parallel Data, Tools and Interfaces in {OPUS}",
    author = {Tiedemann, J{\"o}rg},
    editor = "Calzolari, Nicoletta  and
      Choukri, Khalid  and
      Declerck, Thierry  and
      Do{\u{g}}an, Mehmet U{\u{g}}ur  and
      Maegaard, Bente  and
      Mariani, Joseph  and
      Moreno, Asuncion  and
      Odijk, Jan  and
      Piperidis, Stelios",
    booktitle = "Proceedings of the Eighth International Conference on Language Resources and Evaluation ({LREC}'12)",
    month = may,
    year = "2012",
    address = "Istanbul, Turkey",
    publisher = "European Language Resources Association (ELRA)",
    url = "http://www.lrec-conf.org/proceedings/lrec2012/pdf/463_Paper.pdf",
    pages = "2214--2218",
}

@inproceedings{zhang-etal-2020-improving,
    title = "Improving Massively Multilingual Neural Machine Translation and Zero-Shot Translation",
    author = "Zhang, Biao  and
      Williams, Philip  and
      Titov, Ivan  and
      Sennrich, Rico",
    editor = "Jurafsky, Dan  and
      Chai, Joyce  and
      Schluter, Natalie  and
      Tetreault, Joel",
    booktitle = "Proceedings of the 58th Annual Meeting of the Association for Computational Linguistics",
    month = jul,
    year = "2020",
    address = "Online",
    publisher = "Association for Computational Linguistics",
    url = "https://aclanthology.org/2020.acl-main.148",
    doi = "10.18653/v1/2020.acl-main.148",
    pages = "1628--1639",
}

@misc{codeparrot_githubcodeclean,
  author       = {{CodeParrot}},
  title        = {{GitHub Code Clean}},
  year         = {2022}, 
  howpublished = {Hugging Face Hub, \url{https://huggingface.co/datasets/codeparrot/github-code-clean}},
  note         = {Derived from \texttt{codeparrot/github-code}; accessed 2026-07-20}
}

@article{raffel2020exploring,
  title={Exploring the limits of transfer learning with a unified text-to-text transformer},
  author={Raffel, Colin and Shazeer, Noam and Roberts, Adam and Lee, Katherine and Narang, Sharan and Matena, Michael and Zhou, Yanqi and Li, Wei and Liu, Peter J},
  journal={Journal of Machine Learning Research},
  volume={21},
  number={140},
  pages={1--67},
  year={2020},
  url={https://www.jmlr.org/papers/volume21/20-074/20-074.pdf}
}

@article{radford2019language,
  title={Language models are unsupervised multitask learners},
  author={Radford, Alec and Wu, Jeffrey and Child, Rewon and Luan, David and Amodei, Dario and Sutskever, Ilya and others},
  journal={OpenAI blog},
  volume={1},
  number={8},
  pages={9},
  year={2019},
  url={https://cdn.openai.com/better-language-models/language_models_are_unsupervised_multitask_learners.pdf}
}

@article{
bardes2024revisiting,
title={Revisiting Feature Prediction for Learning Visual Representations from Video},
author={Adrien Bardes and Quentin Garrido and Jean Ponce and Xinlei Chen and Michael Rabbat and Yann LeCun and Mido Assran and Nicolas Ballas},
journal={Transactions on Machine Learning Research},
issn={2835-8856},
year={2024},
url={https://openreview.net/forum?id=QaCCuDfBk2}
}

@inproceedings{
nelson2026statistical,
title={Statistical and structural identifiability in representation learning},
author={Walter Nelson and Marco Fumero and Theofanis Karaletsos and Francesco Locatello},
booktitle={The Fourteenth International Conference on Learning Representations},
year={2026},
url={https://openreview.net/forum?id=Wa3cfE3Iay}
}

@inproceedings{
marconato2025all,
title={All or None: Identifiable Linear Properties of Next-Token Predictors in Language Modeling},
author={Emanuele Marconato and Sebastien Lachapelle and Sebastian Weichwald and Luigi Gresele},
booktitle={The 28th International Conference on Artificial Intelligence and Statistics},
year={2025},
url={https://openreview.net/forum?id=XCmIlemQP5}
}

@inproceedings{
liu2026i,
title={I Predict Therefore {I} Am: Is Next Token Prediction Enough to Learn Human-Interpretable Concepts from Data?},
author={Yuhang Liu and Dong Gong and Yichao Cai and Erdun Gao and Zhen Zhang and Biwei Huang and Mingming Gong and Anton van den Hengel and Javen Qinfeng Shi},
booktitle={The Fourteenth International Conference on Learning Representations},
year={2026},
url={https://openreview.net/forum?id=vVYD74U5KE}
}

@inproceedings{
arriola2025block,
 author = {Arriola, Marianne and Gokaslan, Aaron and Chiu, Justin and Yang, Zhihan and Qi, Zhixuan and Han, Jiaqi and Sahoo, Subham and Kuleshov, Volodymyr},
 booktitle = {International Conference on Learning Representations},
 editor = {Y. Yue and A. Garg and N. Peng and F. Sha and R. Yu},
 pages = {50726--50753},
 title = {Block Diffusion: Interpolating Between Autoregressive and Diffusion Language Models},
 url = {https://proceedings.iclr.cc/paper_files/paper/2025/file/7ede97c3e082c6df10a8d6103a2eebd2-Paper-Conference.pdf},
 volume = {2025},
 year = {2025}
}

@inproceedings{kong2023understanding,
author    = {Kong, Lingjing and Ma, Martin Q. and Chen, Guangyi and Xing, Eric P. and Chi, Yuejie and Morency, Louis-Philippe and Zhang, Kun},
title     = {Understanding Masked Autoencoders via Hierarchical Latent Variable Models},
booktitle = {Proceedings of the IEEE/CVF Conference on Computer Vision and Pattern Recognition (CVPR)},
month     = {June},
year      = {2023},
pages     = {7918-7928},
url={https://doi.org/10.1109/CVPR52729.2023.00765}
}

@inproceedings{
meng2024representation,
title={Representation Deficiency in Masked Language Modeling},
author={Yu Meng and Jitin Krishnan and Sinong Wang and Qifan Wang and Yuning Mao and Han Fang and Marjan Ghazvininejad and Jiawei Han and Luke Zettlemoyer},
booktitle={The Twelfth International Conference on Learning Representations},
year={2024},
url={https://openreview.net/forum?id=b3l0piOrGU}
}

@inproceedings{
pynadath2026candi,
title={{CANDI}: Hybrid Discrete-Continuous Diffusion Models},
author={Patrick Pynadath and Jiaxin Shi and Ruqi Zhang},
booktitle={Forty-third International Conference on Machine Learning},
year={2026},
url={https://openreview.net/forum?id=03KQGUvw1e}
}

@article{zhang2026theoretical,
  title={A Theoretical Framework for Masked Pretraining ({MPT})},
  author={Zhang, Qi and Zhou, Runyu and Wang, Yifei and Wang, Yisen},
  journal={Journal of Machine Learning Research},
  volume={27},
  number={199},
  pages={1--46},
  year={2026},
  url={https://www.jmlr.org/papers/volume27/25-0477/25-0477.pdf}
}
\bibliographystyle{iclr2027_conference}

\clearpage
\appendix

\crefalias{section}{appendix}
\crefalias{subsection}{appendix}
\crefalias{subsubsection}{appendix}

\startcontents[appendix]

\noindent\textbf{\large Contents of Appendix}
\vspace{0.5em}
\hrule
\begingroup
\normalsize
\renewcommand{\addvspace}[1]{}
\setlength{\baselineskip}{0.6em} 
\printcontents[appendix]{l}{1}{\setcounter{tocdepth}{3}}
\endgroup
\vspace{0.5em}    
\hrule

\clearpage

\section{Notation and Conventions}
\label{app:notation}

This section collects the notation and conventions used throughout the paper; \Cref{tab:notation} provides a quick reference.

\begin{table}[tbp]
\centering
\caption{Canonical notation used throughout the paper.}
\label{tab:notation}
\begin{tabular}{ll}
\toprule
\textbf{Symbol} & \textbf{Meaning} \\
\midrule
\multicolumn{2}{l}{\emph{Spaces and probability distributions}}\\[3pt]
$[N]$ & Coordinate index set $\{1,\dots,N\}$.\\
$\mathcal{A}$ & Finite alphabet.\\
\(\mathcal{X}=\mathcal{A}^N\) & Configuration space.\\
$\Delta(S)$ & Probability simplex over a finite set $S$.\\
$p$ & Data distribution on $\mathcal{X}$; unless stated otherwise, $X\sim p$.\\
\(q,r\) & Model distribution \(q\) and generic distribution \(r\) on \(\mathcal X\).\\
$\operatorname{supp}(r)$ & Support of a distribution $r$.\\
$\cQ(p)$ & Support-covering model class $\{q\in\Delta(\mathcal{X}):\operatorname{supp}(p)\subseteq\operatorname{supp}(q)\}$.\\
\midrule
\multicolumn{2}{l}{\emph{Coordinates, marginals, and conditionals}}\\[3pt]
$X,x$ & Random configuration and a realization.\\
$S\subseteq[N]$ & Generic coordinate set.\\
$K,\,K^{\complement}$ & Masked coordinate set and visible set $K^{\complement}:=[N]\setminus K$.\\
$V$ & Generic visible or conditioning coordinate set.\\
$X_S,x_S$ & Subvector indexed by $S$ and its realization.\\
$r_S$ & Marginal distribution of $X_S$ under $r$.\\
$r_K(\cdot\mid x_V)$ & Conditional distribution of $X_K$ given $X_V=x_V$.\\
\midrule
\multicolumn{2}{l}{\emph{Mask schedules and block updates}}\\[3pt]
$\mu,\ \operatorname{supp}(\mu)$ & Mask schedule on $2^{[N]}$ and its support.\\
$v_{\mathrm{min}}(\mu)$ & Minimum visible-set size $\min_{K\in\operatorname{supp}(\mu)}|K^{\complement}|$.\\
$\pi_s(\mu)$ & Low-visibility mass $\mu(\{K:|K^{\complement}|\le s\})$.\\
\midrule
\multicolumn{2}{l}{\emph{Masked objective and identifiability}}\\[3pt]
$\mathfrak D_K(p\|q)$ & Fixed-mask log discrepancy (\Cref{def:masked-log-discrepancy}).\\
$\mathfrak D_\mu(p\|q)$ & Schedule-averaged masked log discrepancy $\mathbb{E}_{K\sim\mu}\mathfrak D_K(p\|q)$.\\
$\modl^{\mu}_{d}(p,\varepsilon)$ & $\varepsilon$-identifiability modulus under distance $d$ (\Cref{def:modulus}).\\
$\modl^{\mu}_{\mathrm{TV}}(p,\varepsilon)$ & Total-variation identifiability modulus.\\
\midrule
\multicolumn{2}{l}{\emph{Mode structure and constructions}}\\[3pt]
$\mathcal{X}_+,\mathcal{X}_-$ & Mode regions forming the partition $\mathcal{X}=\mathcal{X}_+\sqcup\mathcal{X}_-$.\\
$\tau,-\tau$ & Mode label in $\{+,-\}$ and the opposite label.\\
$Z$ & Binary mode indicator $Z:=\indic\{X\in\mathcal{X}_+\}$.\\
$p^\tau$ & Mode-conditioned distribution $p(\cdot\mid\mathcal{X}_\tau)$.\\
$p^\tau_V$ & Marginal of $p^\tau$ on coordinates $V$.\\
$\mathcal{E}^\tau_V$ & Typical visible contexts from mode region $\mathcal{X}_\tau$.\\
$\beta_V(x_V)$ & Posterior mode probability $p(Z=1\mid X_V=x_V)$.\\
\(\mmse_p(Z\mid X_V)\) & Residual mode MMSE; see \Cref{eq:mode-mmse}.\\
\(q_\lambda\) & Mode-reweighted distribution; see \Cref{eq:reweighted-law}.\\
$c_0$ & Uniform lower bound on the probability mass of each mode region.\\
$c_1$ & Exponential rate controlling atypical visible contexts.\\
$\kappa$ & Exponential rate controlling the opposite-mode posterior on $\mathcal{E}^\tau_V$.\\
$s_\mathrm{pin}$ & Visibility threshold in the mode-pinning assumption (\Cref{ass:mode-pinning}).\\
$s_\mathrm{ov}$ & Visibility scale in the low-visibility uncertainty assumption (\Cref{ass:low-vis-uncertainty}).\\
$c$ & Combined rate $c:=\min\{c_1,\kappa\}$.\\
\midrule
\multicolumn{2}{l}{\emph{Divergences and constants}}\\[3pt]
$D_{\mathrm{KL}}(\cdot\|\cdot)$ & KL divergence.\\
$\kl(\cdot\|\cdot)$ & Binary KL divergence.\\
$D_{\mathrm{TV}}(\cdot,\cdot)$ & Total-variation distance.\\
$\logit(\cdot)$ & Log-odds.\\
$\diam_d(\Delta(\mathcal{X}))$ & Diameter of the probability simplex under $d$.\\
$\bar C_{\AT}(r,\mu)$ & Block approximate-tensorization constant of a distribution $r$ (\Cref{eq:at-def}). \\
$\bar C_{\AT}^{\varepsilon}(p,\mu)$ & Budget-uniform AT constant over $\{q\in\cQ(p):\mathfrak D_{\mu}(p\|q)\le\varepsilon\}$ (\Cref{eq:at-upper}). \\
\bottomrule
\end{tabular}
\end{table}

\paragraph{General conventions.}
For an integer $N\ge1$, we write $[N]:=\{1,\ldots,N\}$. Unless stated otherwise, asymptotic statements are taken as $N\to\infty$. For real-valued sequences $f=f_N$ and $g=g_N$, we write $f=O(g)$ if $|f_N|\le C|g_N|$ for some constant $C>0$ and all sufficiently large $N$. We write \(f=\Omega(g)\) if \(g=O(f)\), and \(f=\Theta(g)\), equivalently \(f\asymp g\), if both bounds hold. For sequences with $g_N\ne0$ for all sufficiently large $N$, we write $f\sim g$ if $f_N/g_N\to1$ and $f=o(g)$ if $f_N/g_N\to0$. We use \(\simeq\) for informal approximate equality. A sequence is called \emph{sublinear} if it is $o(N)$. Calligraphic letters such as \(\cA\), \(\cX\), and \(\cQ(p)\) denote sets, spaces, or classes. Uppercase letters generally denote random variables, and lowercase letters denote their realizations. We also use uppercase letters for coordinate sets such as $K,V,S\subseteq[N]$. For a subset $S\subseteq[N]$, we define the projected vectors $X_S:=(X_i)_{i\in S}$ and $x_S:=(x_i)_{i\in S}$.

\paragraph{Probability distributions.}
For a finite space $\cX$, $\Delta(\cX)$ denotes the simplex of probability distributions on $\cX$. Given a distribution $r\in\Delta(\cX)$ on the configuration space $\cX=\cA^N$ and a subset $S\subseteq[N]$, $r_S$ denotes the marginal distribution of $X_S$ under $r$. For complementary sets $V\subseteq[N]$ and $K=[N]\setminus V$, the conditional distribution of $X_K$ given $X_V=x_V$ is written as $r_K(\cdot\mid x_V)$ whenever $r_V(x_V)>0$. We write $r(A)$ and $r(A\mid X_V=x_V)$ for probabilities under a specified distribution $r$. We also use \(\Pp\) when the governing distribution or random experiment is clear from context or indicated by a subscript. Expectations follow the same convention. Subscripts such as $\E_r$, $\E_{X\sim p}$, and $\E_{K\sim\mu}$ indicate the source of randomness when needed.

\paragraph{Information-theoretic quantities.}
All logarithms are natural. For probability distributions $p,q\in\Delta(\cX)$, the total-variation distance is
\begin{equation}
\DTV(p,q)
:=
\sup\nolimits_{A\subseteq\cX}|p(A)-q(A)|
=
\tfrac12\sum\nolimits_{x\in\cX}|p(x)-q(x)|.
\label{eq:tv-defn}
\end{equation}
The KL divergence is
\begin{equation}
\KL(p\|q)
:=
\sum\nolimits_{x\in\cX}
p(x)\log\frac{p(x)}{q(x)},
\label{eq:kl-defn}
\end{equation}
with the conventions $0\log\frac{0}{b}:=0$ for $b\ge0$ and $a\log\frac{a}{0}:=+\infty$ for $a>0$.

For $a\in(0,1)$, define the log-odds
\begin{equation}
\logit(a):=\log\frac{a}{1-a},
\label{eq:log-odds}
\end{equation}
with $\logit(0):=-\infty$ and $\logit(1):=+\infty$.

For $a,b\in[0,1]$, define the binary KL divergence
\begin{equation}
\kl(a\|b)
:=
a\log\frac{a}{b}
+
(1-a)\log\frac{1-a}{1-b},
\label{eq:binary-kl-defn}
\end{equation}
with the same extended-value conventions. In particular, $\kl(a\|b)\ge0$, with equality if and only if $a=b$.

\paragraph{Minimum mean-square error.}
For jointly distributed random variables $Y$ and $X$ under a probability distribution $r$, with $Y\in\R^d$ and $\E_r\|Y\|_2^2<\infty$, define the minimum mean-square error
\begin{equation}
\mmse_r(Y\mid X)
:=
\E_r\!\left[
\left\|Y-\E_r[Y\mid X]\right\|_2^2
\right].
\label{eq:mmse-defn}
\end{equation}
For scalar $Y$, this reduces to $\E_r[(Y-\E_r[Y\mid X])^2]$.

\clearpage
\section{Lemmas and Auxiliary Results}
\label{app:aux-lemmas}

This section collects lemmas used in the proofs and clarifies several theoretical interpretations.

\subsection{Basic Properties of the Identifiability Modulus}
\label{app:prop-modulus}


\begin{proposition}[Basic properties of the modulus]
\label{prop:modulus-wellposed}
For every $p\in\Delta(\cX)$, mask schedule $\mu$, continuous distance $d$, and discrepancy budget $\varepsilon\ge0$,
\[
0
\le
\modl_d^{\mu}(p,\varepsilon)
\le
\diam_d(\Delta(\cX))
<
\infty,\qquad \text{where }\;
\diam_d(\Delta(\cX))
:=
\sup_{q,q'\in\Delta(\cX)} d(q,q').
\]
Moreover, $\modl_d^{\mu}(p,\varepsilon)$ is nondecreasing in $\varepsilon$.
\end{proposition}

\begin{proof}
Write $\cQ_{\varepsilon}^{\mu}(p):=\{q\in\cQ(p):\mathfrak D_{\mu}(p\| q)\le\varepsilon\}$ for the admissible set in \Cref{eq:modulus}. We first show it is nonempty. The data distribution itself belongs to $\cQ(p)$. Moreover, for every mask $K$, the fixed-mask discrepancy 
\[
\mathfrak D_K(p\| p)=\E_{X_{K^{\complement}}\sim p_{K^{\complement}}}[\KL(p_K(\cdot\mid X_{K^{\complement}})\| p_K(\cdot\mid X_{K^{\complement}}))]=0,
\]
since each conditional KL divergence between identical distributions vanishes. Averaging over $K\sim\mu$ gives
\[
\mathfrak D_{\mu}(p\| p)=\E_{K\sim\mu}\mathfrak D_K(p\| p)=0\le\varepsilon,
\]
so $p\in\cQ_{\varepsilon}^{\mu}(p)$. Hence the supremum in \Cref{eq:modulus} is over a nonempty set, and since $d(p,p)=0$,
\[
\modl_d^{\mu}(p,\varepsilon)\ge 0.
\]

Next, for every $q\in\cQ_{\varepsilon}^{\mu}(p)$,
\[
d(p,q)\le\diam_d(\Delta(\cX)).
\]
Because $\cA$ is a finite alphabet, $\cX=\cA^N$ is finite and the simplex $\Delta(\cX)$ is compact. Since $d$ is continuous on $\Delta(\cX)\times\Delta(\cX)$, the diameter $\diam_d(\Delta(\cX))$ is finite. Taking the supremum over $\cQ_{\varepsilon}^{\mu}(p)$ therefore gives
\[
0
\le
\modl_d^{\mu}(p,\varepsilon)
\le
\diam_d(\Delta(\cX))
<
\infty.
\]

Finally, let $0\le\varepsilon\le\varepsilon'$. Then $\cQ_{\varepsilon}^{\mu}(p)\subseteq\cQ_{\varepsilon'}^{\mu}(p)$, because every model distribution with discrepancy at most $\varepsilon$ also has discrepancy at most $\varepsilon'$. Taking the supremum of $d(p,q)$ over these nested sets yields
\[
\modl_d^{\mu}(p,\varepsilon)
\le
\modl_d^{\mu}(p,\varepsilon').
\]
Thus the modulus is nondecreasing in $\varepsilon$.
\end{proof}

\medskip
\subsection{Approximate Tensorization and the Modulus}
\label{app:at-transfer}

This section relates the entropy-functional form of block AT \citep{caputo2021block} to its formulation in terms of conditional KL divergences \citep{li2024promises}, and derives the modulus bound of \Cref{eq:at-upper}.

Throughout, we fix a mask schedule $\mu$. For a distribution $r\in\Delta(\cX)$ and a function $\phi\colon\cX\to[0,\infty)$, the entropy functional is
\begin{equation}
\Ent_r(\phi)
:=
\E_r[\phi\log\phi]-\E_r[\phi]\log\E_r[\phi],
\label{eq:ent-def}
\end{equation}
with the convention $0\log0:=0$. If $\E_r[\phi]=0$, nonnegativity implies $\phi=0$ $r$-almost surely and hence $\Ent_r(\phi)=0$. Jensen's inequality applied to the convex function $t\mapsto t\log t$ gives $\Ent_r(\phi)\ge0$. Because every expectation below integrates $\phi$ against $r$ or its conditionals, the values of $\phi$ outside $\supp(r)$ have no effect on the integrals. We therefore assume without loss of generality that $\phi$ vanishes outside this support.

The subtracted term in \Cref{eq:ent-def} extends relative entropy to unnormalized arguments: When $\E_r[\phi]>0$, normalizing $\phi$ defines the distribution $p_\phi:=(\phi/\E_r[\phi])\,r$. Expanding its relative entropy, $\KL(p_\phi\| r)=\E_r[(\phi/\E_r[\phi])\log(\phi/\E_r[\phi])]$, gives
\begin{equation}
\Ent_r(\phi)
=
\E_r[\phi]\,\KL\!\big(p_\phi\;\big\|\;r\big).
\label{eq:ent-kl}
\end{equation}
Thus, $\Ent_r$ is the homogeneous extension of relative entropy from normalized densities to nonnegative functions.
In particular, for a distribution $p'\in\Delta(\cX)$ with $\supp(p')\subseteq\supp(r)$, its density with respect to $r$,
\begin{equation}
\phi_{p'}
:=
\frac{\mathrm dp'}{\mathrm dr},
\qquad
\phi_{p'}(x)
=
\begin{cases}
p'(x)/r(x), & x\in\supp(r),\\[2pt]
0, & \text{otherwise},
\end{cases}
\label{eq:density-def}
\end{equation}
satisfies $\E_r[\phi_{p'}]\!=\!p'(\supp(r))\!=\!1$, and thus recovers the relative entropy $\Ent_r(\phi_{p'})=\KL(p'\,\|\,r)$. 

For a mask $K$ and a context $x_{K^\complement}$ with $r_{K^\complement}(x_{K^\complement})>0$, we write $\Ent_{r_K(\cdot\mid x_{K^\complement})}(\phi)$ for the entropy of the fiber restriction $x_K\mapsto\phi(x_{K^\complement},x_K)$ under the conditional distribution. Finally, for a distribution $p'$ with $\supp(p')\subseteq\supp(r)$ we have $r\in\cQ(p')$ by \Cref{eq:model-class}, so $\mathfrak D_\mu(p'\,\|\,r)$ is well defined by \Cref{def:masked-log-discrepancy}.

\begin{lemma}[Equivalent forms of block AT and a modulus bound]
\label{lem:at-transfer}
Fix a mask schedule $\mu$.
\begin{enumerate}
\item[(i)] For every distribution $r\in\Delta(\cX)$ and every constant $\bar C\in[0,\infty)$, the entropy-functional block-factorization inequality
\begin{equation}
\Ent_r(\phi)
\;\le\;
\bar C\,
\E_{K\sim\mu}\,
\E_{X_{K^\complement}\sim r_{K^\complement}}
\Big[\Ent_{r_K(\cdot\mid X_{K^\complement})}(\phi)\Big]
\quad\text{for every } \phi\colon\cX\to[0,\infty)
\label{eq:at-functional}
\end{equation}
holds if and only if the density-ratio inequality of \Cref{eq:at-def} holds at the same constant:
\begin{equation*}
\KL(p'\,\|\, r)
\;\le\;
\bar C\,
\mathfrak D_{\mu}(p'\,\|\, r)
\qquad\text{for every } p'\in\Delta(\cX)
\text{ with } \supp(p')\subseteq\supp(r).
\end{equation*}
Thus, the optimal constants of the two inequalities coincide with $\bar C_{\AT}(r,\mu)$ in \Cref{eq:at-def}.

\item[(ii)] For every $p\in\Delta(\cX)$ and every $\varepsilon\ge0$, writing $\cQ_{\varepsilon}^{\mu}(p):=\{q\in\cQ(p): \mathfrak D_{\mu}(p\| q)\le\varepsilon\}$ for the $\varepsilon$-discrepancy ball, the modulus bound of \Cref{eq:at-upper} holds:
\begin{equation*}
\modl_{\TV}^{\mu}(p,\varepsilon)
\;\le\;
\sqrt{\tfrac12\,\bar C_{\AT}^{\varepsilon}(p,\mu)\,\varepsilon},
\qquad\text{with }\;
\bar C_{\AT}^{\varepsilon}(p,\mu)
=
\sup\nolimits_{q\in\cQ_{\varepsilon}^{\mu}(p)}
\bar C_{\AT}(q,\mu),
\end{equation*}
where the right-hand side is understood as $+\infty$ when $\bar C_{\AT}^{\varepsilon}(p,\mu)=\infty$.
\end{enumerate}
\end{lemma}

\begin{formalbox}
\textbf{Proof intuition.} For a normalized density $\phi=p'/r$, the global entropy $\Ent_r(\phi)$ equals $\KL(p'\|r)$. The key point is that conditional entropy supplies exactly the context weighting required by the masked discrepancy: on a fiber with visible context $x_V$, the conditional mean of $\phi$ is $p'_V(x_V)/r_V(x_V)$. The fiber entropy is therefore the conditional KL divergence multiplied by this ratio. Integrating against $r_V$ converts the weights to $p'_V$, recovering the fixed-mask discrepancy. Homogeneity extends the identity to unnormalized functions. Applying block AT with reference distribution $q$, followed by Pinsker's inequality, yields the modulus bound.
\end{formalbox}

\begin{proof}
\emph{Part (i).}
\emph{Step 1: normalization.}
Both sides of \Cref{eq:at-functional} are homogeneous of degree one in $\phi$: for $\lambda>0$, expanding $(\lambda t)\log(\lambda t)=\lambda\,t\log t+\lambda t\log\lambda$ shows that the $\log\lambda$ contributions cancel between the two terms of \Cref{eq:ent-def}, so $\Ent_r(\lambda\phi)=\lambda\,\Ent_r(\phi)$, and the same computation applies to every conditional entropy on the right-hand side. Moreover, if $\E_r[\phi]=0$, then $\phi$ vanishes on $\supp(r)$, both sides are $0$ by convention, and the inequality holds trivially. 

Together, these observations reduce \Cref{eq:at-functional} to functions with $\E_r[\phi]=1$. Such functions are the densities $\phi_{p'}$ of \Cref{eq:density-def}: given such a $\phi$, the distribution $p':=\phi\, r$ satisfies $\supp(p')\subseteq\supp(r)$ and $\phi=\phi_{p'}$, and the correspondence $\phi_{p'}\leftrightarrow p'$ is a bijection. 

\emph{Step 2: the global and fiber identities.}
\Cref{eq:ent-kl} applies with $\E_r[\phi_{p'}]=1$ and $p_{\phi_{p'}}=p'$, giving $\Ent_r(\phi_{p'})=\KL(p'\,\|\,r)$, finite because $\cX$ is finite and $\supp(p')\subseteq\supp(r)$.

Fix a mask $K$ and a context $x_{K^\complement}\in\supp(r_{K^\complement})$, and set the ratio $g:=p'_{K^\complement}(x_{K^\complement})/r_{K^\complement}(x_{K^\complement})$. If $p'_{K^\complement}(x_{K^\complement})=0$, then $\phi_{p'}$ vanishes on this fiber, and the conditional entropy is zero. Otherwise, for every completion $x_K$ with $r_K(x_K\mid x_{K^\complement})>0$, the density factorizes as
\[
\phi_{p'}(x_{K^\complement},x_K)
=
g\cdot
\frac{p'_K(x_K\mid x_{K^\complement})}
{r_K(x_K\mid x_{K^\complement})},
\]
where both sides vanish when $p'_K(x_K\mid x_{K^\complement})=0$. The fiber restriction of $\phi_{p'}$ thus has conditional mean $g$, and its normalization $\phi_{p'}/g$ is the density of $p'_K(\cdot\mid x_{K^\complement})$ with respect to $r_K(\cdot\mid x_{K^\complement})$. Applying \Cref{eq:ent-kl} under the conditional distribution gives
\[
\Ent_{r_K(\cdot\mid x_{K^\complement})}(\phi_{p'})
=
g\,\KL\big(p'_K(\cdot\mid x_{K^\complement})\,\big\|\,
r_K(\cdot\mid x_{K^\complement})\big),
\]
which is finite by the fiberwise domination inherited from $\supp(p')\subseteq\supp(r)$.

\emph{Step 3: averaging.}
Averaging the fiber identity over $X_{K^\complement}\sim r_{K^\complement}$, the ratio $g$ converts the averaging measure into $p'_{K^\complement}$:
\[
\E_{X_{K^\complement}\sim r_{K^\complement}}
\Big[\Ent_{r_K(\cdot\mid X_{K^\complement})}(\phi_{p'})\Big]
=
\E_{X_{K^\complement}\sim p'_{K^\complement}}
\Big[\KL\big(p'_K(\cdot\mid X_{K^\complement})\,\big\|\,
r_K(\cdot\mid X_{K^\complement})\big)\Big]
=
\mathfrak D_K(p'\,\|\, r),
\]
by \Cref{def:masked-log-discrepancy}, with contexts outside $\supp(p'_{K^\complement})$ contributing zero to both sides. Averaging over $K\sim\mu$ yields $\mathfrak D_{\mu}(p'\,\|\, r)$.

By Steps~2--3, the bijection in Step~1 identifies the two sides of \Cref{eq:at-functional} with those of \Cref{eq:at-def}. The two families of inequalities are therefore valid for the same constants $\bar C$, and their optimal constants, the infima of the respective sets of valid constants with $\inf\emptyset:=\infty$, coincide.

\emph{Part (ii).}
If $\bar C_{\AT}^{\varepsilon}(p,\mu)=\infty$, the bound holds vacuously with its right-hand side read as $+\infty$; assume therefore that it is finite. Let $q\in\cQ_{\varepsilon}^{\mu}(p)$. Because $\supp(p)\subseteq\supp(q)$, \Cref{eq:at-def} applies with $r=q$ and $p'=p$, yielding
\[
\KL(p\,\|\, q)
\;\le\;
\bar C_{\AT}(q,\mu)\,
\mathfrak D_{\mu}(p\,\|\, q)
\;\le\;
\bar C_{\AT}(q,\mu)\,\varepsilon
\;\le\;
\bar C_{\AT}^{\varepsilon}(p,\mu)\,\varepsilon.
\]
Pinsker's inequality then gives $\DTV(p,q)\le\sqrt{\tfrac12\, \bar C_{\AT}^{\varepsilon}(p,\mu)\,\varepsilon}$. This bound is independent of $q$, so taking the supremum over $q\in\cQ_{\varepsilon}^{\mu}(p)$, which is the supremum defining $\modl_{\TV}^{\mu}(p,\varepsilon)$ in \Cref{def:modulus}, completes the proof.
\end{proof}

\begin{remark}[Relation to the literature]
\label{rem:at-literature}
\citet{caputo2021block} establish entropy-functional block factorization under strong spatial-mixing conditions. Their normalization includes the coordinate-coverage factor
\(
\gamma(\mu):=\min_{i\in[N]}\sum_{K\ni i}\mu(K).
\)
With block weights $\mu(K)$, their constant $C$ yields $\bar C_{\AT}(r,\mu)\le C/\gamma(\mu)$ whenever $\gamma(\mu)>0$. \citet{li2024promises} use the conditional-KL formulation of block AT in their finite-sample recovery analysis, with the constant evaluated at the learned distribution. Their framework allows configuration-dependent masking and conditions on $(X_{K^\complement},K)$. For masks drawn independently of $X$, conditioning additionally on $K$ does not change the conditional distribution of the masked coordinates, so their conditional discrepancy reduces to $\mathfrak D_\mu(p'\|r)$. \Cref{lem:at-transfer}(ii) applies the same AT-to-TV argument uniformly over all distributions within the discrepancy budget. In the continuous setting, \citet{koehler2023statistical} use the log-Sobolev constant of the fitted distribution to convert excess score-matching risk into a bound on joint KL divergence.
\end{remark}

\medskip
\subsection{Shared-Component Mixtures}
\label{app:shared-mix}

This section provides two elementary identities for mixtures that share the same disjoint components and differ only in their mixing weights.

\begin{lemma}[Shared-component mixtures]
\label{lem:shared-mix}
Let $\cZ$ be a finite set, and let $\nu_0,\nu_1\in\Delta(\cZ)$ have disjoint supports. For $a\in[0,1]$, define $m_a:=a\nu_1+(1-a)\nu_0$. Then, for every $a,b\in[0,1]$,
\begin{equation}
\DTV(m_a,m_b)=|a-b|
\label{eq:shared-mixture-tv}
\end{equation}
and
\begin{equation}
\KL(m_a\| m_b)=\kl(a\| b),
\label{eq:shared-mixture-kl}
\end{equation}
where both sides of \Cref{eq:shared-mixture-kl} are interpreted using the usual extended-value conventions.
\end{lemma}

\begin{formalbox}
\textbf{Proof intuition.} Disjoint supports make the component label directly observable from the sample. Since both mixtures share identical component distributions, this label captures all distinguishing information. The total variation distance and KL divergence between the two mixtures therefore reduce to those between Bernoulli distributions with parameters $a$ and $b$.
\end{formalbox}

\begin{proof}
Let $S_j:=\supp(\nu_j)$ for $j\in\{0,1\}$. Since $S_0\cap S_1=\varnothing$,
\[
m_a-m_b
=
(a-b)\nu_1-(a-b)\nu_0.
\]
The two signed terms have disjoint supports, so
\[
\|m_a-m_b\|_1
=
|a-b|\bigl(\|\nu_1\|_1+\|\nu_0\|_1\bigr)
=
2|a-b|,
\]
which proves \Cref{eq:shared-mixture-tv} by substituting it into the definition of the total-variation distance (\Cref{eq:tv-defn}).

For the relative entropy, first suppose $a,b\in(0,1)$. The complement of $S_0\cup S_1$ contributes zero. Splitting the sum over the two supports gives
\begin{align}
\KL(m_a\| m_b)
&=
\sum\nolimits_{z\in S_1}
a\nu_1(z)
\log\frac{a\nu_1(z)}{b\nu_1(z)}
\nonumber
+
\sum\nolimits_{z\in S_0}
(1-a)\nu_0(z)
\log
\frac{(1-a)\nu_0(z)}{(1-b)\nu_0(z)}
\nonumber\\
&=
a\log\frac{a}{b}
+
(1-a)\log\frac{1-a}{1-b}
\nonumber\\
&=
\kl(a\| b).
\end{align}
At the boundary points $a,b\in[0,1]$, the same identity follows directly from the extended-value conventions.
\end{proof}

\medskip
\subsection{Pooled Posterior Confidence}
\label{app:pooled-confidence}

The next lemma converts the mode-wise guarantees in \Cref{ass:mode-pinning} into a high-probability statement under the pooled visible marginal: large visible contexts are rarely ambiguous about the global mode.

\begin{lemma}[Pooled posterior confidence]
\label{lem:pooled}
Suppose \Cref{ass:mode-pinning} holds, and let $c_1$ and $\kappa$ be as there. Recall $\beta_V(x_V)=p(Z=1\mid X_V=x_V)$ from \Cref{eq:mode-posterior}. For every $V\subseteq[N]$ with $|V|\ge s_{\mathrm{pin}}$, define
\begin{equation}
G_V
:=
\left\{
x_V\in\supp(p_V):
\min\{\beta_V(x_V),1-\beta_V(x_V)\}
\le e^{-\kappa|V|}
\right\}.
\label{eq:pooled-confidence-set}
\end{equation}
Then
\begin{equation*}
p_V(G_V^\complement)
\le
e^{-c_1|V|},
\end{equation*}
where the complement is taken relative to $\supp(p_V)$, i.e., $G_V^\complement := \supp(p_V)\setminus G_V$.
\end{lemma}

\begin{formalbox}
\textbf{Proof intuition.}
Under either mode, posterior ambiguity can occur only outside that mode's typical context set. Each mode assigns probability at most $e^{-c_1|V|}$ to this exceptional event. The overall mixture probability is a weighted average of these two probabilities, and thus inherits the same bound.
\end{formalbox}

\begin{proof}
Let $w_\tau:=p(\cX_\tau)$ for $\tau\in\{+,-\}$. Since $\cX=\cX_+\sqcup\cX_-$, the visible marginal decomposes as
\begin{equation}
p_V = w_+p_V^+ + w_-p_V^-.
\label{eq:visible-mixture-decomposition}
\end{equation}
For $x_V\in\cE_V^+$, \Cref{ass:mode-pinning} gives
\[
1-\beta_V(x_V)
=
p(\cX_-\mid X_V=x_V)
\le
e^{-\kappa|V|},
\]
and hence $x_V\in G_V$. Thus,
\[
\cE_V^+\subseteq G_V.
\]
Similarly, for $x_V\in\cE_V^-$,
\[
\beta_V(x_V)
=
p(\cX_+\mid X_V=x_V)
\le
e^{-\kappa|V|},
\]
so
\[
\cE_V^-\subseteq G_V.
\]
Therefore, for each $\tau\in\{+,-\}$,
\[
p_V^\tau(G_V^\complement)
\le
p_V^\tau\!\left((\cE_V^\tau)^\complement\right)
\le
e^{-c_1|V|},
\]
where the final inequality is exactly the high-probability condition in \Cref{ass:mode-pinning}. Averaging these two bounds using \Cref{eq:visible-mixture-decomposition} yields
\[
p_V(G_V^\complement)
\;=\;
w_+p_V^+(G_V^\complement)
+
w_-p_V^-(G_V^\complement)
\;\le\;
(w_++w_-)e^{-c_1|V|}
\;=\;
e^{-c_1|V|}.
\]
\end{proof}

\subsection{Cross-Mode Transition Bound}
\label{app:cross-mode-transition}

This section shows that mode pinning suppresses one-step transitions between the two mode regions when a block update retains a large visible context.

\begin{lemma}[One-step cross-mode transition bound]
\label{lem:cross-mode-transition}
Suppose \Cref{ass:mode-pinning} holds, and recall the combined rate $c=\min\{c_1,\kappa\}$. Fix a mode $\tau\in\{+,-\}$ and let $X\sim p^\tau$. Unless indicated otherwise, $\Pp$ includes randomness from the initial configuration, the heat-bath resampling, and the mask draw when the mask is random.
\begin{enumerate}
\item[\textup{(i)}] Fix a mask $K\subseteq[N]$ with $|K^{\complement}|\ge s_{\mathrm{pin}}$. Let $X'$ be a single fixed-mask heat-bath update of $X$: set $X'_{K^{\complement}}=X_{K^{\complement}}$ and draw $X'_K\sim p_K(\cdot\mid X_{K^{\complement}})$. Then
    \begin{equation}
    \Pp\!\left(X'\in\cX_{-\tau}\right)
    \le
    2e^{-c|K^{\complement}|}.
    \label{eq:component-crossing}
    \end{equation}
\item[\textup{(ii)}] Let $\widetilde X$ be a single step of the block dynamics induced by $(p,\mu)$: draw $K\sim\mu$ independently of $X$ and set $\widetilde X_{K^{\complement}}=X_{K^{\complement}}$, $\widetilde X_K\sim p_K(\cdot\mid X_{K^{\complement}})$. Then
    \begin{equation}
    \Pp\!\big(\widetilde{X}\in\cX_{-\tau}\big)
    \le
    \pi_s(\mu)+2e^{-cs},\qquad\text{for every}\;
    s\in\{s_{\mathrm{pin}},\ldots,N\}.
    \label{eq:mixed-crossing}
    \end{equation}
\item[\textup{(iii)}] In particular, if $v_{\mathrm{min}}(\mu)\ge s_{\mathrm{pin}}$, then
\begin{equation}
\Pp\!\big(\widetilde{X}\in\cX_{-\tau}\big)
\le
2e^{-c\,v_{\mathrm{min}}(\mu)}.
\label{eq:min-visible-crossing}
\end{equation}
\end{enumerate}
\end{lemma}

\begin{formalbox}
\textbf{Proof intuition.} Given the retained context on $V=K^\complement$, a heat-bath update crosses modes with probability equal to the opposite-mode posterior. Under $p^\tau$, this posterior is at most $e^{-\kappa|V|}$ on typical contexts, while atypical contexts have probability at most $e^{-c_1|V|}$. Summing these two contributions yields the fixed-mask bound. Under a random schedule, masks leaving at most $s$ coordinates visible contribute at most $\pi_s(\mu)$, while the remaining masks contribute at most $2e^{-cs}$.
\end{formalbox}

\begin{proof}
\emph{Step 1: reduction to the opposite-mode posterior.}
Fix $\tau\in\{+,-\}$ and a mask $K\subseteq[N]$ with $|K^\complement|\ge s_{\mathrm{pin}}$. Let $X\sim p^\tau$, and let $X'$ be the fixed-mask update of part~(i): $X'_{K^\complement}=X_{K^\complement}$ and $X'_K\sim p_K(\cdot\mid X_{K^\complement})$. Since the update keeps the visible coordinates fixed and redraws the masked block from $p_K(\cdot\mid X_{K^\complement})$, the updated configuration has conditional distribution
\[
X'\mid X_{K^\complement}
\sim
p(\cdot\mid X_{K^\complement}),
\]
and therefore
\begin{equation}
\Pp\!\left(
X'\in\cX_{-\tau}
\,\middle|\,
X_{K^\complement}
\right)
=
p(\cX_{-\tau}\mid X_{K^\complement}).
\label{eq:crossing-given-visible}
\end{equation}
Since $X\sim p^\tau$, its visible marginal satisfies $X_{K^\complement}\sim p^\tau_{K^\complement}$. Taking the expectation in \Cref{eq:crossing-given-visible} over $X_{K^\complement}\sim p^\tau_{K^\complement}$ gives
\[
\Pp\!\left(X'\in\cX_{-\tau}\right)
=
\E_{X_{K^\complement}\sim p^\tau_{K^\complement}}
\left[
p(\cX_{-\tau}\mid X_{K^\complement})
\right].
\]

\emph{Step 2: typical--atypical split.}
Because $|K^\complement| \ge s_{\mathrm{pin}}$, \Cref{ass:mode-pinning} applies with $V=K^\complement$. We split the expectation according to whether the visible context belongs to the typical set $\cE^\tau_{K^\complement}$:
\begin{align}
\Pp&\left(X'\in\cX_{-\tau}\right)
= \nonumber\\
&\E_{X_{K^\complement}\sim p^\tau_{K^\complement}}
\left[
p(\cX_{-\tau}\mid X_{K^\complement})
\indic\{X_{K^\complement}\in\cE^\tau_{K^\complement}\}
\right]
+
\E_{X_{K^\complement}\sim p^\tau_{K^\complement}}
\left[
p(\cX_{-\tau}\mid X_{K^\complement})
\indic\{X_{K^\complement}\notin\cE^\tau_{K^\complement}\}
\right].
\label{eq:crossing-split}
\end{align}
For the first term in \Cref{eq:crossing-split}, the definition of $\cE^\tau_{K^\complement}$ in \Cref{ass:mode-pinning} guarantees $p(\cX_{-\tau}\mid X_{K^\complement}=x_{K^\complement})\le e^{-\kappa|K^\complement|}$ for every $x_{K^\complement}\in\cE^\tau_{K^\complement}$, hence
\begin{equation}
\E_{X_{K^\complement}\sim p^\tau_{K^\complement}}
\left[
p(\cX_{-\tau}\mid X_{K^\complement})
\indic\{X_{K^\complement}\in\cE^\tau_{K^\complement}\}
\right]
\le
e^{-\kappa|K^\complement|}.
\label{eq:typical-crossing}
\end{equation}
For the second term, we use the trivial bound $p(\cX_{-\tau}\mid X_{K^\complement})\le 1$ alongside the typical-set mass bound $p^\tau_{K^\complement}(\cE^\tau_{K^\complement})\ge 1-e^{-c_1|K^\complement|}$ from \Cref{ass:mode-pinning}:
\begin{equation}
\E_{X_{K^\complement}\sim p^\tau_{K^\complement}}
\left[
p(\cX_{-\tau}\mid X_{K^\complement})
\indic\{X_{K^\complement}\notin\cE^\tau_{K^\complement}\}
\right]
\;\le\;
p^\tau_{K^\complement}\!\left((\cE^\tau_{K^\complement})^\complement\right)
\;\le\;
e^{-c_1|K^\complement|}.
\label{eq:atypical-crossing}
\end{equation}
Combining \Cref{eq:crossing-split,eq:typical-crossing,eq:atypical-crossing} yields
\[
\Pp\!\left(X'\in\cX_{-\tau}\right)
\le
e^{-\kappa|K^\complement|}
+
e^{-c_1|K^\complement|}.
\]
Since $c=\min\{c_1,\kappa\}$, both exponents are bounded by $-c|K^\complement|$. Therefore,
\[
\Pp\!\left(X'\in\cX_{-\tau}\right)
\le
2e^{-c|K^\complement|},
\]
which proves \Cref{eq:component-crossing}.

\emph{Step 3: averaging over the schedule.}
We now consider the mixed block dynamics. Draw $K\sim\mu$ independently of $X$ and let $\widetilde X$ be the induced update of part~(ii). Conditioning on the drawn mask, the law of total probability gives
\[
\Pp\!\left(\widetilde X\in\cX_{-\tau}\right)
=
\E_{K\sim\mu}
\left[
\Pp\!\left(
\widetilde X\in\cX_{-\tau}
\,\middle|\,
K
\right)
\right].
\]
Given $K$, the configuration $\widetilde X$ is the fixed-mask update of part~(i) for that mask, so the conditional crossing probability is the quantity bounded above whenever $|K^{\complement}|\ge s_{\mathrm{pin}}$. Fix $s\ge s_{\mathrm{pin}}$ and split the mask space into $\{|K^{\complement}|\le s\}$ and $\{|K^{\complement}|>s\}$. On the first event we use the trivial upper bound $1$; on the second, $|K^{\complement}|>s\ge s_{\mathrm{pin}}$, so \Cref{eq:component-crossing} applies. Hence,
\begin{equation}
\Pp\!\left(\widetilde X\in\cX_{-\tau}\right)
\le
\mu\!\left(|K^\complement|\le s\right)
+
\E_{K\sim\mu}
\left[
2e^{-c|K^{\complement}|}
\indic\{|K^{\complement}|>s\}
\right].
\label{eq:mixed-crossing-split}
\end{equation}
The first term is $\pi_s(\mu)$ by \Cref{eq:low-visibility-mass}. For the second term, on $\{|K^{\complement}|>s\}$ we have $e^{-c|K^{\complement}|}\le e^{-cs}$ since $c>0$, so
\[
\E_{K\sim\mu}
\left[
2e^{-c|K^{\complement}|}
\indic\{|K^{\complement}|>s\}
\right]
\;\le\;
2e^{-cs}
\Pp\nolimits_{K\sim\mu}\!\left(
|K^{\complement}|>s
\right)
\;\le\;
2e^{-cs}.
\]
Substituting both into \Cref{eq:mixed-crossing-split} proves \Cref{eq:mixed-crossing}.

\emph{Step 4: minimum visible size.}
Finally, suppose $v_\mathrm{min}(\mu)\ge s_{\mathrm{pin}}$. By the definition of $v_\mathrm{min}(\mu)$, $|K^{\complement}|\ge v_\mathrm{min}(\mu)\ge s_{\mathrm{pin}}$ for every $K\in\supp(\mu)$, so \Cref{eq:component-crossing} applies to every mask in the support of $\mu$. Averaging over $K\sim\mu$,
\[
\Pp\!\left(\widetilde X\in\cX_{-\tau}\right)
\;\le\;
\E_{K\sim\mu}
\left[
2e^{-c|K^{\complement}|}
\right]
\;\le\;
2e^{-c\,v_\mathrm{min}(\mu)},
\]
where the last inequality uses $|K^{\complement}|\ge v_\mathrm{min}(\mu)$. This proves \Cref{eq:min-visible-crossing}.
\end{proof}

\medskip
\subsection{Reverse Mode Blindness}
\label{app:reverse-blindness}

This section shows that the mode-blindness bound persists when the masked discrepancy is taken in the reverse KL direction.

\begin{lemma}[Mode blindness for reverse KL]
\label{lem:reverse-blindness}
Suppose \Cref{ass:mode-pinning} holds with combined rate $c=\min\{c_1,\kappa\}$, and let $B_0=\log\frac{(2-c_0)(1-c_0)}{c_0^2}$ and $A_0=e^{B_0}-1+B_0$ be the constants of the proof of \Cref{thm:mode-blind}. Set $C_2:=2e^{B_0}A_0+B_0\le16\,c_0^{-4}$. Then every mask schedule $\mu$ with $v_{\mathrm{min}}(\mu)\ge s_{\mathrm{pin}}$ satisfies
\begin{equation}
\sup\nolimits_{\lambda\in I_{\mathrm{bl}}}
\mathfrak D_\mu(q_\lambda\,\|\,p)
\;\le\;
C_2\,\E_{K\sim\mu}\big[e^{-c|K^\complement|}\big]
\;\le\;
C_2\, e^{-c\,v_{\mathrm{min}}(\mu)}.
\end{equation}
\end{lemma}

\begin{formalbox}
\textbf{Proof intuition.}
Reversing the KL divergence changes both the order of the posterior comparison and the distribution used to average contexts. A bounded logit shift preserves exponential posterior confidence, so the reversed conditional KL remains exponentially small on the confidence set. The complement has exponentially small probability under each component and therefore remains negligible after reweighting. The reverse discrepancy consequently retains the same exponential decay rate.
\end{formalbox}

\begin{proof}
Fix $\lambda\in I_{\mathrm{bl}}$ and $K\in\supp(\mu)$. Since $v_{\min}(\mu)\ge s_{\mathrm{pin}}$, every visible set below satisfies $|K^\complement|\ge s_{\mathrm{pin}}$.

\emph{Step 1: binary reduction.}
Since $\lambda\in(0,1)$, we have $\supp(q_\lambda)=\supp(p)$, so the corresponding marginals and conditionals share the same support. For any $x_{K^\complement}\in\supp(p_{K^\complement})$ such that $\beta_{K^\complement}\in(0,1)$, both masked-block conditionals act as mixtures of the same disjointly supported components $\nu_\pm$. As established in the proof of \Cref{lem:logit-shift}(iii), these mixtures have weights $(\beta_{K^\complement},1-\beta_{K^\complement})$ and $(\beta_{\lambda,K^\complement},1-\beta_{\lambda,K^\complement})$. Applying \Cref{lem:shared-mix} in either direction yields
\[
\KL\big(q_{\lambda,K}(\cdot\mid x_{K^\complement})\,\big\|\,p_K(\cdot\mid x_{K^\complement})\big)
=\kl\big(\beta_{\lambda,K^\complement}\,\big\|\,\beta_{K^\complement}\big).
\]
Furthermore, the boundary fibers where $\beta_{K^\complement}\in\{0,1\}$ contribute zero to the divergence since $\beta_{\lambda,K^\complement}=\beta_{K^\complement}$ in those cases. It then follows that 
\[
\mathfrak D_K(q_\lambda\|p)=\E_{X_{K^\complement}\sim q_{\lambda,K^\complement}}[\kl(\beta_{\lambda,K^\complement}(X_{K^\complement})\|\beta_{K^\complement}(X_{K^\complement}))].
\]

\emph{Step 2: fiberwise bound.} On the interior fibers, \Cref{lem:logit-shift}(ii) gives that $\logit(\beta_{K^\complement})=\logit(\beta_{\lambda,K^\complement})-\Delta_\lambda$. Recall $\Delta_\lambda:=\logit(\lambda)-\logit(w)$. Since $w\in[c_0,1-c_0]$ and $\lambda\in I_{\mathrm{bl}}=[c_0/2,1-c_0/2]$, we have $|\Delta_\lambda|\le B_0$. 

By applying the fiberwise bound from Step~3 of the proof of \Cref{thm:mode-blind} with a base point of $\beta_{\lambda,K^\complement}$ and a shift of $-\Delta_\lambda$, we obtain
\[
\kl(\beta_{\lambda,K^\complement}\|\beta_{K^\complement})
\le\min\big\{B_0,\;
A_0\min\{\beta_{\lambda,K^\complement},1-\beta_{\lambda,K^\complement}\}\big\}.
\]
A logit shift of size at most $B_0$ increases the smaller posterior probability by at most a factor $2e^{B_0}$. To see this, write $o:=\beta_{K^\complement}/(1-\beta_{K^\complement})$ and $o_\lambda:=\beta_{\lambda,K^\complement}/(1-\beta_{\lambda,K^\complement})$ for the two posterior odds, and note the elementary sandwich, valid for every $\beta\in(0,1)$ with odds $o$,
\[
\min\{\beta,\,1-\beta\}
\;\le\;
\min\{o,\,1/o\}
\;\le\;
2\min\{\beta,\,1-\beta\},
\]
which follows, for $\beta\le\tfrac12$, from $\beta\le\beta/(1-\beta)\le2\beta$, and by the symmetry $(\beta,o)\mapsto(1-\beta,1/o)$ otherwise. Since $\logit(\beta_{\lambda,K^\complement}) =\logit(\beta_{K^\complement})+\Delta_\lambda$, the odds transform as $o_\lambda=e^{\Delta_\lambda}o$, whence 
\[
\min\{o_\lambda,1/o_\lambda\}\le e^{|\Delta_\lambda|}\min\{o,1/o\}
\le e^{B_0}\min\{o,1/o\}.
\]
Chaining the three bounds,
\[
\min\{\beta_{\lambda,K^\complement},\,1-\beta_{\lambda,K^\complement}\}
\;\le\;
\min\{o_\lambda,\,1/o_\lambda\}
\;\le\;
e^{B_0}\min\{o,\,1/o\}
\;\le\;
2e^{B_0}\min\{\beta_{K^\complement},\,1-\beta_{K^\complement}\}.
\]

\emph{Step 3: averaging.} Let $G_{K^\complement}$ denote the pooled confidence set defined in \Cref{lem:pooled}. Within $G_{K^\complement}$, the fiberwise bound evaluates to at most $2e^{B_0}A_0\,e^{-\kappa|K^\complement|}$, while outside $G_{K^\complement}$ we rely on the uniform bound $B_0$. The proof of \Cref{lem:pooled} gives
\[
p^\tau_{K^\complement}(G_{K^\complement}^\complement)
\le e^{-c_1|K^\complement|}
\qquad\text{for each }\tau\in\{+,-\}.
\]
Since
$q_{\lambda,K^\complement}
=\lambda p^+_{K^\complement}
+(1-\lambda)p^-_{K^\complement}$,
\[
\begin{aligned}
q_{\lambda,K^\complement}(G_{K^\complement}^\complement)
&=
\lambda p^+_{K^\complement}(G_{K^\complement}^\complement)
+(1-\lambda)
p^-_{K^\complement}(G_{K^\complement}^\complement)
\\
&\le e^{-c_1|K^\complement|}.
\end{aligned}
\] 
Combining these observations gives
\[
\mathfrak D_K(q_\lambda\|p)
\le 2e^{B_0}A_0\, e^{-\kappa|K^\complement|}+B_0\, e^{-c_1|K^\complement|}
\le C_2\, e^{-c|K^\complement|}.
\]
Taking the average over $K\sim\mu$ establishes the main claim. Finally, to derive the numerical bound, we let $x:=\frac{(2-c_0)(1-c_0)}{c_0^2}\le2c_0^{-2}$. By applying the inequality $\log x\le x-1$, we conclude that
\[
C_2=2x(x-1+\log x)+\log x\le(x-1)(4x+1)\le4x^2\le16\,c_0^{-4}.
\]
Since this bound is uniform in $\lambda\in I_{\mathrm{bl}}$, averaging over $K\sim\mu$ and taking the supremum over $\lambda$ gives the claimed schedule bound.
\end{proof}

\subsection{Factorized Per-Token Losses}
\label{app:factorized}

\Cref{def:masked-log-discrepancy} scores the \emph{joint} masked-block conditional, whereas BERT-style pretraining sums \emph{per-token} losses over the masked block \citep{devlin2019bert,liu2019roberta}, a masked analog of pseudolikelihood \citep{besag1974spatial}. For $q\in\mathcal{Q}(p)$, a mask $K$, and $i\in K$, write $q_i(\cdot\mid x_{K^\complement})$ for the coordinate-$i$ conditional marginal, and define
\begin{equation}\label{eq:fac-discrepancy}
\mathfrak D^{\mathrm{fac}}_K(p\|q)
:= \E_{X_{K^\complement}\sim p_{K^\complement}}
\sum_{i\in K}\KL\big({p_i(\cdot\mid X_{K^\complement})}\|{q_i(\cdot\mid X_{K^\complement})}\big),
\quad
\mathfrak D^{\mathrm{fac}}_\mu(p\|q) := \E_{K\sim\mu}\,\mathfrak D^{\mathrm{fac}}_K(p\|q).
\end{equation}

\begin{corollary}[Mode blindness under per-token log loss]
\label{lem:fac-blindness} 
Suppose \Cref{ass:mode-pinning} holds with combined rate $c$, and let $C\le 6c_0^{-2}$ be the constant of \Cref{thm:mode-blind}. Then every mask schedule $\mu$ with $v_{\min}(\mu)\ge s_{\mathrm{pin}}$ satisfies
\begin{equation}
\label{eq:fac-blindness}
\sup\nolimits_{\lambda\in I_{\mathrm{bl}}} \mathfrak D^{\mathrm{fac}}_\mu(p\|q_\lambda)
\;\le\; C\,N\,\E_{K\sim\mu}\!\big[e^{-c|K^\complement|}\big]
\;\le\; C\,e^{-c\,v_{\min}(\mu)+\log N}.
\end{equation}
\end{corollary}

\begin{formalbox}
\textbf{Proof intuition.} Projecting a masked block onto one coordinate cannot increase KL divergence. The joint masked-block discrepancy therefore bounds each per-token discrepancy, and summing over the masked coordinates introduces a factor of at most $N$. When the visible count grows linearly with $N$, this exponential decay easily absorbs the factor $N$, keeping the overall bound exponentially small.
\end{formalbox}

\begin{proof}
Fix $\lambda\in I_{\mathrm{bl}}$, $K\in\supp(\mu)$, and $x_{K^\complement}\in\supp(p_{K^\complement})$. The coordinate projection $x_K\mapsto x_i$ is a stochastic map, so the data-processing inequality and \Cref{lem:logit-shift}(iii) give, for every $i\in K$, 
\begin{align*}
\KL\big({p_i(\cdot\mid x_{K^\complement})}\|{q_{\lambda,i}(\cdot\mid x_{K^\complement})}\big)
&\le
\KL\big({p_K(\cdot\mid x_{K^\complement})}\|{q_{\lambda,K}(\cdot\mid x_{K^\complement})}\big)\\
&= \kl\big({\beta_{K^\complement}(x_{K^\complement})}\|{\beta_{\lambda,K^\complement}(x_{K^\complement})}\big).
\end{align*}
Summing over $i\in K$ and using $|K|\le N$, then averaging over contexts as in \Cref{eq:per-mask-mode-blind-bound} of the proof of \Cref{thm:mode-blind} and over $K\sim\mu$, proves \Cref{eq:fac-blindness}.
\end{proof}

\paragraph{Mode information under factorization.}
If
\(
v_{\mathrm{min}}(\mu)
\ge
\max\{s_{\mathrm{pin}},\,2c^{-1}\log N\},
\)
then \Cref{eq:fac-blindness} gives
\[
\sup\nolimits_{\lambda\in I_{\mathrm{bl}}}
\mathfrak D^{\mathrm{fac}}_\mu(p\|q_\lambda)
\le
C e^{-c\,v_{\mathrm{min}}(\mu)/2}.
\]
Since $\DTV(p,q_\lambda)=|w-\lambda|$, the same reweighted alternatives give a nonvanishing lower bound on the modulus defined through $\mathfrak D^{\mathrm{fac}}_\mu$, as in \Cref{eq:mode-blindness-modulus}. In particular, when $v_{\mathrm{min}}(\mu)=\Omega(N)$, this occurs at an excess-risk tolerance exponentially small in $N$. The corresponding low-visibility mode-weight recovery bound does not follow from the same comparison. Although the fiber components $\nu_\pm$ have disjoint supports, their coordinate marginals $\nu_{\pm,i}$ may overlap. Coordinate projection can therefore discard mode information, and data processing provides only an upper bound on the per-token KL divergence. At the full mask, for example,
\[
\DTV(p_i,q_{\lambda,i})
=
|w-\lambda|\,\DTV(p_i^+,p_i^-).
\]
Distribution-free joint-distribution control from the full-mask term also fails under factorization. That term scores only the unconditional coordinate marginals. In particular, the product distribution
\(
q^{\mathrm{prod}}:=\bigotimes_{i=1}^{N}p_i
\)
belongs to $\cQ(p)$ and satisfies
\[
\mathfrak D^{\mathrm{fac}}_{[N]}(p\|q^{\mathrm{prod}})=0.
\]
Indeed, $p(x)>0$ implies $p_i(x_i)>0$ for every coordinate, so $q^{\mathrm{prod}}$ covers the support of $p$. For the construction in \Cref{prop:product-instantiation}, with $\theta\in(0,1)$ and $w\in(0,1)$ fixed as $N\to\infty$,
\[
\DTV(p,q^{\mathrm{prod}})=1-o(1).
\]
To see this, $p$ concentrates near empirical magnetizations $\pm\theta$, whereas $q^{\mathrm{prod}}$ concentrates near $(2w-1)\theta$, which lies strictly between them. When $w=1/2$, $q^{\mathrm{prod}}$ is the uniform distribution. The full-mask joint-KL control of \Cref{prop:full-mask} therefore relies on scoring the joint masked-block distribution.

\subsection{A Sufficient Condition for \texorpdfstring{\Cref{ass:low-vis-uncertainty}}{Asm. 4.2}}
\label{app:low-vis-sufficient}

\begin{lemma}[Bounded visible evidence implies residual uncertainty]
\label{lem:low-vis-sufficient}
Fix $L_0>0$ and $V\subseteq[N]$, and suppose the mode-conditional visible marginals satisfy
\begin{equation}
\left|\log\frac{p_V^+(x_V)}{p_V^-(x_V)}\right|\le L_0|V|
\qquad\text{for all }x_V\in\supp(p_V).
\label{eq:bounded-llr}
\end{equation}
Let $w:=p(\cX_+)\in(0,1)$. Then
\begin{equation}
\mmse_p(Z\mid X_V)
\ge
w(1-w)\,e^{-L_0|V|}.
\label{eq:mmse-lower-sufficient}
\end{equation}
In particular, \Cref{ass:low-vis-uncertainty} holds on any scale $s_{\mathrm{ov}}\ge s_{\mathrm{pin}}$ for which \Cref{eq:bounded-llr} holds for every $|V|\le s_{\mathrm{ov}}$ with the same constant $L_0$, taking $u_0=w(1-w)$ and $L=L_0$.
\end{lemma}

\begin{formalbox}
\textbf{Proof intuition.} Bayes' rule expresses the posterior odds as the prior odds multiplied by the visible-context likelihood ratio. Bounding this likelihood ratio therefore limits how concentrated the mode posterior can become. The resulting pointwise lower bound on posterior variance yields the MMSE bound after averaging over visible contexts.
\end{formalbox}

\begin{proof}
Fix $x_V\in\supp(p_V)$ and write $\ell:=\log\big(p_V^+(x_V)/p_V^-(x_V)\big)$, so $|\ell|\le L_0|V|$ by \Cref{eq:bounded-llr}. By Bayes' rule with prior $p(\cX_+)=w$,
\[
\beta_V(x_V)
=
\frac{w\,p_V^+(x_V)}{w\,p_V^+(x_V)+(1-w)\,p_V^-(x_V)}
=
\frac{w\,e^{\ell}}{w\,e^{\ell}+(1-w)},
\]
where $p_V^-(x_V)>0$ follows from the finite log-likelihood-ratio bound in \Cref{eq:bounded-llr}. Hence
\[
\beta_V(x_V)\big(1-\beta_V(x_V)\big)
=
\frac{w(1-w)\,e^{\ell}}{\big(w\,e^{\ell}+(1-w)\big)^2}.
\]
We upper-bound the denominator by splitting on the sign of $\ell$. If $\ell \ge 0$, then $w e^{\ell} + (1-w) \le w e^{\ell} + (1-w)e^{\ell} = e^{\ell}$; if $\ell < 0$, then $w e^{\ell} + (1-w) \le w + (1-w) = 1$. Writing $\ell^{+} := \max\{\ell, 0\}$, both cases are subsumed by
\[
  w e^{\ell} + (1-w) \;\le\; e^{\ell^{+}} .
\]
Therefore
\begin{align*}
\beta_V(x_V)\bigl(1-\beta_V(x_V)\bigr)
&\ge
w(1-w)e^{\ell-2\ell^+}\\
&=
w(1-w)e^{-|\ell|}\\
&\ge
w(1-w)e^{-L_0|V|},
\end{align*}
where the last step uses $|\ell| \le L_0|V|$ from \Cref{eq:bounded-llr}. 

Taking $\mathbb{E}_{X_V \sim p_V}$ of both sides and recalling $\mmse_p(Z \mid X_V) = \mathbb{E}[\beta_V(1-\beta_V)]$ proves \Cref{eq:mmse-lower-sufficient}.
\end{proof}

\bigskip
\section{Proofs of Statements}
\label{app:proofs}


\subsection{\texorpdfstring{Proof of \Cref{lem:logit-shift}}{Proof of Lemma 4.1}}
\label{app:proof-logit-shift}

\begin{proof}
Fix $x_V\in\supp(p_V)$ and abbreviate $\beta_V=\beta_V(x_V)$ and $\beta_{\lambda,V}=\beta_{\lambda,V}(x_V)$.

\emph{Step 1: density ratio and the reweighted posterior.}
The reweighted distribution satisfies
\begin{equation}
q_\lambda(x)
=
\begin{cases}
\displaystyle \frac{\lambda}{w}\,p(x),
& x\in\cX_+,\\[0.8em]
\displaystyle \frac{1-\lambda}{1-w}\,p(x),
& x\in\cX_-.
\end{cases}
\label{eq:reweighted-density}
\end{equation}
Summing \Cref{eq:reweighted-density} over completions in the two mode regions gives
\begin{align}
q_\lambda(\cX_+,X_V=x_V)
&=
\frac{\lambda}{w}\,
p(\cX_+,X_V=x_V),
\label{eq:reweighted-positive-joint}\\
q_\lambda(\cX_-,X_V=x_V)
&=
\frac{1-\lambda}{1-w}\,
p(\cX_-,X_V=x_V).
\label{eq:reweighted-negative-joint}
\end{align}
By Bayes' rule, dividing the numerator and denominator by $p_V(x_V)>0$ yields \Cref{eq:posterior-reweighting}:
\[
\beta_{\lambda,V}
=
\frac{q_\lambda(\cX_+, X_V=x_V)}{q_\lambda(\cX_+, X_V=x_V)+q_\lambda(\cX_-, X_V=x_V)}
=
\frac{\frac{\lambda}{w}\,\beta_V}
{\frac{\lambda}{w}\,\beta_V
+
\frac{1-\lambda}{1-w}\,\big(1-\beta_V\big)}.
\]

\emph{Step 2: the logit shift.}
If $\beta_V\in(0,1)$, both terms in the denominator of \Cref{eq:posterior-reweighting} are positive. Taking posterior odds gives
\[
\frac{\beta_{\lambda,V}}{1-\beta_{\lambda,V}}
=
\frac{\lambda}{1-\lambda}\,
\frac{1-w}{w}\,
\frac{\beta_V}{1-\beta_V}.
\]
Taking logarithms proves \Cref{eq:posterior-logit-shift}:
\[
\logit\big(\beta_{\lambda,V}(x_V)\big)
=
\logit\big(\beta_V(x_V)\big)
+
\logit(\lambda)
-
\logit(w).
\]

\emph{Step 3: fiber mixtures and the binary reduction.}
Continue with the case $\beta_V\in(0,1)$. Both conditional component distributions
\[
\nu_\tau
:=
p\!\left(
X_{V^\complement}\in\cdot
\,\middle|\,
X_V=x_V,\ X\in\cX_\tau
\right),
\qquad
\tau\in\{+,-\},
\]
are well defined. Reweighting changes only the prior masses of the two mode regions, so the same conditional component distributions appear under $q_\lambda$. Hence
\begin{align}
p_{V^\complement}(\cdot\mid x_V)
&=
\beta_V\,\nu_+
+
(1-\beta_V)\,\nu_-,
\label{eq:p-fiber-mixture}\\
q_{\lambda,V^\complement}(\cdot\mid x_V)
&=
\beta_{\lambda,V}\,\nu_+
+
(1-\beta_{\lambda,V})\,\nu_-.
\label{eq:q-fiber-mixture}
\end{align}
The supports of $\nu_+$ and $\nu_-$ are disjoint. Indeed, for fixed $x_V$, any completion $x_{V^\complement}$ determines a unique full configuration $(x_V,x_{V^\complement})$, which belongs to at most one of the disjoint regions $\cX_+$ and $\cX_-$. Applying the KL identity of \Cref{lem:shared-mix} to the mixtures \Cref{eq:p-fiber-mixture,eq:q-fiber-mixture} proves \Cref{eq:conditional-kl-reweighting} when $\beta_V\in(0,1)$:
\[
\KL\!\left(
p_{V^\complement}(\cdot\mid x_V)
\,\|\,
q_{\lambda,V^\complement}(\cdot\mid x_V)
\right)
=
\kl(\beta_V\,\|\,\beta_{\lambda,V}).
\]

\emph{Step 4: boundary posteriors.}
Finally, if $\beta_V=1$, then $p(\cX_-,X_V=x_V)=0$. By \Cref{eq:reweighted-negative-joint}, the same is true under $q_\lambda$, so $\beta_{\lambda,V}=1$. Moreover, both conditional block distributions equal the within-$\cX_+$ conditional distribution. The case $\beta_V=0$ is symmetric. Thus the conditional divergence in \Cref{eq:conditional-kl-reweighting} vanishes, completing the proof.
\end{proof}

\medskip
\subsection{\texorpdfstring{Proof of \Cref{prop:info-decomp}}{Proof of Prop. 4.1}}
\label{app:proof-info-decomp}

\begin{formalbox}
\textbf{Proof intuition.} Because the two distributions share their within-mode components, their joint KL divergence depends only on the mode weights. The KL chain rule decomposes this divergence into a visible-marginal term and a masked-conditional term. The masked objective therefore measures only the portion of the mode-weight discrepancy that remains unresolved by the visible coordinates.
\end{formalbox}

\begin{proof}
Fix a mask $K\subseteq[N]$. Because $\lambda\in(0,1)$, the model distribution satisfies $\supp(q_\lambda)=\supp(p)$, ensuring that the conditional $q_{\lambda,K}(\cdot\mid x_{K^\complement})$ is well defined for all $x_{K^\complement}\in\supp(p_{K^\complement})$.

By \Cref{lem:logit-shift}(iii), the pointwise divergence between the masked conditionals satisfies
\[
\KL\!\left(
p_K(\cdot\mid x_{K^\complement})
\,\|\,
q_{\lambda,K}(\cdot\mid x_{K^\complement})
\right)
=
\kl\big(\beta_{K^\complement}(x_{K^\complement})\,\|\,\beta_{\lambda,K^\complement}(x_{K^\complement})\big).
\]
Taking the expectation over $X_{K^\complement}\sim p_{K^\complement}$ directly yields \Cref{eq:decomp-binary} by the definition of the masked log discrepancy $\mathfrak D_K(p\| q_\lambda)$ (\Cref{def:masked-log-discrepancy}).

To establish \Cref{eq:decomp-chain}, decompose the joint divergence along $(X_{K^\complement},X_K)$ by the KL chain rule:
\begin{equation}
\KL(p\| q_\lambda)
=
\KL(p_{K^\complement}\| q_{\lambda,K^\complement})
+
\E_{X_{K^\complement}\sim p_{K^\complement}}\!\left[
\KL\!\left(
p_K(\cdot\mid X_{K^\complement})
\,\middle\|\,
q_{\lambda,K}(\cdot\mid X_{K^\complement})
\right)
\right],
\label{eq:kl-chain-rule-disc}
\end{equation}
where the conditional term is $\mathfrak D_K(p\|q_\lambda)$ by \Cref{def:masked-log-discrepancy}.

The mode distributions $p^+$ and $p^-$ have disjoint supports, so $p$ and $q_\lambda$ are shared-component mixtures in the sense of \Cref{lem:shared-mix}, with weights $w$ and $\lambda$:
\begin{equation}
\KL(p\| q_\lambda)
=
\kl(w\|\lambda).
\label{eq:joint-kl-binary}
\end{equation}
Substituting \Cref{eq:joint-kl-binary} into \Cref{eq:kl-chain-rule-disc} and rearranging gives
\[
\mathfrak D_K(p\|q_\lambda)
=
\kl(w\|\lambda)
-
\KL(p_{K^\complement}\|q_{\lambda,K^\complement}),
\]
which proves \Cref{eq:decomp-chain}.
\end{proof}

\medskip
\subsection{\texorpdfstring{Proof of \Cref{thm:mode-blind}}{Proof of Thm. 5.1}}
\label{app:proof-mode-blind}

\begin{formalbox}
\textbf{Proof intuition.} Reweighting shifts the posterior log-odds by a bounded amount. The resulting conditional KL divergence is controlled by the smaller posterior mode probability. Mode pinning makes this probability exponentially small on typical visible contexts, while atypical contexts carry exponentially small probability. Averaging over contexts yields the masked-discrepancy upper bound. The global mode weights can nevertheless differ by a constant amount, yielding the total-variation lower bound.
\end{formalbox}

\begin{proof}
\emph{Step 1: uniform bound on the logit shift.}
Recall $I_\mathrm{bl}=[c_0/2, 1-c_0/2]$ from the theorem statement. Under \Cref{ass:mode-pinning}, the true weight satisfies $w\in[c_0,1-c_0]$. Since $\logit$ is increasing and $\logit(1-x)=-\logit(x)$, the posterior log-odds shift $\Delta_\lambda:=\logit(\lambda)-\logit(w)$ satisfies
\[
|\Delta_\lambda|
\le
\logit(1-\tfrac{c_0}{2})-\logit(c_0)
=
\log\frac{(2-c_0)(1-c_0)}{c_0^2}
=:B_0
\]
for every $\lambda\in I_{\mathrm{bl}}$. This bound depends only on $c_0$. Set $A_0:=e^{B_0}-1+B_0$.

\emph{Step 2: reduction to the binary divergence.}
Fix $\lambda\in I_\mathrm{bl}$ and a mask $K\in\supp(\mu)$. Since $v_{\mathrm{min}}(\mu)\ge s_{\mathrm{pin}}$, we have $|K^\complement|\ge s_{\mathrm{pin}}$. For every $x_{K^\complement} \in \supp(p_{K^\complement})$, abbreviating $\beta_{K^\complement}=\beta_{K^\complement}(x_{K^\complement})$ and $\beta_{\lambda,{K^\complement}}=\beta_{\lambda,{K^\complement}}(x_{K^\complement})$, \Cref{lem:logit-shift}(iii) provides
\begin{equation}
\KL\!\left( p_K(\cdot\mid x_{K^\complement}) \| q_{\lambda,K}(\cdot\mid x_{K^\complement}) \right) = \kl(\beta_{K^\complement}\|\beta_{\lambda,{K^\complement}}).
\label{eq:mode-blind-binary-reduction}
\end{equation}

\emph{Step 3: fiberwise divergence bounds.}
We bound this binary divergence in two regimes. For $\beta_{K^\complement}\in(0,1)$, \Cref{lem:logit-shift}(ii) gives $\logit(\beta_{\lambda,{K^\complement}}) = \logit(\beta_{K^\complement})+\Delta_\lambda$, which algebraically rearranges to
\begin{equation*}
\kl(\beta_{K^\complement}\|\beta_{\lambda,{K^\complement}}) = \log\!\left( 1-\beta_{K^\complement}+\beta_{K^\complement} e^{\Delta_\lambda} \right) - \beta_{K^\complement}\Delta_\lambda.
\end{equation*}
We use two upper bounds for this expression. First, since $1-\beta_{K^\complement} +\beta_{K^\complement}e^{\Delta_\lambda} \le e^{\max\{0,\Delta_\lambda\}}$, we obtain the uniform bound
\[
\kl(\beta_{K^\complement}\|
     \beta_{\lambda,K^\complement})
\le
\max\{0,\Delta_\lambda\}
-\beta_{K^\complement}\Delta_\lambda
\le
|\Delta_\lambda|
\le B_0.
\]
Second, applying $\log(1+t)\le t$ yields 
\[
\kl(\beta_{K^\complement}\|\beta_{\lambda,{K^\complement}}) \le \beta_{K^\complement}(e^{|\Delta_\lambda|}-1+|\Delta_\lambda|) \le A_0\beta_{K^\complement}.
\]
By symmetry, replacing $(\beta_{K^\complement},\Delta_\lambda)$ with $(1-\beta_{K^\complement},-\Delta_\lambda)$ in the same argument yields the complementary bound $A_0(1-\beta_{K^\complement})$. Combining these properties gives
\begin{equation}
\kl(\beta_{K^\complement}\|\beta_{\lambda,{K^\complement}}) \le \min\big\{ B_0,\, A_0\min\{\beta_{K^\complement},1-\beta_{K^\complement}\} \big\}.
\label{eq:binary-kl-confident-new}
\end{equation}
The same bound holds when $\beta_{K^\complement}\in\{0,1\}$, since $\beta_{\lambda,K^\complement}=\beta_{K^\complement}$ and the divergence is zero.

\emph{Step 4: averaging over contexts and masks.}
Let $G_{K^\complement}$ denote the pooled confidence set from \Cref{lem:pooled}, which applies because $|K^\complement|\ge s_{\mathrm{pin}}$. On $G_{K^\complement}$, the residual uncertainty is exponentially small: 
\[
\min\{\beta_{K^\complement}(X_{K^\complement}),1-\beta_{K^\complement}(X_{K^\complement})\} \le e^{-\kappa|K^\complement|}.
\]
On the complement $G_{K^\complement}^\complement$, we use the uniform bound $B_0$, noting that $p_{K^\complement}(G_{K^\complement}^\complement)\le e^{-c_1|K^\complement|}$. Splitting the expectation of \Cref{eq:binary-kl-confident-new} over $X_{K^\complement}\sim p_{K^\complement}$ accordingly yields
\begin{align}
\E_{X_{K^\complement}\sim p_{K^\complement}} \left[ \kl\!\left( \beta_{K^\complement}(X_{K^\complement}) \| \beta_{\lambda,{K^\complement}}(X_{K^\complement}) \right) \right]
&\le A_0 e^{-\kappa|K^\complement|} + B_0\,p_{K^\complement}(G_{K^\complement}^\complement) \nonumber\\
&\le (A_0+B_0)e^{-c|K^\complement|},\qquad\text{where }\ c=\min\{c_1,\kappa\}.
\label{eq:per-mask-mode-blind-bound}
\end{align}

Define $C:=A_0+B_0$. We can write $C = x-1+2\log x$ where $x:=\frac{(2-c_0)(1-c_0)}{c_0^{2}}$. Since $\log x\le x-1$ and $x \le 2c_0^{-2}$, it follows immediately that $C \le 3(x-1) \le 6c_0^{-2}$.

By \Cref{def:masked-log-discrepancy} and \Cref{eq:mode-blind-binary-reduction}, the masked discrepancy $\mathfrak D_\mu(p\| q_\lambda)$ is the average of the left-hand side of \Cref{eq:per-mask-mode-blind-bound} over $K\sim\mu$. Therefore, taking the supremum over $\lambda\in I_\mathrm{bl}$ gives
\begin{equation*}
\sup\nolimits_{\lambda\in I_\mathrm{bl}} \mathfrak D_\mu(p\| q_\lambda) \le C\,\E_{K\sim\mu} \big[ e^{-c|K^{\complement}|} \big] \le C\,e^{-c\,v_{\mathrm{min}}(\mu)},
\end{equation*}
which proves \Cref{eq:mode-blindness}.

\emph{Step 5: identifiability modulus lower bound.} 
Since the reweighted models satisfy $q_\lambda \in \cQ(p)$, any tolerance
\[
\varepsilon\ge C\,\E_{K\sim\mu}[e^{-c|K^{\complement}|}]
\]
makes the entire family \(\{q_\lambda\}_{\lambda\in I_{\mathrm{bl}}}\) admissible in the supremum defining \(\modl_{\TV}^{\mu}(p,\varepsilon)\). Because $p$ and $q_\lambda$ share the same disjointly supported components, \Cref{lem:shared-mix} yields $\DTV(p,q_\lambda)=|w-\lambda|$. Using the structural bound $w\in[c_0,1-c_0]\subseteq I_\mathrm{bl}$, we explicitly evaluate the supremum over the interval:
\begin{align*}
\modl_{\TV}^\mu(p,\varepsilon) \;\ge\; \sup\nolimits_{\lambda\in I_\mathrm{bl}} |w-\lambda| &= \max\{w-c_0/2,\,1-c_0/2-w\} \\
&\ge \frac{(w-c_0/2) + (1-c_0/2-w)}{2} = \frac{1-c_0}{2}.
\end{align*}
This proves \Cref{eq:mode-blindness-modulus} and completes the proof.
\end{proof}

\medskip
\subsection{Proof of \texorpdfstring{\Cref{cor:at-blowup}}{Cor. 5.1}}
\label{app:proof-at-blowup}

\begin{formalbox}
\textbf{Proof intuition.} An AT inequality controls the joint KL divergence by a constant times the masked discrepancy. The chosen reweighting keeps the joint divergence bounded away from zero while making the masked discrepancy exponentially small, so the AT constant must grow exponentially. The forward comparison lower-bounds the constant for $q_\lambda$, while the reverse comparison lower-bounds the constant for $p$.
\end{formalbox}

\begin{proof}
\emph{Step 1: reweighting and joint separation.}
Pick $\lambda\in\{c_0/2,\,1-c_0/2\}$ maximizing $|w-\lambda|$. Since \Cref{ass:mode-pinning} gives $w\in[c_0,1-c_0]$, as in Step~5 of the proof of \Cref{thm:mode-blind},
\[
|w-\lambda|
=\max\{w-c_0/2,\;1-c_0/2-w\}
\ge\frac{(w-c_0/2)+(1-c_0/2-w)}{2}
=\frac{1-c_0}{2}.
\]
Because $p$ and $q_\lambda$ share the same disjointly supported components, \Cref{lem:shared-mix} gives the binary KL identities. Pinsker's inequality then yields
\[
\KL(p\|q_\lambda)=\kl(w\|\lambda)\ge2(w-\lambda)^2\ge\tfrac{(1-c_0)^2}{2},
\qquad
\KL(q_\lambda\|p)=\kl(\lambda\|w)\ge\tfrac{(1-c_0)^2}{2}.
\]
The equality $\supp(p)=\supp(q_\lambda)$ makes both applications of \Cref{eq:at-def} below admissible. In each case, the AT constant is bounded below by the ratio of the joint divergence to the masked discrepancy, with a positive numerator divided by zero interpreted as $+\infty$.

\emph{Step 2: model-side bound.} \Cref{thm:mode-blind} gives $\mathfrak D_\mu(p\|q_\lambda)\le C\,e^{-c\,v_{\mathrm{min}}(\mu)}$ with $C\le6c_0^{-2}$. Taking $r=q_\lambda$, $p'=p$ in \Cref{eq:at-def},
\[
\bar C_{\AT}(q_\lambda,\mu)
\;\ge\;
\frac{\KL(p\|q_\lambda)}{\mathfrak D_\mu(p\|q_\lambda)}
\;\ge\;
\frac{(1-c_0)^2/2}{6c_0^{-2}}\;e^{c\,v_{\mathrm{min}}(\mu)}
\;=\;
\frac{c_0^2(1-c_0)^2}{12}\;e^{c\,v_{\mathrm{min}}(\mu)}.
\]

\emph{Step 3: data-side bound.} \Cref{lem:reverse-blindness} gives $\mathfrak D_\mu(q_\lambda\|p)\le C_2\,e^{-c\,v_{\mathrm{min}}(\mu)}$ with $C_2\le16\,c_0^{-4}$. Taking $r=p$, $p'=q_\lambda$ in
\Cref{eq:at-def},
\[
\bar C_{\AT}(p,\mu)
\;\ge\;
\frac{\KL(q_\lambda\|p)}{\mathfrak D_\mu(q_\lambda\|p)}
\;\ge\;
\frac{(1-c_0)^2/2}{16c_0^{-4}}\;e^{c\,v_{\mathrm{min}}(\mu)}
\;=\;
\frac{c_0^{4}(1-c_0)^2}{32}\;e^{c\,v_{\mathrm{min}}(\mu)},
\]
completing the proof.
\end{proof}

\medskip
\subsection{\texorpdfstring{Proof of \Cref{lem:sensitivity}}{Proof of Lemma 5.1}}
\label{app:proof-sensitivity}

\begin{formalbox}
\textbf{Proof intuition.} The prior and each context-conditional posterior undergo the same log-odds shift. Along this shift, the second derivative of the binary KL divergence is the Bernoulli variance. Comparing these variances relates the conditional divergence to the prior divergence by a factor proportional to the posterior variance. Averaging this posterior variance over visible contexts yields the residual mode MMSE.
\end{formalbox}

\begin{proof}
Recall $\delta_I=\logit(b)-\logit(a)$ with $I=[a,b]$ from the statement of \Cref{lem:sensitivity}. Fix a mask $K\subseteq[N]$ and weights $w,\lambda\in I$. Set
\[
\delta:=\logit(\lambda)-\logit(w),
\qquad
|\delta|\le\delta_I.
\]
For $\beta\in(0,1)$ and $t\in\R$, define the logit path
\[
\beta_t:=\logit^{-1}\!\big(\logit(\beta)+t\big),
\]
so that $\beta_0=\beta$ and, by \Cref{lem:logit-shift}(ii), the reweighted posterior on every interior fiber is the endpoint of this path: if $\beta=\beta_{K^\complement}(x_{K^\complement})\in(0,1)$, then $\beta_{\lambda,{K^\complement}}(x_{K^\complement})=\beta_\delta$. 

For the prior weight, define
\[
w_t:=\logit^{-1}\!\big(\logit(w)+t\big),
\qquad t\in\R.
\]
Then $w_0=w$ and $w_\delta=\lambda$ by the definition of $\delta$.

\emph{Step 1: integral representation of the binary divergence.}
Fix $\beta\in(0,1)$ and let $f(t):=\kl(\beta\|\beta_t)$. Since $t\mapsto\beta_t$ is smooth with values in $(0,1)$, $f$ is smooth. Differentiating and using the logistic identity $\frac{\dd}{\dd t}\beta_t=\beta_t(1-\beta_t)$,
\[
f'(t)
=
\beta_t(1-\beta_t)
\left(
\frac{1-\beta}{1-\beta_t}
-
\frac{\beta}{\beta_t}
\right)
=
\beta_t-\beta,
\qquad
f''(t)
=
\beta_t(1-\beta_t).
\]
Hence $f(0)=0$ and $f'(0)=0$, and Taylor's theorem with integral remainder gives, for either sign of $\delta$,
\[
\kl(\beta\|\beta_\delta)
=
\int_0^\delta(\delta-t)\,\beta_t(1-\beta_t)\,\dd t
=
\delta^2\int_0^1(1-s)\,\beta_{\delta s}(1-\beta_{\delta s})\,\dd s,
\]
where the second form follows from the substitution $t=\delta s$ and is manifestly nonnegative.

\emph{Step 2: variance comparison along the path.}
For every $\beta\in(0,1)$ and $t\in\R$, writing $z:=e^{\logit(\beta)}$,
\[
\frac{\beta_t(1-\beta_t)}{\beta(1-\beta)}
=
e^{t}\,\frac{(1+z)^2}{(1+ze^{t})^2}.
\]
Since $\min\{1,e^t\}(1+z)\le 1+ze^t\le\max\{1,e^t\}(1+z)$, both sign cases give
\begin{equation}
e^{-|t|}\,\beta(1-\beta)
\;\le\;
\beta_t(1-\beta_t)
\;\le\;
e^{|t|}\,\beta(1-\beta).
\label{eq:path-variance-comparison}
\end{equation}
Along the integral of Step~1 we have $|t|\le|\delta|\le\delta_I$. Substituting \Cref{eq:path-variance-comparison} and $\int_0^1(1-s)\dd s=\tfrac12$ yields the two-sided fiber bound
\begin{equation}
\frac{e^{-\delta_I}}{2}\,\beta(1-\beta)\,\delta^2
\;\le\;
\kl(\beta\|\beta_\delta)
\;\le\;
\frac{e^{\delta_I}}{2}\,\beta(1-\beta)\,\delta^2.
\label{eq:fiber-sandwich}
\end{equation}

\emph{Step 3: comparison with the prior divergence.}
Recall $m_I:=\min_{u\in I}u(1-u)>0$. For every $s\in[0,1]$,
\[
\logit(w_{\delta s})
=
(1-s)\logit(w)+s\logit(\lambda).
\]
Since $\logit^{-1}$ is increasing, $w_{\delta s}$ lies between $w$ and $\lambda$ and therefore belongs to $I$. Hence
\[
m_I
\le
w_{\delta s}(1-w_{\delta s})
\le
\tfrac14.
\]
Applying the integral representation from Step~1 gives
\[
\kl(w\|\lambda)
=
\delta^2\int_0^1
(1-s)\,w_{\delta s}(1-w_{\delta s})\,\dd s.
\]
Using $\int_0^1(1-s)\,\dd s=\tfrac12$, we obtain
\begin{equation}
\frac{m_I}{2}\,\delta^2
\le
\kl(w\|\lambda)
\le
\frac18\,\delta^2,
\qquad\text{equivalently}\quad
8\,\kl(w\|\lambda)
\le
\delta^2
\le
\frac{2}{m_I}\,\kl(w\|\lambda).
\label{eq:prior-comparison}
\end{equation}

\emph{Step 4: two-sided fiberwise bound.}
Substituting \Cref{eq:prior-comparison} into \Cref{eq:fiber-sandwich} eliminates $\delta^2$: for every $\beta\in(0,1)$,
\begin{equation}
4e^{-\delta_I}\,\beta(1-\beta)\,\kl(w\|\lambda)
\;\le\;
\kl(\beta\|\beta_\delta)
\;\le\;
\frac{e^{\delta_I}}{m_I}\,\beta(1-\beta)\,\kl(w\|\lambda).
\label{eq:fiber-final}
\end{equation}
For $\beta\in\{0,1\}$, extend the path by setting $\beta_t:=\beta$ for every $t\in\R$. Both $\kl(\beta\|\beta_\delta)$ and $\beta(1-\beta)$ then vanish, so \Cref{eq:fiber-final} holds for every $\beta\in[0,1]$.

By \Cref{lem:logit-shift}, for every $x_{K^\complement}\in\supp(p_{K^\complement})$, setting $\beta=\beta_{K^\complement}(x_{K^\complement})$ gives $\beta_\delta=\beta_{\lambda,K^\complement}(x_{K^\complement})$, including when the posterior is zero or one. Thus \Cref{eq:fiber-final} applies to every supported visible context.

\emph{Step 5: average over contexts.}
Taking the expectation of \Cref{eq:fiber-final} over $X_{K^\complement}\sim p_{K^\complement}$ and using \Cref{eq:decomp-binary} of \Cref{prop:info-decomp} for the middle term and \Cref{eq:mode-mmse} for the outer terms,
\[
4e^{-\delta_I}\,
\mmse_p(Z\mid X_{K^\complement})\,
\kl(w\|\lambda)
\;\le\;
\mathfrak D_K(p\| q_\lambda)
\;\le\;
\frac{e^{\delta_I}}{m_I}\,
\mmse_p(Z\mid X_{K^\complement})\,
\kl(w\|\lambda).
\]
This proves \Cref{eq:sensitivity} with
\[
c_I:=4e^{-\delta_I},
\qquad
C_I:=\frac{e^{\delta_I}}{m_I},
\]
both depending only on $I$.
\end{proof}

\smallskip
\begin{remark}[Exact binary reduction]
\label{rem:exactly-solvable}
The mode-reweighting family makes the analysis tractable. Because $\{q_\lambda\}$ preserves the two within-mode distributions and varies only their mixture weight, the $N$-dimensional comparison reduces on every visible fiber to a one-parameter binary problem (\Cref{lem:logit-shift}). This yields both the exact information decomposition of \Cref{prop:info-decomp} and the Fisher limit in \Cref{eq:exact-sensitivity}. The local limit in \Cref{eq:exact-sensitivity} has coefficient $1/(w(1-w))$. Consequently, any upper constant valid uniformly over $\lambda,w\in I$ for a nondegenerate interval $I$ must be at least $\sup_{u\in I}1/(u(1-u))$. Thus no finite upper constant works uniformly over intervals approaching the boundary of $(0,1)$. This reduction also enables the enumerations in \Cref{sec:experiments,app:experiments} to compare the theoretical bounds with directly computed masked discrepancies.
\end{remark}

\medskip
\subsection{\texorpdfstring{Proof of \Cref{cor:mmse-upper}}{Proof of Cor. 5.2}}
\label{app:mmse-upper}

\begin{proof}
\emph{Step 1: elementary bounds.}
By \Cref{eq:mode-mmse}, the residual mode uncertainty is given by
\[
\mmse_p(Z\mid X_V)
=
\E_{X_V\sim p_V}\big[\beta_V(X_V)(1-\beta_V(X_V))\big].
\]
We abbreviate $\beta_V(X_V)$ as $\beta_V$ when the argument is clear. For every $\beta_V\in[0,1]$, the following two bounds hold. First, because $\max\{\beta_V,1-\beta_V\}\le 1$, the product satisfies
\begin{equation}
\beta_V(1-\beta_V)
=
\min\{\beta_V,1-\beta_V\}\cdot\max\{\beta_V,1-\beta_V\}
\le
\min\{\beta_V,1-\beta_V\}.
\label{eq:beta-min-bound}
\end{equation}
Second,
\begin{equation}
\beta_V(1-\beta_V)
=
\tfrac14-(\beta_V-\tfrac12)^2
\le\tfrac14.
\label{eq:beta-max-bound}
\end{equation}

\emph{Step 2: the pooled confidence set.}
Recall from \Cref{lem:pooled} the pooled confidence set:
\[
G_V=\big\{x_V\in\supp(p_V):\min\{\beta_V(x_V),1-\beta_V(x_V)\}\le e^{-\kappa|V|}\big\}.
\]
Since $|V|\ge s_{\mathrm{pin}}$ by hypothesis, \Cref{lem:pooled} ensures that the probability of encountering an atypical context is exponentially small: 
\[
p_V(G_V^\complement)\le e^{-c_1|V|}, \qquad\text{where}\quad G_V^\complement:=\supp(p_V)\setminus G_V. 
\]

\emph{Step 3: splitting the expectation.}
We bound the total expectation by splitting it over these two sets, 
\[
\mmse_p(Z\mid X_V)
=
\E_{X_V\sim p_V}\big[\beta_V(1-\beta_V)\indic\{X_V\in G_V\}\big]
+
\E_{X_V\sim p_V}\big[\beta_V(1-\beta_V)\indic\{X_V\in G_V^\complement\}\big].
\]
On the typical set $G_V$, \Cref{eq:beta-min-bound} and the definition of $G_V$ give $\beta_V(1-\beta_V)\le e^{-\kappa|V|}$, so the first term is at most $e^{-\kappa|V|}\,p_V(G_V)\le e^{-\kappa|V|}$. On the atypical set $G_V^\complement$, \Cref{eq:beta-max-bound} bounds the integrand by $\frac14$, so the second term is at most $\frac14\,p_V(G_V^\complement)\le\tfrac14 e^{-c_1|V|}$. Summing these two contributions yields:
\[
\mmse_p(Z\mid X_V)
\le
e^{-\kappa|V|}+\frac14 e^{-c_1|V|}
\le
\frac54 e^{-c|V|},
\]
where the final inequality uses $c=\min\{c_1,\kappa\}$, so that $e^{-\kappa|V|}\le e^{-c|V|}$ and $e^{-c_1|V|}\le e^{-c|V|}$.
\end{proof}

\medskip
\subsection{\texorpdfstring{Proof of \Cref{thm:mode-sensitivity}}{Proof of Thm. 5.2}}
\label{app:proof-mode-sensitivity}

\begin{formalbox}
\textbf{Proof intuition.}
The fixed-mask bound reduces schedule dependence to the average residual mode uncertainty. Under the low-visibility assumption, each mask below the visibility threshold retains a guaranteed minimum level of uncertainty, yielding a lower bound proportional to $\pi_s(\mu)$. For the upper bound, the Bernoulli variance is at most $1/4$, so low-visibility masks contribute at most $\pi_s(\mu)/4$, while mode pinning exponentially suppresses the contribution from masks above the threshold.
\end{formalbox}

\begin{proof}
\emph{Step 1: averaging the fixed-mask sensitivity.}
By \Cref{lem:sensitivity}, for every mask $K\subseteq[N]$ and every $\lambda\in I$,
\begin{equation*}
c_I\,\mmse_p(Z\mid X_{K^{\complement}})\,\kl(w\|\lambda)
\;\le\;
\mathfrak D_K(p\| q_\lambda)
\;\le\;
C_I\,\mmse_p(Z\mid X_{K^{\complement}})\,\kl(w\|\lambda).
\end{equation*}
Since $\kl(w\|\lambda)$ is independent of $K$, averaging over $K\sim\mu$ and using \Cref{def:masked-log-discrepancy} gives \Cref{eq:schedule-mmse-equivalence}.

For any visibility threshold $s$, split the schedule average of the residual MMSE as
\begin{align}
\E_{K\sim\mu}\big[\mmse_p(Z\mid X_{K^{\complement}})\big]
=&
\underbrace{\E_{K\sim\mu}\big[\mmse_p(Z\mid X_{K^{\complement}})\,\indic\{|K^{\complement}|\le s\}\big]}_{=:\mathsf{T}_{\mathrm{low}}} \nonumber\\
&+
\underbrace{\E_{K\sim\mu}\big[\mmse_p(Z\mid X_{K^{\complement}})\,\indic\{|K^{\complement}|> s\}\big]}_{=:\mathsf{T}_{\mathrm{high}}}.
\label{eq:mmse-schedule-split}
\end{align}

\emph{Step 2: the low-visibility lower bound.}
Suppose \Cref{ass:low-vis-uncertainty} holds and fix $0\le s\le s_{\mathrm{ov}}$. On the event $\{|K^{\complement}|\le s\}$, we have $|K^{\complement}|\le s\le s_{\mathrm{ov}}$, so \Cref{ass:low-vis-uncertainty} applies to the visible set $K^{\complement}$ and gives
\[
\mmse_p(Z\mid X_{K^{\complement}})
\;\ge\;
u_0\,e^{-L|K^{\complement}|}
\;\ge\;
u_0\,e^{-Ls},
\]
where the second inequality uses $|K^{\complement}|\le s$ and $L>0$. Therefore
\begin{equation}
\mathsf{T}_{\mathrm{low}}
\;\ge\;
u_0\,e^{-Ls}\,
\Pp\nolimits_{K\sim\mu}\!\big(|K^{\complement}|\le s\big)
=
u_0\,e^{-Ls}\,\pi_s(\mu),
\label{eq:Tlow-bound}
\end{equation}
by the definition of the low-visibility mass. Since $\mathsf{T}_{\mathrm{high}}\ge0$, combining \Cref{eq:mmse-schedule-split,eq:Tlow-bound} with the lower half of \Cref{eq:schedule-mmse-equivalence} gives
\[
\mathfrak D_\mu(p\| q_\lambda)
\;\ge\;
c_I\,\kl(w\|\lambda)\,\mathsf{T}_{\mathrm{low}}
\;\ge\;
c_I\,u_0\,e^{-Ls}\,\pi_s(\mu)\,\kl(w\|\lambda),
\]
which proves \Cref{eq:schedule-lower}.

\emph{Step 3: the high-visibility upper bound.}
Suppose \Cref{ass:mode-pinning} holds and fix $s\ge s_{\mathrm{pin}}$. On the event $\{|K^{\complement}|\le s\}$, the elementary bound $\beta(1-\beta)\le\tfrac14$ for $\beta\in[0,1]$ gives $\mmse_p(Z\mid X_{K^{\complement}})\le\tfrac14$, hence
\begin{equation}
\mathsf{T}_{\mathrm{low}}
\;\le\;
\tfrac14\,\pi_s(\mu).
\label{eq:Tlow-upper}
\end{equation}
On the event $\{|K^{\complement}|>s\}$, we have $|K^{\complement}|>s\ge s_{\mathrm{pin}}$, so \Cref{cor:mmse-upper} applies and gives $\mmse_p(Z\mid X_{K^{\complement}})\le \tfrac54 e^{-c|K^{\complement}|}$, hence
\begin{equation}
\mathsf{T}_{\mathrm{high}}
\;\le\;
\tfrac54\,
\E_{K\sim\mu}\big[e^{-c|K^{\complement}|}\,\indic\{|K^{\complement}|>s\}\big].
\label{eq:Thigh-upper}
\end{equation}
Combining \Cref{eq:mmse-schedule-split,eq:Tlow-upper,eq:Thigh-upper} with the upper half of \Cref{eq:schedule-mmse-equivalence},
\begin{align*}
\mathfrak D_\mu(p\| q_\lambda)
&\;\le\;
C_I\,
\Big(
\tfrac14\,\pi_s(\mu)
+
\tfrac54\,\E_{K\sim\mu}\big[e^{-c|K^{\complement}|}\,\indic\{|K^{\complement}|>s\}\big]
\Big)
\,\kl(w\|\lambda)\\
&\;\le\;
\tfrac54 C_I 
\Big(
\pi_s(\mu)
+
\E_{K\sim\mu}\big[e^{-c|K^{\complement}|}\,\indic\{|K^{\complement}|>s\}\big]
\Big) \kl(w\|\lambda).
\end{align*}
This proves \Cref{eq:schedule-upper} and completes the proof.
\end{proof}

\medskip
\subsection{\texorpdfstring{Proof of \Cref{cor:recovery}}{Proof of Cor. 5.3}}
\label{app:proof-recovery}

\begin{proof}
Throughout, fix $\lambda\in I$ and recall $m_I=\min_{u\in I}u(1-u)$ from \Cref{lem:sensitivity}.

\emph{Step 1: recovery via the low-visibility lower bound.}
Suppose \Cref{ass:low-vis-uncertainty} holds. Fix $0\le s\le s_{\mathrm{ov}}$ with $\pi_s(\mu)>0$, and assume $\mathfrak D_\mu(p\|q_\lambda)\le\varepsilon$. Since $\pi_s(\mu)>0$ and $c_I,u_0>0$, the lower bound \Cref{eq:schedule-lower} of \Cref{thm:mode-sensitivity} gives
\begin{equation}
\kl(w\|\lambda)
\;\le\;
\frac{\mathfrak D_\mu(p\| q_\lambda)}{c_I\,u_0\,e^{-Ls}\,\pi_s(\mu)}
\;\le\;
\frac{e^{Ls}\,\varepsilon}{c_I\,u_0\,\pi_s(\mu)}.
\label{eq:kl-from-budget}
\end{equation}
By Pinsker's inequality applied to the pair of Bernoulli distributions with parameters $w$ and $\lambda$, whose total-variation distance is $|w-\lambda|$,
\begin{equation}
\kl(w\|\lambda)
\;\ge\;
2\,|w-\lambda|^2.
\label{eq:binary-pinsker}
\end{equation}
Combining \Cref{eq:kl-from-budget,eq:binary-pinsker} and taking square roots,
\[
|w-\lambda|
\;\le\;
\sqrt{\frac{e^{Ls}\,\varepsilon}{2\,c_I\,u_0\,\pi_s(\mu)}}
\;=\;
\frac{e^{Ls/2}}{\sqrt{2c_I\,u_0}}
\sqrt{\frac{\varepsilon}{\pi_s(\mu)}},
\]
which proves \Cref{eq:recovery}.

\emph{Step 2: a quadratic upper bound on the binary divergence.}
For the converse, we first bound $\kl(w\|\lambda)$ by a quadratic. Using $\log x\le x-1$ for $x>0$ on both terms,
\[
\kl(w\|\lambda)
=
w\log\frac{w}{\lambda}
+
(1-w)\log\frac{1-w}{1-\lambda}
\;\le\;
w\,\frac{w-\lambda}{\lambda}
+
(1-w)\,\frac{\lambda-w}{1-\lambda}
=
\frac{(w-\lambda)^2}{\lambda(1-\lambda)},
\]
where the final equality follows from
\[
\frac{w}{\lambda}-\frac{1-w}{1-\lambda}
=\frac{w(1-\lambda)-\lambda(1-w)}{\lambda(1-\lambda)}
=\frac{w-\lambda}{\lambda(1-\lambda)}.
\]
Since $\lambda\in I$, we conclude
\begin{equation}
\kl(w\|\lambda)
\;\le\;
\frac{|w-\lambda|^2}{m_I}.
\label{eq:kl-quadratic-upper}
\end{equation}

\emph{Step 3: the converse via the high-visibility upper bound.} Suppose \Cref{ass:mode-pinning} holds. Fix $s\ge s_{\mathrm{pin}}$, and assume the hypothesis of \Cref{eq:recovery-converse}:
\[
\frac{5 C_I}{4 m_I}\Big(
\pi_s(\mu)
+
\E_{K\sim\mu}\big[e^{-c|K^{\complement}|}\,\indic\{|K^{\complement}|>s\}\big]
\Big)\,
|w-\lambda|^2
\;\le\;
\varepsilon.
\]
Then, combining the upper bound \Cref{eq:schedule-upper} of \Cref{thm:mode-sensitivity} with \Cref{eq:kl-quadratic-upper},
\begin{align*}
\mathfrak D_\mu(p\| q_\lambda)
&\;\le\;
\tfrac54 C_I\Big(
\pi_s(\mu)
+
\E_{K\sim\mu}\big[e^{-c|K^{\complement}|}\,\indic\{|K^{\complement}|>s\}\big]
\Big)\,
\kl(w\|\lambda)\\
&\;\le\;
\frac{5 C_I}{4 m_I}\Big(
\pi_s(\mu)
+
\E_{K\sim\mu}\big[e^{-c|K^{\complement}|}\,\indic\{|K^{\complement}|>s\}\big]
\Big)\,
|w-\lambda|^2\\
&\;\le\;
\varepsilon,
\end{align*}
which proves \Cref{eq:recovery-converse} and completes the proof.
\end{proof}

\medskip
\subsection{\texorpdfstring{Proof of \Cref{prop:full-mask}}{Proof of Prop. 5.1}}
\label{app:proof-full-mask}

\begin{formalbox}
\textbf{Proof intuition.}
The full mask directly measures the joint KL divergence. For necessity, consider two configurations that differ in every coordinate. Any nonempty visible context uniquely identifies which configuration occurred, so perturbing their mixture weights leaves all conditionals under non-full masks unchanged. Only the full mask detects the weight mismatch, proving both necessity and sharpness of the control constant.
\end{formalbox}

\begin{proof}
\emph{Step 1: sufficiency and the modulus bound.}
Suppose $\pi_0(\mu)>0$, and fix $p\in\Delta(\cX)$ and $q\in\cQ(p)$. At the full mask, the visible context is empty, so
\begin{equation}
\mathfrak D_{[N]}(p\|q)=\KL(p\|q).
\label{eq:full-mask-identity}
\end{equation}
Since every fixed-mask discrepancy is nonnegative,
\begin{align}
\mathfrak D_\mu(p\|q)
&=
\pi_0(\mu)\,\KL(p\|q)
+
\sum_{K\subsetneq[N]}
\mu(K)\,\mathfrak D_K(p\|q)
\nonumber\\
&\ge
\pi_0(\mu)\,\KL(p\|q).
\label{eq:full-mask-lower-bound}
\end{align}
Thus $C_\mu=1/\pi_0(\mu)$ satisfies \Cref{eq:full-mask-coercivity} uniformly over $p$ and $q$.

If $\mathfrak D_\mu(p\|q)\le\varepsilon$, then \Cref{eq:full-mask-lower-bound} and Pinsker's inequality give
\[
\DTV(p,q)
\le
\sqrt{\frac12\,\KL(p\|q)}
\le
\sqrt{\frac{\varepsilon}{2\,\pi_0(\mu)}}.
\]
Taking the supremum over all such $q\in\cQ(p)$ proves \Cref{eq:full-mask-modulus}.

\emph{Step 2: necessity and optimality.}
Choose configurations $x^-,x^+\in\cX$ that differ in every coordinate, and define
\[
p=\tfrac12\delta_{x^-}+\tfrac12\delta_{x^+},
\qquad
q_\lambda
=
\lambda\delta_{x^+}+(1-\lambda)\delta_{x^-},
\]
where $\lambda\in(0,1)\setminus\{\tfrac12\}$. These distributions have the same support, so $p\ll q_\lambda$ and $q_\lambda\in\cQ(p)$.

Fix $K\subsetneq[N]$ and write $V=K^\complement$. Because $V\ne\varnothing$, the two visible configurations $x_V^-$ and $x_V^+$ are distinct. They constitute the support of $p_V$ and both have positive probability under $(q_\lambda)_V$. Conditioning on either context therefore identifies the configuration under both distributions:
\[
p_K(\cdot\mid x_V^\sigma)
=
(q_\lambda)_K(\cdot\mid x_V^\sigma)
=
\delta_{x_K^\sigma},
\qquad \sigma\in\{-,+\}.
\]
Hence the conditional distributions coincide $p_V$-almost surely, and $\mathfrak D_K(p\|q_\lambda)=0$.

Since $p$ and $q_\lambda$ are distinct and have the same finite support, the divergence $\kappa_\lambda:=\KL(p\|q_\lambda)$ satisfies $0<\kappa_\lambda<\infty$. Combining the vanishing non-full-mask discrepancies with \Cref{eq:full-mask-identity} yields
\begin{equation}
\mathfrak D_\mu(p\|q_\lambda)
=
\pi_0(\mu)\,\kappa_\lambda
=
\pi_0(\mu)\,\KL(p\|q_\lambda).
\label{eq:full-mask-boundary-witness}
\end{equation}
Any finite constant satisfying \Cref{eq:full-mask-coercivity} uniformly over $p,q$ must therefore satisfy
\[
\kappa_\lambda
\le
C_\mu\,\pi_0(\mu)\,\kappa_\lambda.
\]
Dividing by $\kappa_\lambda>0$ gives $1\le C_\mu\pi_0(\mu)$. This inequality is impossible when $\pi_0(\mu)=0$ and forces $C_\mu\ge1/\pi_0(\mu)$ otherwise. Together with Step~1, this proves necessity and optimality.
\end{proof}

\smallskip
\begin{remark}[Sharpness of the full-mask TV scaling under pinning]
\label{rem:full-mask-sharp}
The constant $1/\pi_0(\mu)$ in \Cref{eq:full-mask-coercivity} is optimal uniformly over data distributions, as witnessed by the two-point construction in the proof. For a fixed distribution under \Cref{ass:mode-pinning}, the TV bound \Cref{eq:full-mask-modulus} is also sharp in its \(\sqrt{\varepsilon/\pi_0}\) scaling. Write $\pi_0:=\pi_0(\mu)>0$, and suppose all non-full masks in the support of $\mu$ satisfy $|K^\complement|\ge s_{\mathrm{pin}}$. For $\lambda\in I_{\mathrm{bl}}$, the full-mask case of \Cref{eq:decomp-chain} and \Cref{eq:per-mask-mode-blind-bound} give
\[
\mathfrak D_\mu(p\|q_\lambda)
\;\le\;
\pi_0\,\kl(w\|\lambda)
+
\underbrace{
C\sum\nolimits_{K\subsetneq[N]}
\mu(K)e^{-c|K^\complement|}
}_{=:r_\mu}.
\]
Define $m:=\min_{u\in I_{\mathrm{bl}}}u(1-u)>0$ and $\eta_I:=\max_{\lambda\in I_{\mathrm{bl}}}\kl(w\|\lambda)>0$. Suppose $r_\mu<\pi_0\eta_I$. For every tolerance satisfying
\[
2r_\mu\le\varepsilon\le2\pi_0\eta_I,
\]
continuity of $\lambda\mapsto\kl(w\|\lambda)$ on $I_{\mathrm{bl}}$, together with $\kl(w\|w)=0$, ensures that there exists $\lambda\in I_{\mathrm{bl}}$ such that
\[
\kl(w\|\lambda)=\frac{\varepsilon}{2\pi_0}.
\] 
For this choice, $\mathfrak D_\mu(p\|q_\lambda)\le\varepsilon$. Moreover, $\kl(w\|\lambda)\le |w-\lambda|^2/m$ and $\DTV(p,q_\lambda)=|w-\lambda|$. Combining this with \Cref{eq:full-mask-modulus} yields
\[
\sqrt{\frac{m}{2}}\,
\sqrt{\frac{\varepsilon}{\pi_0}}
\;\le\;
\modl_{\TV}^{\mu}(p,\varepsilon)
\;\le\;
\frac{1}{\sqrt{2}}\,
\sqrt{\frac{\varepsilon}{\pi_0}}.
\]
Hence the lower and upper bounds match in their \(\sqrt{\varepsilon/\pi_0}\) dependence throughout this range, with constants controlled by \(I_{\mathrm{bl}}\). The lower tolerance threshold accounts for residual sensitivity from non-full masks, while the upper threshold keeps the witness within the admissible reweighting interval.
\end{remark}

\bigskip
\section{Mode Blindness with Multiple Modes (\texorpdfstring{$M\ge 2$}{M})}
\label{app:M-modes}

This section extends \Cref{thm:mode-blind} to a fixed number \(M\ge2\) of modes. Under \(M\)-mode pinning, large-context masked prediction can be exponentially insensitive to global mixture weights even when the within-mode distributions are fixed.
 
\subsection{Mode Structure and Pinning}
\label{app:M-mode-setup}

Let $\cX = \bigsqcup_{j \in [M]} \cX_j$ be a partition where $w_j := p(\cX_j) > 0$ for every $j$. We write $\boldsymbol w := (w_1,\dots,w_M)$ for the true mode-weight vector and $p^j := p(\cdot \mid \cX_j)$ for the within-mode distributions. Let
\[
\Delta_{M-1} := \Big\{ \boldsymbol{\lambda} \in \R^M : \lambda_j \ge 0 \;\;\forall\ j \in [M], \ \sum\nolimits_{j\in[M]} \lambda_j = 1 \Big\}.
\]
The relative interior of this simplex is its interior within the affine hyperplane
$\{\boldsymbol\lambda\in\R^M:\sum_{j=1}^M\lambda_j=1\}$:
\[
\operatorname{relint}(\Delta_{M-1})
:=
\left\{
\boldsymbol\lambda\in\Delta_{M-1}:
\lambda_j>0 \text{ for every }j\in[M]
\right\}.
\]
For any $\boldsymbol\lambda\in \operatorname{relint}(\Delta_{M-1})$, we define the reweighted distribution
\begin{equation}
  q_{\boldsymbol\lambda} := \sum\nolimits_{j \in [M]} \lambda_j\, p^j.
  \label{eq:M-reweighted-law}
\end{equation}
This construction preserves every within-mode conditional distribution of $p$ while setting $q_{\boldsymbol\lambda}(\cX_j) = \lambda_j$. As in the binary case, $\supp(q_{\boldsymbol\lambda}) = \supp(p)$, and thus $q_{\boldsymbol\lambda} \in \cQ(p)$. We represent the mode indicator by the one-hot vector $\mathbf Z := \bigl(\indic\{X \in \cX_j\}\bigr)_{j \in [M]}$, with posterior
\[
  \boldsymbol\beta_V(x_V)
  := \bigl( p(X \in \cX_j \mid X_V = x_V) \bigr)_{j \in [M]} \in \Delta_{M-1}.
\]
We denote the corresponding posterior under $q_{\boldsymbol\lambda}$ as $\boldsymbol\beta_{\boldsymbol\lambda,V}$. We use two scalar summaries of $\boldsymbol\beta_V$. The first is the \emph{vertex deficiency}
\begin{equation}
  \gamma_V(x_V) := 1 - \max\nolimits_{j \in [M]} \beta_{V,j}(x_V),
  \label{eq:vertex-deficiency}
\end{equation}
which vanishes when the visible context perfectly identifies the region. The second is the one-hot residual MMSE, equivalently the posterior Gini impurity averaged over visible contexts:
\begin{equation}
  \mmse_p(\mathbf Z \mid X_V)
  = \E\bigl[\|\mathbf Z - \E[\mathbf Z \mid X_V]\|^2\bigr]
  = \E_{X_V \sim p_V}\bigl[\, 1 - \|\boldsymbol\beta_V(X_V)\|_2^2 \,\bigr].
  \label{eq:M-mmse}
\end{equation}
For $M = 2$, this is twice the quantity in \Cref{eq:mode-mmse}. Pointwise, the vertex deficiency and the posterior Gini impurity satisfy
\[
\gamma_V(x_V)
\le
1-\|\boldsymbol\beta_V(x_V)\|_2^2
\le
2\gamma_V(x_V).
\]
Indeed, writing $\gamma=\gamma_V(x_V)$ and $\boldsymbol\beta=\boldsymbol\beta_V(x_V)$, we have
\(
(1-\gamma)^2
\le
\|\boldsymbol\beta\|_2^2
\le
1-\gamma,
\)
which gives $\gamma\le1-\|\boldsymbol\beta\|_2^2 \le2\gamma-\gamma^2\le2\gamma$. Averaging over visible contexts therefore yields
\[
\E_{X_V\sim p_V}\big[\gamma_V(X_V)\big]
\le
\mmse_p(\mathbf Z\mid X_V)
\le
2\,\E_{X_V\sim p_V}\big[\gamma_V(X_V)\big].
\]

\paragraph{The pinning condition.}
We assume that sufficiently large visible contexts identify the mode with high confidence.

\begin{assumption}[$M$-mode pinning]
\label{ass:M-mode-pinning}
There exist constants $c_0\in(0,1/M]$, $c_1,\kappa>0$, and a visibility threshold $s_{\mathrm{pin}}\in[N]$ such that $w_j\ge c_0$ for every $j\in[M]$. Moreover, for every visible set $V\subseteq[N]$ with $|V|\ge s_{\mathrm{pin}}$ and every mode $j\in[M]$, the typical set
\[
\cE^j_V
:=
\bigl\{
x_V\in\supp(p^j_V):
p(\cX\setminus\cX_j\mid X_V=x_V)
\le e^{-\kappa|V|}
\bigr\}
\]
satisfies
\[
p^j_V(\cE^j_V)\ge1-e^{-c_1|V|}.
\]
\end{assumption}

As before, we define the combined rate $c:= \min\{c_1,\kappa\}$. For $M = 2$, \Cref{ass:M-mode-pinning} reduces to \Cref{ass:mode-pinning}.

\medskip
\subsection{Mixture Identities and Posterior Tilting}
\label{app:M-mode-identities}

We first record how total variation and KL divergence reduce over disjoint mixture components.

\begin{lemma}[Shared-component mixtures with $M$ components]
\label{lem:shared-mix-M}
Let $\cZ$ be a finite set, and let $\nu_1,\dots,\nu_M\in\Delta(\cZ)$ have pairwise disjoint supports. For $\boldsymbol a\in\Delta_{M-1}$, define
\[
m_{\boldsymbol a}:=\sum_{j=1}^M a_j\nu_j.
\]
Then, for every
$\boldsymbol a,\boldsymbol b\in\Delta_{M-1}$,
\[
\DTV(m_{\boldsymbol a},m_{\boldsymbol b})
=
\tfrac12\|\boldsymbol a-\boldsymbol b\|_1,
\qquad
\KL(m_{\boldsymbol a}\|m_{\boldsymbol b})
=
\KL(\boldsymbol a\|\boldsymbol b),
\]
with the usual extended-value conventions for zero coordinates.
\end{lemma}

\begin{proof}
Let $S_j:=\supp(\nu_j)$ for $j\in[M]$. Since these supports are pairwise disjoint, for every $z\in S_j$,
\[
m_{\boldsymbol a}(z)=a_j\nu_j(z),
\qquad
m_{\boldsymbol b}(z)=b_j\nu_j(z).
\]
Hence
\[
\begin{aligned}
\|m_{\boldsymbol a}-m_{\boldsymbol b}\|_1
=
\sum_{j=1}^M\sum_{z\in S_j}
|a_j-b_j|\,\nu_j(z)
=
\sum_{j=1}^M |a_j-b_j|
=
\|\boldsymbol a-\boldsymbol b\|_1,
\end{aligned}
\]
because $\sum_{z\in S_j}\nu_j(z)=1$. Therefore,
\[
\DTV(m_{\boldsymbol a},m_{\boldsymbol b})
=
\frac12\|\boldsymbol a-\boldsymbol b\|_1.
\]
For the KL divergence, splitting over the disjoint supports gives
\[
\KL(m_{\boldsymbol a}\|m_{\boldsymbol b})
=
\sum_{j=1}^M
\sum_{z\in S_j}
a_j\nu_j(z)
\log\frac{a_j\nu_j(z)}{b_j\nu_j(z)}
=
\sum_{j=1}^M a_j\log\frac{a_j}{b_j}
=
\KL(\boldsymbol a\|\boldsymbol b),
\]
with the equality understood in the extended sense: if $a_j>0=b_j$ for some $j$, both sides are $+\infty$, while terms with $a_j=0$ vanish.
\end{proof}

\paragraph{Posterior tilting and information decomposition.}
Reweighting the global modes induces an exponential tilt of the visible posterior. This tilt determines the fiberwise divergence and hence the masked discrepancy.

\begin{lemma}[$M$-mode posterior tilt and information decomposition]
\label{lem:M-tilt}
Fix $V\subseteq[N]$, let $K=[N]\setminus V$, and take $x_V\in\supp(p_V)$ and $\boldsymbol\lambda\in\operatorname{relint}(\Delta_{M-1})$. Then
\begin{equation}
  \beta_{\boldsymbol\lambda,V,j}(x_V)
  = \frac{(\lambda_j / w_j)\, \beta_{V,j}(x_V)}
         {\sum_{l\in[M]} (\lambda_l / w_l)\, \beta_{V,l}(x_V)},
  \label{eq:M-posterior-tilt}
\end{equation}
so the posterior under $q_{\boldsymbol\lambda}$ is an exponential tilt of the posterior under $p$ by the log-ratio vector $\boldsymbol t := (\log(\lambda_j/w_j))_{j\in[M]}$. Moreover,
\begin{equation}
  \KL\bigl( p_{V^\complement}(\cdot\mid x_V) \,\|\, q_{\boldsymbol\lambda, V^\complement}(\cdot\mid x_V) \bigr)
  = \KL\bigl( \boldsymbol\beta_V(x_V) \,\|\, \boldsymbol\beta_{\boldsymbol\lambda,V}(x_V) \bigr).
  \label{eq:M-conditional-kl}
\end{equation}
Consequently, for every mask $K$ with visible set $K^{\complement}=[N]\setminus K$,
\begin{equation}
  \mathfrak D_K(p \| q_{\boldsymbol\lambda})
  = \E_{X_{K^{\complement}} \sim p_{K^{\complement}}}\bigl[ \KL(\boldsymbol\beta_{K^{\complement}} \| \boldsymbol\beta_{\boldsymbol\lambda,K^{\complement}}) \bigr]
  = \KL(\boldsymbol w \| \boldsymbol\lambda) - \KL(p_{K^{\complement}} \| q_{\boldsymbol\lambda,K^{\complement}}).
  \label{eq:M-decomp}
\end{equation}
\end{lemma}
 
\begin{proof}
As in \Cref{lem:logit-shift}, $q_{\boldsymbol\lambda}$ has density $\lambda_j/w_j$ relative to $p$ on $\cX_j$. Summing over completions on each region and normalizing yields \Cref{eq:M-posterior-tilt}. 

For \Cref{eq:M-conditional-kl}, fix $x_V\in\supp(p_V)$ and abbreviate $\beta_{V,j}=\beta_{V,j}(x_V)$ and $\beta_{\boldsymbol\lambda,V,j}=\beta_{\boldsymbol\lambda,V,j}(x_V)$. For each $j$ with $\beta_{V,j}>0$, define
\[
\nu_j
:=
p\!\left(
X_K\in\cdot
\,\middle|\,
X_V=x_V,\ X\in\cX_j
\right).
\]
By \Cref{eq:M-posterior-tilt}, $\beta_{\boldsymbol\lambda,V,j}>0$ if and only if $\beta_{V,j}>0$. Reweighting preserves the conditional distribution within each mode, so
\[
p_K(\cdot\mid x_V)
=
\sum_{j:\beta_{V,j}>0}\beta_{V,j}\nu_j,
\qquad
q_{\boldsymbol\lambda,K}(\cdot\mid x_V)
=
\sum_{j:\beta_{V,j}>0}
\beta_{\boldsymbol\lambda,V,j}\nu_j.
\]
For fixed $x_V$, each completion $x_K$ determines a full configuration belonging to exactly one mode region. The components in these sums therefore have pairwise disjoint supports. The shared-component identity of \Cref{lem:shared-mix-M} gives
\[
\KL\bigl(
p_K(\cdot\mid x_V)
\,\|\,
q_{\boldsymbol\lambda,K}(\cdot\mid x_V)
\bigr)
=
\sum_{j:\beta_{V,j}>0}
\beta_{V,j}\log
\frac{\beta_{V,j}}{\beta_{\boldsymbol\lambda,V,j}}
=
\KL\bigl(
\boldsymbol\beta_V(x_V)
\,\|\,
\boldsymbol\beta_{\boldsymbol\lambda,V}(x_V)
\bigr),
\]
where omitted indices have zero probability under both posteriors. This proves \Cref{eq:M-conditional-kl}.

Finally, for every mask $K$, set $V=K^\complement$. The first equality in \Cref{eq:M-decomp} is exactly \Cref{def:masked-log-discrepancy}. The second follows from the KL chain rule together with $\KL(p \| q_{\boldsymbol\lambda}) = \KL(\boldsymbol w \| \boldsymbol\lambda)$, which is \Cref{lem:shared-mix-M} applied to $p$ and $q_{\boldsymbol\lambda}$.
\end{proof}

\smallskip
\subsection{Mode Blindness and the Modulus Lower Bound}
\label{app:M-mode-bound}

By \Cref{lem:M-tilt}, the masked discrepancy reduces to KL divergence between the original and tilted mode posteriors. The vertex deficiency controls this divergence when the visible context concentrates near a single mode.

\begin{lemma}[KL bound near a simplex vertex]
\label{lem:M-vertex}
Let $B > 0$, and consider a vector $\boldsymbol t \in \R^{M}$ satisfying $\max_{j,l \in [M]} |t_j - t_l| \le 2B$. For any $\boldsymbol\beta \in \Delta_{M-1}$, let $\boldsymbol\beta_{\boldsymbol t}$ denote its tilt defined by $\beta_{\boldsymbol t,j} \propto \beta_j e^{t_j}$. Setting $\gamma := 1 - \max_{j} \beta_j$, we have
\begin{equation}
  \KL(\boldsymbol\beta \| \boldsymbol\beta_{\boldsymbol t})
  \;\le\; \min\bigl\{\, B_0,\; A_0\, \gamma \,\bigr\},
  \qquad B_0 := 2B, \quad A_0 := e^{2B} - 1 + 2B .
  \label{eq:M-vertex-bound}
\end{equation}
\end{lemma}
 
\begin{proof}
Let $\psi(\boldsymbol\beta) := \log\bigl(\sum_l \beta_l e^{t_l}\bigr) - \sum_j \beta_j t_j$. A direct computation from the definition of the tilt yields $\KL(\boldsymbol\beta \| \boldsymbol\beta_{\boldsymbol t}) = \psi(\boldsymbol\beta)$. Because replacing $\boldsymbol t$ with $\boldsymbol t + s\boldsymbol 1$ adds $s$ to both terms, $\psi$ is invariant under scalar shifts of $\boldsymbol t$. The bounding hypothesis on $\boldsymbol t$ is likewise shift-invariant.

For the first bound, we have $\log\bigl(\sum_l \beta_l e^{t_l}\bigr) \le \max_l t_l$ and $-\sum_j \beta_j t_j \le -\min_l t_l$. Combining these inequalities yields $\psi \le \max_l t_l - \min_l t_l \le 2B = B_0$.

For the second bound, let $j^\ast \in \arg\max_j \beta_j$. This implies $\beta_{j^\ast} = 1-\gamma$ and $\sum_{l \ne j^\ast} \beta_l = \gamma$. By shift invariance, we may assume $t_{j^\ast} = 0$. The hypothesis then guarantees $|t_l| \le 2B$ for every $l$. Using the inequality $\log(1+u) \le u$, we obtain
\[
  \log\Bigl(\sum\nolimits_l \beta_l e^{t_l}\Bigr)
  \le \log\bigl( (1-\gamma) + \gamma e^{2B} \bigr)
  = \log\bigl( 1 + \gamma(e^{2B}-1) \bigr)
  \le \gamma\,(e^{2B}-1).
\]
Also, the second term is bounded as $-\sum_j \beta_j t_j \le \sum_{l \ne j^\ast} \beta_l |t_l| \le 2B\gamma$. Adding these two bounds together gives $\psi \le \gamma\,(e^{2B}-1+2B) = A_0 \gamma$.
\end{proof}

Combining \Cref{lem:M-tilt,lem:M-vertex} yields a uniform discrepancy bound over the \((M-1)\)-dimensional reweighting region below.
 
\begin{theorem}[Mode blindness with $M$ modes]
\label{thm:M-mode-blind}
Suppose the $M$-mode pinning assumption above holds with combined rate $c = \min\{c_1,\kappa\}$, and let $\Lambda := \{\boldsymbol\lambda \in \Delta_{M-1} : \lambda_j \ge c_0/2 \ \ \forall j\}$. There is a constant $C \le 12c_0^{-2}$, depending only on $c_0$, such that every mask schedule $\mu$ with $v_{\mathrm{min}}(\mu) \ge s_{\mathrm{pin}}$ satisfies
\begin{equation}
  \sup\nolimits_{\boldsymbol\lambda \in \Lambda} \mathfrak D_\mu(p \| q_{\boldsymbol\lambda})
  \;\le\; C\, \E_{K\sim\mu}\bigl[ e^{-c|K^{\complement}|} \bigr]
  \;\le\; C e^{-c\, v_{\mathrm{min}}(\mu)} .
  \label{eq:M-mode-blindness}
\end{equation}
Consequently, for every
$\varepsilon \ge C\,\E_{K\sim\mu}\bigl[e^{-c|K^{\complement}|}\bigr]$,
\begin{equation}
\modl_{\TV}^{\mu}(p,\varepsilon)
\ge
(M-1)\left(\frac1M-\frac{c_0}{2}\right)
\ge
\frac{M-1}{2M}
\ge\frac14.
\label{eq:M-modulus}
\end{equation}
\end{theorem}

\begin{formalbox}
\textbf{Proof intuition.}
When the visible context identifies one mode with high confidence, bounded reweighting can affect the conditional distribution only through the small posterior mass assigned to the remaining modes. Atypical contexts carry negligible probability under each mode. Averaging their probabilities under the true mode weights introduces no extra factor of $M$, establishing the uniform discrepancy bound.
\end{formalbox}
 
\begin{proof}
\emph{Step 1: uniform constants.}
Since $w_j\ge c_0$ and $\lambda_j\ge c_0/2$, the log-ratios $t_j = \log(\lambda_j/w_j)$ satisfy $\max_{j,l}|t_j - t_l| \le 2B$ with $B := \log(2/c_0)$. Writing $y := e^{2B} = 4c_0^{-2}$, the constant 
\[
C := A_0 + B_0 = y - 1 + 2\log y \le 3(y - 1) < 12c_0^{-2}.
\]
 
\emph{Step 2: fiberwise bound.}
Fix $\boldsymbol\lambda \in \Lambda$ and a mask $K \in \supp(\mu)$. Let $V := K^{\complement}$, ensuring the visible set satisfies 
\[
|V| \ge v_{\mathrm{min}}(\mu) \ge s_{\mathrm{pin}}.
\]
By \Cref{lem:M-tilt}, $\boldsymbol\beta_{\boldsymbol\lambda,V}$ is the tilt of $\boldsymbol\beta_V$ by $\boldsymbol t$. Furthermore, the conditional divergence on the fiber over $x_V$ equals $\KL(\boldsymbol\beta_V \| \boldsymbol\beta_{\boldsymbol\lambda,V})$. Applying \Cref{lem:M-vertex} therefore guarantees that for every $x_V \in \supp(p_V)$,
\begin{equation}
  \KL\bigl(\boldsymbol\beta_V(x_V) \| \boldsymbol\beta_{\boldsymbol\lambda,V}(x_V)\bigr)
  \;\le\; \min\bigl\{ B_0,\; A_0\, \gamma_V(x_V) \bigr\} .
  \label{eq:M-fiberwise}
\end{equation}
 
\emph{Step 3: pooled confidence set.}
Define $G_V := \{ x_V \in \supp(p_V) : \gamma_V(x_V) \le e^{-\kappa|V|} \}$. For $x_V \in \cE^{j}_V$, the pinning assumption yields $1 - \beta_{V,j}(x_V) \le e^{-\kappa|V|}$. This implies $\gamma_V(x_V) \le e^{-\kappa|V|}$, so $x_V \in G_V$. We therefore have $\cE^{j}_V \subseteq G_V$ for every $j$. Because $p_V = \sum_j w_j p^{j}_V$, we obtain
\[
  p_V(G_V^\complement)
  = \sum\nolimits_j w_j\, p^{j}_V(G_V^\complement)
  \le \sum\nolimits_j w_j\, p^{j}_V\bigl((\cE^{j}_V)^\complement\bigr)
  \le e^{-c_1|V|},
\]
which extends \Cref{lem:pooled} verbatim to $M$ regions.
 
\emph{Step 4: averaging.}
We split the expectation of \Cref{eq:M-fiberwise} over $G_V$ and its complement. Applying the bound $\gamma_V \le e^{-\kappa|V|}$ on $G_V$ alongside the uniform bound $B_0$ outside it yields
\[
  \E_{X_V\sim p_V}\bigl[\KL(\boldsymbol\beta_V \| \boldsymbol\beta_{\boldsymbol\lambda,V})\bigr]
  \;\le\; A_0 e^{-\kappa|V|} + B_0\, p_V(G_V^\complement)
  \;\le\; (A_0 + B_0)\, e^{-c|V|} .
\]
By \Cref{lem:M-tilt}, the masked discrepancy $\mathfrak D_\mu(p \| q_{\boldsymbol\lambda})$ is the average of the left-hand side over $K \sim \mu$. Setting $C = A_0 + B_0$ and taking the supremum over $\boldsymbol\lambda \in \Lambda$ establishes \Cref{eq:M-mode-blindness}. The final inequality relies on the minimum visible size constraint $|K^{\complement}| \ge v_{\mathrm{min}}(\mu)$.

\emph{Step 5: modulus lower bound.}
At the stated tolerance, every $q_{\boldsymbol\lambda}$ with $\boldsymbol\lambda\in\Lambda$ is admissible in the supremum defining $\modl_{\TV}^{\mu}(p,\varepsilon)$.

Choose $j^\ast\in\arg\min_j w_j$, so $w_{j^\ast}\le1/M$, and define
\[
\lambda_{j^\ast}
:=
1-\frac{(M-1)c_0}{2},
\qquad
\lambda_j:=\frac{c_0}{2}
\quad (j\ne j^\ast).
\]
These weights sum to one and belong to $\Lambda$. For every $j\ne j^\ast$, we have $\lambda_j=c_0/2\le w_j$. Hence \Cref{lem:shared-mix-M} gives
\begin{align*}
\DTV(p,q_{\boldsymbol\lambda})
&=
\sum_{j\ne j^\ast}\left(w_j-\frac{c_0}{2}\right)
\\
&=
1-w_{j^\ast}-\frac{(M-1)c_0}{2}
\\
&\ge
(M-1)\left(\frac1M-\frac{c_0}{2}\right)
\\
&\ge
\frac{M-1}{2M}\\
&\ge\frac14,
\end{align*}
where the last two inequalities use $c_0\le1/M$ and $M\ge2$. Since this distribution is admissible, the same lower bound holds for $\modl_{\TV}^{\mu}(p,\varepsilon)$.
\end{proof}
 

\medskip
\subsection{Directional Sensitivity Beyond Two Modes}
\label{sec:multimode-sensitivity}

For $M=2$, the lower bound in \Cref{eq:M-modulus} becomes $(1-c_0)/2$, matching \Cref{thm:mode-blind}. More generally, for every $M\ge2$, \Cref{thm:M-mode-blind} establishes mode blindness over an $(M-1)$-dimensional reweighting region containing a distribution at total-variation distance at least $1/4$ from $p$, while every distribution in that region satisfies the stated discrepancy budget.

The two-sided sensitivity characterization of \Cref{thm:mode-sensitivity} does not extend by simply replacing binary residual uncertainty with the one-hot MMSE. For \(M=2\), after quotienting out the common log-weight shift, there is only one mode-contrast direction. For \(M>2\), sensitivity depends on the direction of reweighting within the contrast subspace \(\mathbf 1^\perp\). Define the schedule-averaged posterior
covariance
\[
\Sigma_\mu
:=
\mathbb{E}_{K\sim\mu}\mathbb{E}_p
\!\left[
\operatorname{Cov}_p
\!\left(\mathbf{Z}\mid X_{K^\complement}\right)
\right].
\]
For a log-weight perturbation direction $\boldsymbol{h}\in\mathbb{R}^M$, consider the path $\boldsymbol{\lambda}(a)$ with components
\[
\lambda_j(a)
=
\frac{w_j e^{a h_j}}
{\sum_{\ell=1}^M w_\ell e^{a h_\ell}},
\qquad j\in[M].
\]
By \Cref{lem:M-tilt}, the fiberwise discrepancy is the KL divergence between a categorical posterior and its exponential tilt. A second-order expansion therefore gives
\[
\mathfrak{D}_\mu
\bigl(p\|q_{\boldsymbol{\lambda}(a)}\bigr)
=
\frac{a^2}{2}\,
\boldsymbol{h}^{\!\top}\Sigma_\mu\boldsymbol{h}
+o(a^2),
\qquad a\to0.
\]
Moreover, the schedule-averaged residual uncertainty satisfies
\[
\mathbb{E}_{K\sim\mu}
\bigl[
\operatorname{mmse}_p
(\mathbf{Z}\mid X_{K^\complement})
\bigr]
=
\operatorname{tr}(\Sigma_\mu).
\]
Thus, the scalar MMSE measures total posterior uncertainty but does not determine how that uncertainty is distributed across reweighting directions. In particular, \(\operatorname{tr}(\Sigma_\mu)>0\) does not imply positive sensitivity in every contrast direction. For example, suppose the scheduled visible contexts perfectly distinguish mode $3$ from the group $\{1,2\}$ but leave modes $1$ and $2$ unresolved. Then changing the weights from $(1/3,1/3,1/3)$ to $(1/4,1/4,1/2)$ preserves the relative weights of modes $1$ and $2$ and leaves all scheduled masked conditionals unchanged, despite positive residual mode uncertainty. The universal null direction \(\boldsymbol 1\) reflects invariance to a common log-weight shift, but additional null or weak contrast directions may remain after schedule averaging.

A multimode analog of \Cref{thm:mode-sensitivity} therefore requires directional control of \(\Sigma_\mu\) on \(\mathbf1^\perp\), rather than scalar residual uncertainty alone. This distinction does not affect the multimode blindness result or the full-mask guarantee in \Cref{prop:full-mask}.

\bigskip
\section{Low-Visibility Mass for Common Mask Schedules}
\label{app:instantiations}

We compute the low-visibility mass $\pi_s(\mu):=\mu(|K^{\complement}|\le s)$ for several common mask schedules. This quantity enters the mode-weight recovery bound in \Cref{cor:recovery}. \Cref{tab:schedule-taxonomy} summarizes the asymptotic rates for fixed $s\ge0$ as $N\to\infty$, with all schedule parameters independent of $N$. As shown in the table, common schedules can place polynomially small, exponentially small, zero, or constant-order mass on low-visibility masks.

\paragraph{Random masking probabilities.}
Absorbing-state diffusion uses independent coordinate masking conditional on a noise level \citep{austin2021structured,shi2024simplified}. Let $t\in[0,1]$ denote the masking probability, drawn from a continuous density $g$. Conditional on $t$, mask each coordinate independently with probability $t$. Then $|K^{\complement}|\mid t\sim\operatorname{Bin}(N,1-t)$, and
\begin{equation}
\pi_s(\mu_{\mathrm{diff}})
=
\int_0^1
g(t)\,
\Pp\!\left(
\operatorname{Bin}(N,1-t)\le s
\right)
\,\dd t.
\label{eq:pis-diffusion}
\end{equation}
Substituting $u=1-t$ and expanding the binomial tail gives
\[
\begin{aligned}
\pi_s(\mu_{\mathrm{diff}})
&=
\sum_{j=0}^{s}
\binom Nj
\int_0^1
g(1-u)\,u^j(1-u)^{N-j}
\,\dd u
\\
&=
\frac{1}{N+1}
\sum_{j=0}^{s}
\E\!\left[g(1-U_{N,j})\right],
\end{aligned}
\]
where $U_{N,j}\sim\operatorname{Beta}(j+1,N-j+1)$. The second equality uses the Beta density normalization. For fixed $j$, $\E[U_{N,j}]=(j+1)/(N+2)$, so $U_{N,j}\to0$ in probability. Since $g$ is bounded and continuous,
\[
\E[g(1-U_{N,j})]\longrightarrow g(1).
\]
Consequently, if $g(1)>0$,
\[
\pi_s(\mu_{\mathrm{diff}})
\sim
\frac{(s+1)g(1)}{N}.
\]
For the uniform density $g\equiv1$, the exact expression is $\pi_s(\mu_{\mathrm{diff}})=(s+1)/(N+1)$. The \(N^{-1}\) scale arises from masking probabilities within \(O(N^{-1})\) of full masking.

\paragraph{Masking probabilities bounded away from one.}
Clipped ranges of masking probabilities are used in block diffusion to reduce training variance \citep{arriola2025block}. Consider independent coordinate masking conditional on a masking probability $t$, with $t\le t_{\max}$ almost surely for a fixed $0<t_{\max}<1$. Monotonicity of the binomial lower tail gives
\[
\pi_s(\mu_{\mathrm{diff}}^{t_{\max}})
\le
\Pp\!\left(
\operatorname{Bin}(N,1-t_{\max})\le s
\right).
\]
For fixed $s$, the $j=s$ term dominates this tail, so
\[
\Pp\!\left(
\operatorname{Bin}(N,1-t_{\max})\le s
\right)
=
\Theta\!\left(N^s t_{\max}^{N}\right).
\]
Thus
\[
\pi_s(\mu_{\mathrm{diff}}^{t_{\max}})
\le
e^{-rN+O(\log N)},
\qquad
r:=\log(1/t_{\max})>0.
\]
To obtain the corresponding asymptotic rate, suppose the masking probability has a density $g_{\mathrm{clip}}$ continuous on $0,t_{\max}$, with $g_{\mathrm{clip}}(t_{\max})>0$. Expanding the binomial tail gives
\[
\pi_s(\mu_{\mathrm{diff}}^{t_{\max}})
=
\sum_{j=0}^{s}
\binom Nj
\int_0^{t_{\max}}
g_{\mathrm{clip}}(t)(1-t)^j t^{N-j}\,\dd t.
\]
For each fixed $j$, substituting $t=t_{\max}v$ yields
\[
\begin{aligned}
\int_0^{t_{\max}}
g_{\mathrm{clip}}(t)(1-t)^j t^{N-j}\,\dd t
=
\frac{t_{\max}^{N-j+1}}{N-j+1}
\E\!\left[
g_{\mathrm{clip}}(t_{\max}Y_{N,j})
(1-t_{\max}Y_{N,j})^j
\right],
\end{aligned}
\]
where $Y_{N,j}\sim\operatorname{Beta}(N-j+1,1)$. Since $Y_{N,j}\to1$ in probability, boundedness and continuity give
\[
\E\!\left[
g_{\mathrm{clip}}(t_{\max}Y_{N,j})
(1-t_{\max}Y_{N,j})^j
\right]
\longrightarrow
g_{\mathrm{clip}}(t_{\max})(1-t_{\max})^j.
\]
Using $\binom Nj\sim N^j/j!$, the $j$-th term therefore has order $N^{j-1}t_{\max}^N$. The $j=s$ term dominates, giving
\[
\pi_s(\mu_{\mathrm{diff}}^{t_{\max}})
\sim
\frac{
g_{\mathrm{clip}}(t_{\max})
(1-t_{\max})^s t_{\max}^{1-s}
}{s!}
N^{s-1}t_{\max}^N.
\]
In particular,
\[
\pi_s(\mu_{\mathrm{diff}}^{t_{\max}})
=
\Theta(N^{s-1}t_{\max}^N)
=
e^{-rN+O(\log N)}.
\]
The extra factor \(N^{-1}\), relative to fixed-probability masking at \(t_{\max}\), comes from averaging over masking probabilities near the endpoint.

\paragraph{Independent masking at a fixed probability.}
If each coordinate is masked independently with a fixed probability $t_0\in(0,1)$, then
\[
\pi_s(\mu_{\mathrm{ind}})
=
\Pp\!\left(
\operatorname{Bin}(N,1-t_0)\le s
\right)
=
\Theta\!\left(N^s t_0^N\right).
\]
Low-visibility masks therefore have exponentially small probability.

\begin{table}[t]
  \centering
  \caption{Low-visibility mass for fixed $s$ as $N\to\infty$, under the conditions specified in this section.}
  \label{tab:schedule-taxonomy}
  \begin{tabular}{@{}ll@{}}
    \toprule
    \textbf{Mask schedule} & $\pi_s(\mu)$ \\
    \midrule
    Continuous density $g$, with $g(1)>0$
      & $\Theta(N^{-1})$ \\
    Continuous density up to \(t_{\max}<1\), with \(g(t_{\max})>0\)
      & $\Theta(N^{s-1}t_{\max}^{N})$ \\
    Independent masking at fixed $t_0\in(0,1)$
      & $\Theta(N^s t_0^N)$ \\
    Visible count at least $\alpha N$ for every mask
      & $0$ \\
    Full-mask probability at least a fixed $\eta>0$
      & $\Theta(1)$ \\
    \bottomrule
  \end{tabular}
\end{table}

\paragraph{Exact-count and structured masking.}
Fixed-ratio masking is used in masked language modeling and masked autoencoding \citep{devlin2019bert,liu2019roberta,he2022masked}. Consider the exact-count case, in which $\lfloor t_0N\rfloor$ coordinates are masked for a fixed $t_0\in[0,1)$. Then
\[
|K^{\complement}|
=
N-\lfloor t_0N\rfloor
\ge
(1-t_0)N.
\]
More generally, suppose every mask \(K\in\supp(\mu)\) leaves at least \(\alpha N\) coordinates visible for some fixed \(\alpha>0\). Then
\[
v_{\mathrm{min}}(\mu)\ge\alpha N,
\qquad
\pi_s(\mu)=0
\quad\text{whenever }s<\alpha N.
\]
This includes exact-count masking and any structured object- or block-masking schedule that leaves a fixed positive fraction of coordinates visible, including object-masked prediction settings satisfying this condition \citep{nam2026cjepa}. In particular, $\pi_s(\mu)=0$ eventually for every sequence $s=o(N)$.

\paragraph{A full-mask component.}
If the schedule assigns probability at least $\eta>0$ to the full mask, with $\eta$ independent of $N$, then
\[
\pi_s(\mu)\ge\eta
\qquad\text{for every }s\ge0.
\]
Since \(\pi_s(\mu)\le1\), this gives \(\pi_s(\mu)=\Theta(1)\). Under masked-block log loss, this component also provides the joint-distribution control of \Cref{prop:full-mask}.

\bigskip
\section{Interpretation and Instantiation of Assumptions}
\label{app:assumption-just}

\Cref{ass:mode-pinning,ass:low-vis-uncertainty} describe how visible context informs the mode label: the first controls residual uncertainty at large visibility, while the second prevents it from decaying too quickly at low visibility. We interpret both conditions and construct a distribution satisfying them on an overlapping range of visible-set sizes.

\subsection{Large-Context Mode Pinning}
\label{app:modepinning-just}

\paragraph{Statistical-mechanics motivation.}
At zero external field and below the critical temperature, the Curie--Weiss distribution concentrates near two opposite magnetizations \citep{ellis1978statistics}. Visible spins then provide evidence about the active phase. This motivates the mode-pinning condition: with high probability under either mode, a sufficiently large visible set assigns exponentially small posterior probability to the other mode. The construction in \Cref{app:both-assumptions} verifies this behavior through explicit finite-$N$ bounds.

\paragraph{Natural-data interpretation.}
For natural data, a mode may represent a global attribute such as language, document format, or code-versus-prose structure. Such global attributes can be represented naturally through mixture or topic structure \citep{blei2003latent,hofmann1999probabilistic}. \Cref{ass:mode-pinning} makes this interpretation quantitative: the posterior concentration bound must hold with high probability under each mode, for every visible set of sufficient size.

\paragraph{Effect on masked prediction.}
When $x_V\in\cE^\tau_V$, the conditional distribution of the masked variables assigns at most $e^{-\kappa|V|}$ probability to completions in the opposite mode. Reweighting the modes while preserving the within-mode distributions changes this conditional by at most a constant multiple of $e^{-\kappa|V|}$ in total variation, provided both reweighted mode probabilities remain bounded away from zero. The weak sensitivity arises because the visible context already resolves the mode, even though both modes retain nonvanishing global probability.

\subsection{Low-Visibility Uncertainty}
\label{app:low-visibility-just}

\paragraph{Bounded evidence from small contexts.}
\Cref{ass:low-vis-uncertainty} requires $\mmse_p(Z\mid X_V)\ge u_0e^{-L|V|}$ for $|V|\le s_{\mathrm{ov}}$. This bound allows residual mode uncertainty to decrease with visibility. For fixed $|V|$ and constants independent of $N$, it prevents this uncertainty from vanishing as the dimension grows.

A sufficient condition is that the absolute log-likelihood ratio between the two mode-conditional visible distributions is bounded linearly in $|V|$ (\Cref{lem:low-vis-sufficient}). Product distributions with bounded evidence per coordinate satisfy this condition directly. The construction below shows that the bound also survives conditioning the components on disjoint mode regions.

\paragraph{Natural-data interpretation.}
Short contexts can remain compatible with several languages, topics, or document types. Studies of language identification and short-text classification report difficulties associated with limited context \citep{baldwin2010language,phan2008learning}. These observations motivate the assumption, whose formal requirement is a lower bound on the Bayes uncertainty about the mode. Under that condition, low-visibility masks preserve sensitivity to mode weights (\Cref{thm:mode-sensitivity}).

\subsection{An Explicit Distribution Satisfying Both Assumptions}
\label{app:both-assumptions}

We construct two components by conditioning biased product distributions on opposite signs of the total magnetization and combine them with mixing weight \(w\in(0,1)\). Their visible log-likelihood ratio decomposes into a product-distribution term and a correction induced by conditioning on the mode regions. Controlling this correction yields both assumptions.

\paragraph{Construction.}
Let $\cA=\{-1,+1\}$, let $\cX=\cA^N$ with $N$ odd, and fix $\theta,w\in(0,1)$. For every nonempty $V\subseteq[N]$, define
\[
S_V(x_V):=\sum\nolimits_{i\in V}x_i,
\qquad
m_V(x_V):=\frac{S_V(x_V)}{|V|},
\]
and write $m(x):=m_{[N]}(x)$. Since $N$ is odd, $m(x)\ne0$ for every $x\in\cX$. We use $\tau\in\{+,-\}$ to label the modes, interpreting these labels as $+1$ and $-1$ in arithmetic expressions. Define the partition
\[
\cX_\tau:=\{x\in\cX:\tau\,m(x)>0\},
\qquad \tau\in\{+,-\}.
\]
Let $\nu_\tau$ be the product distribution under which the coordinates are i.i.d.\ with $\nu_\tau(X_i=\tau)=(1+\theta)/2$, and set
\[
p^\tau:=\nu_\tau(\,\cdot\mid\cX_\tau),
\qquad
p:=w p^++(1-w)p^-.
\]
We use the notation
\[
\begin{gathered}
L_0:=\log\frac{1+\theta}{1-\theta},
\qquad
\alpha:=\frac{\theta}{2(1+\theta)},
\qquad
\eta_N:=2e^{-\theta^2N/8},
\qquad
Z_\theta:=\nu_+(\cX_+).
\end{gathered}
\]
The components have disjoint supports, and $p$ has full support on $\cX$. Reflection gives
\[
p^-(x)=p^+(-x),
\qquad
\nu_-(\cX_-)=Z_\theta,
\]
while the mode probabilities are $p(\cX_+)=w$ and $p(\cX_-)=1-w$.

\begin{lemma}[Exact visible log-likelihood ratio]
\label{lem:exact-llr}
For every nonempty $V\subseteq[N]$ and every $x_V\in\cA^V$,
\[
\log\frac{p_V^+(x_V)}{p_V^-(x_V)}
=
L_0S_V(x_V)+T_V(x_V),
\]
where
\[
T_V(x_V):=
\log
\frac{
\Pp_{\nu_+}(\cX_+\mid X_V=x_V)
}{
\Pp_{\nu_+}(\cX_+\mid X_V=-x_V)
}.
\]
Consequently,
\[
\logit\beta_V(x_V)
=
\logit(w)+L_0S_V(x_V)+T_V(x_V).
\]
The identities are interpreted in the extended real sense when a mode marginal vanishes.
\end{lemma}

\begin{proof}
Marginalizing the conditioned distribution gives
\[
p_V^+(x_V)
=
\frac{\nu_{+,V}(x_V)}{Z_\theta}
\Pp_{\nu_+}(\cX_+\mid X_V=x_V).
\]
By reflection,
\[
p_V^-(x_V)
=
\frac{\nu_{+,V}(-x_V)}{Z_\theta}
\Pp_{\nu_+}(\cX_+\mid X_V=-x_V).
\]
The normalizers cancel, giving
\[
\frac{p_V^+(x_V)}{p_V^-(x_V)}
=
\frac{\nu_{+,V}(x_V)}{\nu_{+,V}(-x_V)}
\frac{
\Pp_{\nu_+}(\cX_+\mid X_V=x_V)
}{
\Pp_{\nu_+}(\cX_+\mid X_V=-x_V)
}.
\]
The second factor contributes \(T_V(x_V)\). For the first factor, independence and $x_i\in\{-1,+1\}$ give
\[
\begin{aligned}
\frac{\nu_{+,V}(x_V)}{\nu_{+,V}(-x_V)}
&=
\prod_{i\in V}
\frac{\nu_+(X_i=x_i)}{\nu_+(X_i=-x_i)}
=
\prod_{i\in V}
\left(\frac{1+\theta}{1-\theta}\right)^{x_i}
\\
&=
\exp\!\left(
\log\frac{1+\theta}{1-\theta}
\sum_{i\in V}x_i
\right)\\
&=
e^{L_0S_V(x_V)}.
\end{aligned}
\]
Taking logarithms of the factorization therefore yields
\[
\log\frac{p_V^+(x_V)}{p_V^-(x_V)}
=
L_0S_V(x_V)+T_V(x_V).
\]

To obtain the posterior-logit identity, recall that
$\beta_V(x_V)=p(\cX_+\mid X_V=x_V)$.
Since the mixture components are supported on their respective mode regions,
Bayes' rule gives
\[
\beta_V(x_V)
=
\frac{w\,p_V^+(x_V)}
{w\,p_V^+(x_V)+(1-w)p_V^-(x_V)},
\qquad
1-\beta_V(x_V)
=
\frac{(1-w)p_V^-(x_V)}
{w\,p_V^+(x_V)+(1-w)p_V^-(x_V)}.
\]
Dividing these expressions and taking logarithms,
\[
\begin{aligned}
\logit\beta_V(x_V)
&=
\log\frac{\beta_V(x_V)}{1-\beta_V(x_V)}
=
\log\frac{w}{1-w}
+
\log\frac{p_V^+(x_V)}{p_V^-(x_V)}
\\
&=
\logit(w)+L_0S_V(x_V)+T_V(x_V).
\end{aligned}
\]
When either mode marginal vanishes, the corresponding identities hold under the stated extended-real convention.
\end{proof}

The term \(T_V(x_V)\) accounts for conditioning \(\nu_\tau\) on \(\cX_\tau\). For mode pinning, we use its sign on contexts whose magnetization favors the corresponding mode; for low-visibility uncertainty, we bound its magnitude uniformly over small visible sets.

\begin{lemma}[Bounds on the conditioning term]
\label{lem:T-control}
In the setting above, for every nonempty $V\subseteq[N]$, the following hold:
\begin{enumerate}
\item[(i)] If $S_V(x_V)\ge0$, then $T_V(x_V)\ge0$. Moreover,
\[
T_V(-x_V)=-T_V(x_V).
\]
\item[(ii)] If $|V|\le\alpha N$ and $\eta_N\le1$, then
\[
|T_V(x_V)|\le\eta_N \qquad\text{for every}\quad x_V\in\cA^V.
\]
\end{enumerate}
\end{lemma}

\begin{proof}
Under $\nu_+$, let
\[
S_{V^\complement}:=\sum_{i\notin V}X_i.
\]
This sum is independent of $X_V$ and has mean $\theta(N-|V|)$. Since $\cX_+=\{S_V+S_{V^\complement}>0\}$,
\[
\Pp_{\nu_+}(\cX_+\mid X_V=x_V)
=
\Pp(S_{V^\complement}>-S_V(x_V)).
\]

\emph{Part (i).}
If $S_V(x_V)\ge0$, then
\[
\Pp(S_{V^\complement}>-S_V(x_V))
\ge
\Pp(S_{V^\complement}>S_V(x_V)),
\]
which gives $T_V(x_V)\ge0$. Replacing $x_V$ by $-x_V$ interchanges the numerator and denominator in the definition of $T_V$, proving the reflection identity.

\emph{Part (ii).}
For $|V|\le\alpha N$, the distance from the mean of $S_{V^\complement}$ to the threshold $-S_V(x_V)$ satisfies
\[
\theta(N-|V|)+S_V(x_V)
\ge
\theta N-(1+\theta)|V|
\ge
\frac{\theta N}{2}.
\]
Hoeffding's inequality therefore gives
\[
\Pp(S_{V^\complement}\le-S_V(x_V))
\le
\exp\!\left(
-\frac{(\theta N/2)^2}{2(N-|V|)}
\right)
\le
e^{-\theta^2N/8}.
\]
The same bound holds with $x_V$ replaced by $-x_V$. Both conditional probabilities in the ratio defining $T_V$ thus lie in
$[1-e^{-\theta^2N/8},1]$. Using
$-\log(1-u)\le2u$ for $0\le u\le1/2$ yields
$|T_V(x_V)|\le\eta_N$.
\end{proof}

\begin{proposition}[Compatibility of mode pinning and low-visibility uncertainty]
\label{prop:product-instantiation}
Let $p$ be the distribution constructed above and set
\[
\begin{gathered}
c_0:=\min\{w,1-w\},
\qquad
\kappa:=\frac{L_0\theta}{4},
\qquad
c_1:=\frac{\theta^2}{16},
\qquad
u_0:=w(1-w),\\
s_{\mathrm{pin}}
:=
\left\lceil
\max\left\{
\frac{16\ln2}{\theta^2},
\frac{4|\logit(w)|}{L_0\theta}
\right\}
\right\rceil,
\qquad
L:=L_0+1,
\qquad
s_{\mathrm{ov}}:=\lfloor\alpha N\rfloor.
\end{gathered}
\]
There exists $N_0=N_0(\theta,w)$ 
such that, for every odd $N\ge N_0$:
\begin{enumerate}
\item[(i)] \Cref{ass:mode-pinning} holds with constants $c_0,\kappa,c_1,s_{\mathrm{pin}}$.

\item[(ii)] \Cref{ass:low-vis-uncertainty} holds with $u_0=w(1-w)$, rate $L$, and scale $s_{\mathrm{ov}}$.

\item[(iii)] The combined rate is $c=\min\{c_1,\kappa\}=\theta^2/16<L_0<L$, and the visibility ranges overlap: $s_{\mathrm{pin}}\le s_{\mathrm{ov}}$.
\end{enumerate}
\end{proposition}

\begin{formalbox}
\textbf{Proof intuition.}
With sufficiently many visible coordinates, magnetization concentrates in the direction favored by the active mode. The accumulated likelihood evidence dominates the fixed prior log-odds, pinning the mode posterior. For small visible sets, the conditioning correction is uniformly small, so the log-likelihood ratio is bounded in magnitude by a constant multiple of the visible count. This limits how quickly posterior uncertainty can decay, establishing the MMSE lower bound.
\end{formalbox}

\begin{proof}
Choose $N_0=N_0(\theta,w)$ sufficiently large that $\eta_N\le1$ and $\alpha N\ge s_{\mathrm{pin}}$ for every $N\ge N_0$.

\emph{Part (i): mode pinning.}
For $|V|\ge s_{\mathrm{pin}}$ and $\tau\in\{+,-\}$, define
\[
G_V^\tau
:=
\left\{
x_V\in\operatorname{supp}(p_V^\tau):
\tau\,m_V(x_V)\ge\theta/2
\right\}.
\]

For $x_V\in G_V^\tau$, \Cref{lem:T-control}(i) gives $\tau T_V(x_V)\ge0$. Hence
\[
\begin{aligned}
\tau\logit\beta_V(x_V)
&=
\tau\logit(w)
+
L_0\tau S_V(x_V)
+
\tau T_V(x_V)
\\
&\ge
-|\logit(w)|+\frac{L_0\theta}{2}|V|
\\
&\ge
\frac{L_0\theta}{4}|V|
=
\kappa|V|.
\end{aligned}
\]
Here, the last inequality uses $|V|\ge s_{\mathrm{pin}}\ge4|\logit(w)|/(L_0\theta)$. Therefore
\[
p(\cX_{-\tau}\mid X_V=x_V)
=
\frac{1}{1+\exp\!\bigl(\tau\logit\beta_V(x_V)\bigr)}
\le
e^{-\kappa|V|},
\]
so $G_V^\tau\subseteq\cE_V^\tau$ for both modes. For every visible event $A$,
\[
p_V^\tau(A)\le\frac{\nu_{\tau,V}(A)}{Z_\theta}.
\]
Under $\nu_{\tau,V}$, the variables $\tau X_i$ are independent with mean $\theta$, so
\[
\nu_{\tau,V}(\tau m_V<\theta/2)
\le
e^{-\theta^2|V|/8}.
\]
Also,
\[
Z_\theta
=
\Pp_{\nu_+}\!\left(\sum_iX_i>0\right)
\ge
1-e^{-\theta^2N/2}
\ge
\tfrac12.
\]
It follows that
\[
p_V^\tau\bigl((G_V^\tau)^\complement\bigr)
\le
2e^{-\theta^2|V|/8}
\le
e^{-\theta^2|V|/16}
=
e^{-c_1|V|},
\]
where the second inequality uses $|V|\ge16\ln2/\theta^2$. Thus $p_V^\tau(\cE_V^\tau)\ge1-e^{-c_1|V|}$. Since $p(\cX_\tau)\ge c_0$ for both modes, \Cref{ass:mode-pinning} holds.

\emph{Part (ii): low-visibility uncertainty.}
For $V=\varnothing$,
\[
\mmse_p(Z\mid X_V)
=
\operatorname{Var}_p(Z)
=
w(1-w).
\]
For $1\le|V|\le s_{\mathrm{ov}}$, \Cref{lem:exact-llr,lem:T-control}(ii) give
\begin{equation}
\left|
\log\frac{p_V^+(x_V)}{p_V^-(x_V)}
\right|
\le
L_0|V|+\eta_N
\le
(L_0+1)|V|
=
L|V|,
\label{eq:rate_L}
\end{equation}
where $\eta_N\le1$ for $N\ge N_0$. Applying \Cref{lem:low-vis-sufficient} yields
\[
\mmse_p(Z\mid X_V)
\ge
w(1-w)e^{-L|V|}.
\]

\emph{Part (iii): rates and overlap.}
Since $L_0=2\operatorname{artanh}(\theta)\ge2\theta$, we have $\kappa\ge\theta^2/2>c_1$. Therefore
\[
c=c_1=\frac{\theta^2}{16}<L_0<L.
\]
The choice of $N_0$ ensures $\alpha N\ge s_{\mathrm{pin}}$. Since $s_{\mathrm{pin}}$ is an integer, this gives
\[
s_{\mathrm{pin}}
\le
\lfloor\alpha N\rfloor
=
s_{\mathrm{ov}}.
\]
\end{proof}

\begin{remark}[Compatible bounds and a reference decay rate]
\label{rem:bracketing}
On the overlap window, the two assumptions give
\[
w(1-w)e^{-L|V|}
\le
\mmse_p(Z\mid X_V)
\le
\frac54e^{-c|V|},
\qquad
s_{\mathrm{pin}}\le|V|\le s_{\mathrm{ov}},
\]
with $c<L$. The mixing weight enters the lower-bound prefactor and the pinning threshold, while the rates $c$ and $L$ depend only on $\theta$. For comparison with the experiments, nearly balanced visible configurations under the unconditioned product distributions have large-deviation rate
\[
r^\star(\theta)
:=
\kl\!\left(
\frac12\,\middle\|\,\frac{1+\theta}{2}
\right)
=
-\frac12\log(1-\theta^2).
\]
This gives the reference rate used in the experiments. It satisfies $c={\theta^2}/{16}<r^\star(\theta)<L_0<L$.
\end{remark}

\bigskip
\section{Experimental Protocols and Additional Results}
\label{app:experiments}

We evaluate mode-weight sensitivity through population geometry, schedule interventions, optimization, and sampling cost, holding the within-mode distributions fixed to isolate the mode-weight signal from component-estimation error and conditional incompatibility. We also include an exploratory measurement of label uncertainty in natural text. \Cref{tab:exp-protocol} summarizes the experimental settings.

\subsection{Controlled Distributions, Estimands, and Computation}
\label{app:exp-protocol}

\paragraph{Conditioned-product mixture.}
For $\tau\in\{-1,+1\}$, let $\nu_\tau$ be the product distribution on $\{-1,+1\}^N$ satisfying $\nu_\tau(X_i=\tau)=(1+\theta)/2$, and define
\[
\cX_\tau=\big\{x:\tau\sum\nolimits_i x_i>0\big\},
\qquad
p^\tau=\nu_\tau(\,\cdot\mid\cX_\tau).
\]
The data distribution is $p=w p^+ +(1-w)p^-$, with $w\in(0,1)$ and binary mode indicator $Z:=\indic\{X\in\cX_+\}$. Setting $c_0:=\min\{w,1-w\}$, we vary only the global mode weight over a compact interval:
\[
q_\lambda=\lambda p^+ +(1-\lambda)p^-,
\qquad
\lambda\in I:=
\left[{c_0}/{2},\,1-{c_0}/{2}\right].
\]
The within-mode distributions are identical under $p$ and $q_\lambda$, and for each $\lambda$, all masked conditionals are induced by the same joint distribution \(q_\lambda\). Because the components have disjoint sign-of-magnetization supports, $\DTV(p,q_\lambda)=|w-\lambda|$. This separates the joint-distribution error from the visibility-dependent signal available to the masked objective. 

The main experiments use $\theta=0.8$ and $w=1/2$, giving $I=[1/4,3/4]$. Additional robustness evaluations cover $\theta\in\{0.6,0.8\}$ and $w\in\{1/2,0.9\}$; for $w=0.9$, the corresponding interval is $I=[0.05,0.95]$. Direct optimization uses the main setting $(\theta,w)=(0.8,1/2)$.

\paragraph{Curie--Weiss model.}
We use the zero-field Curie--Weiss model \citep{ellis1978statistics},
\[
p_{\beta,N}(x)
\propto
\exp\!\left(
\frac{\beta}{2N}\left(\sum_{i=1}^N x_i\right)^2
\right),
\qquad x\in\{-1,+1\}^N,
\]
where $\beta$ is the inverse temperature. At $\beta=1.5$, the spins interact and the magnetization has two macroscopic modes, testing whether mode blindness persists beyond the conditioned-product construction. At $\beta=0.5$, the model is in the high-temperature unimodal regime and serves as a null control without comparable large-context mode pinning. Both settings use the same sign-of-magnetization partition and mode-reweighting construction defined above. Note that the closed-form reference rate of \Cref{rem:bracketing} applies only to the conditioned-product mixture. Any Curie--Weiss reference obtained by substituting its spontaneous magnetization is used only as a heuristic visual guide.

\paragraph{Estimands.}
For the MMSE, we encode the positive and negative mode labels as $1$ and $0$, respectively. Since $p$ and $q_\lambda$ are exchangeable, the quantities below depend on the visible set $V=K^\complement$ only through its size $m=|V|$. We define
\begin{equation}
\mathfrak D_{N,m}(\lambda)
:=\mathfrak D_K(p\Vert q_\lambda),
\qquad
\Delta_{N,m}
:=\sup\nolimits_{\lambda\in I}\mathfrak D_{N,m}(\lambda),
\qquad
U_{N,m}:=\mmse_p(Z\mid X_V).
\label{eq:exp-estimands}
\end{equation}
Here \(U_{N,m}\) is the residual mode uncertainty at visibility \(m\). The maximum defining $\Delta_{N,m}$ is attained at an endpoint of $I$: the discrepancy is nonincreasing for $\lambda\le w$ and nondecreasing for $\lambda\ge w$. For $U_{N,m}>0$ and $\lambda\ne w$, we also define
\begin{equation}
\rho_{N,m}(\lambda)
:=
\frac{\mathfrak D_{N,m}(\lambda)}
{U_{N,m}\kl(w\Vert\lambda)}.
\label{eq:exp-ratio}
\end{equation}
This ratio compares the discrepancy with the sensitivity scale $U_{N,m}\kl(w\Vert\lambda)$ in \Cref{lem:sensitivity}.

\begin{table}[t]
    \centering
    \small
    \renewcommand{\arraystretch}{1.08}
    \setlength{\tabcolsep}{5pt}
    \caption{Experimental settings and evaluation protocols.}
    \label{tab:exp-protocol}
    \begin{tabularx}{\linewidth}{
        @{}
        >{\raggedright\arraybackslash}p{0.25\linewidth}
        >{\raggedright\arraybackslash}X
        @{}
    }
        \toprule
        \textbf{Component} & \textbf{Settings} \\
        \midrule

        \textbf{Exact population geometry}
        &
        Conditioned-product mixture: $\theta\in\{0.6,0.8\}$.\par 
        Curie--Weiss model: $\beta\in\{1.5,0.5\}$.\par
        Sequence lengths $N\in\{127,255,511,1023\}$ and mode weights $w\in\{1/2,9/10\}$.\par
        Visible sizes $m=|K^\complement|$ include $\{0,1,2,4\}$, dyadic sizes, and $\lfloor\varrho N\rfloor$ for $\varrho\in\{1/4,1/2,3/4\}$.\par
        Exact summation over visible magnetization.
        \\
        \midrule

        \addlinespace[0.2em]
        \textbf{Schedule intervention}
        &
        Two-point visible-size distribution $(1-\pi)\delta_{m_{\mathrm{base}}}+\pi\delta_s$, with $m_{\mathrm{base}}\in\{\lfloor N/4\rfloor,\lfloor N/2\rfloor\}$, $s\in\{0,1,2,4\}$, and $\pi\in\{0,10^{-5},10^{-4},10^{-3},10^{-2},10^{-1}\}$.\par
        Evaluation at four excess-risk budgets.
        \\
        \midrule

        \addlinespace[0.2em]
        \textbf{Optimization within $q_\lambda$}
        &
        Conditioned-product mixture with $\theta=0.8$ and $w=1/2$.\par
        Calibration over $N\in\{31,63,127\}$ selects $N=63$ and $m_{\mathrm{base}}=31$.\par
        Learning rate $\eta=0.1$ and a budget of $10^5$ steps.\par
        Stochastic runs use batch size $512$ and five seeds.
        \\
        \midrule

        \addlinespace[0.2em]
        \textbf{Sampling cost}
        &
        Conditioned-product mixture with $N=1023$, $m_{\mathrm{base}}=511$, and $\lambda=3/4$.\par
        Exact first and second moments of the per-example excess loss.
        \\
        \midrule
        
        \addlinespace[0.2em]
        \textbf{Natural-text exploration}
        &
        Code versus prose (\texttt{github-code-clean} Python~\citep{codeparrot_githubcodeclean} against C4 \texttt{en}~\citep{raffel2020exploring}) and German versus English (OPUS-100 \texttt{de-en}~\citep{tiedemann-2012-parallel,zhang-etal-2020-improving}).\par
        Disjoint sets of packed sequences for fitting, calibration, and evaluation; GPT-2 tokenization~\citep{radford2019language}.\par
        A short-context posterior estimator and calibrated long-context classifiers.
        Uncertainty estimated using $1{,}000$ bootstrap replicates obtained by resampling packed sequences.
        \\
        \bottomrule
    \end{tabularx}
\end{table}

For both controlled model families, the mode posteriors depend on the visible configuration only via its magnetization $r=2j-m$, where $j$ is the number of visible $+1$ spins. Write $\beta_V=p(Z=1\mid X_V)$ and $\beta_{\lambda,V}=q_\lambda(Z=1\mid X_V)$. Conditional on $X_V$, the positive- and negative-mode distributions of the masked block have disjoint supports and are identical under $p$ and $q_\lambda$; only their posterior mixing weights differ. Consequently, the conditional block KL divergence is exactly $\kl(\beta_V\Vert\beta_{\lambda,V})$. Each expectation therefore reduces to a sum over the $m+1$ possible visible magnetizations, weighted by their probabilities under $p$. We evaluate these sums using log-domain arithmetic.

\paragraph{Numerical validation.}
The log-domain implementation and a 50-digit reference differ by at most \(8.2\times10^{-14}\) in \(\log\mathfrak D\) on representative cells. Exhaustive enumeration at small \(N\) verifies the posterior tilt, risk identity, gradient, curvature, schedule linearity, and full-mask endpoint.

\subsection{Exact Computation of Mode-Blindness Geometry}
\label{app:exp-e1}

\begin{figure}[tbp]
    \centering
    \begin{subfigure}[t]{0.32\linewidth}
        \centering
        \includegraphics[width=\linewidth]{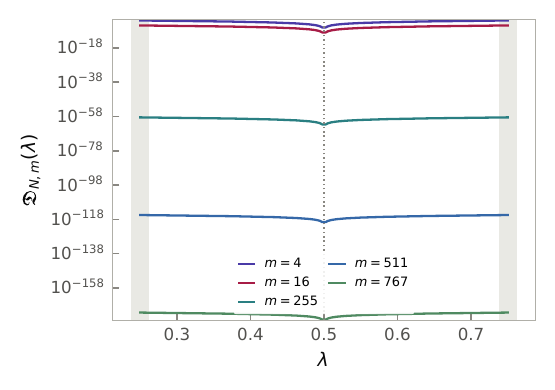}
        \subcaption{Conditioned-product mixture.}
        \label{fig:exp-profiles-pinned}
    \end{subfigure}\hfill
    \begin{subfigure}[t]{0.32\linewidth}
        \centering
        \includegraphics[width=\linewidth]{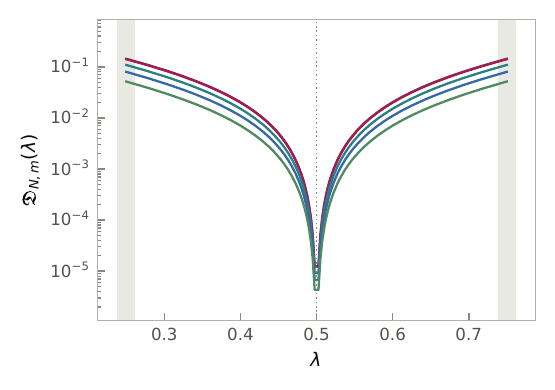}
        \subcaption{Curie--Weiss null control.}
        \label{fig:exp-profiles-null}
    \end{subfigure}\hfill
    \begin{subfigure}[t]{0.32\linewidth}
        \centering
        \includegraphics[width=\linewidth]{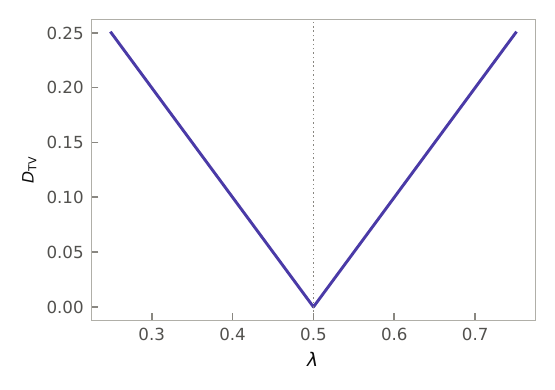}
        \subcaption{Joint total-variation error.}
        \label{fig:exp-profiles-tv}
    \end{subfigure}
    \caption{Mode-reweighting geometry at $N=1023$ and $w=1/2$. (a--b) Discrepancy profiles $\lambda\mapsto\mathfrak D_{N,m}(\lambda)$ at different visible sizes $m$ for the conditioned-product mixture ($\theta=0.8$) and Curie--Weiss control ($\beta=0.5$). Shaded strips mark the endpoints of $I=[1/4,3/4]$. (c) Joint TV error, common to both constructions and all visible sizes.}
    \label{fig:exp-profiles}
\end{figure}

\paragraph{Discrepancy profiles.}
\Cref{fig:exp-profiles-pinned} shows the conditioned-product profiles at $\theta=0.8$, $w=1/2$, and $N=1023$. With little visible context, uncertainty about $Z$ makes the masked conditional sensitive to the mixture weight. As visibility increases, the context identifies the mode more reliably, so changing its prior weight has less effect on conditional prediction. The discrepancy profiles consequently flatten across the reweighting interval. The Curie--Weiss control at $\beta=0.5$ uses the same sign-of-magnetization partition and reweighting construction (\Cref{fig:exp-profiles-null}). Its magnetization distribution is unimodal, so the sign partition does not correspond to two separated macroscopic modes. Its profiles retain appreciable curvature and show much weaker suppression over the displayed visibility range. For both constructions, the disjoint component supports give $\DTV(p,q_\lambda)=|w-\lambda|$, independently of visibility (\Cref{fig:exp-profiles-tv}). Increasing visibility can therefore suppress the masked discrepancy while leaving the joint-distribution error unchanged.

\paragraph{Decay rates.}
At $\theta=0.8$ and $N=1023$, the maximal discrepancy $\Delta_{N,m}$ decreases from $1.44\times10^{-1}$ at $m=0$ to $4.33\times10^{-173}$ at $m=767$, a reduction of more than \(171\) orders of magnitude. Across the evaluated system sizes, the fitted conditioned-product rates $\hat{c}_N$ approach the closed-form reference
\(
r^{\star}(\theta)
=
-1/2\log(1-\theta^2)
\)
from above (\Cref{fig:exp-rates-measured}). At $N=1023$, the relative errors are $0.27\%$ for $\theta=0.8$ and $0.58\%$ for $\theta=0.6$. The comparison is with the reference rate \(r^\star(\theta)\) in \Cref{rem:bracketing}, not with the conservative mode-pinning exponent \(c=\theta^2/16\) certified by \Cref{prop:product-instantiation}; the theorem guarantees exponential suppression, while agreement with \(r^\star(\theta)\) is specific to this construction. The unimodal control's fitted rate instead decreases with $N$. At each fixed visible fraction $\varrho\in\{1/4,1/2,3/4\}$, $\log\Delta_{N,\lfloor\varrho N\rfloor}$ decreases approximately linearly with $N$ (\Cref{fig:exp-rates-fixed-fraction}), consistent with the predicted exponential suppression in system size.

\begin{figure}[tbp]
    \centering
    \begin{subfigure}[t]{0.5\linewidth}
        \centering
        \includegraphics[width=\linewidth]{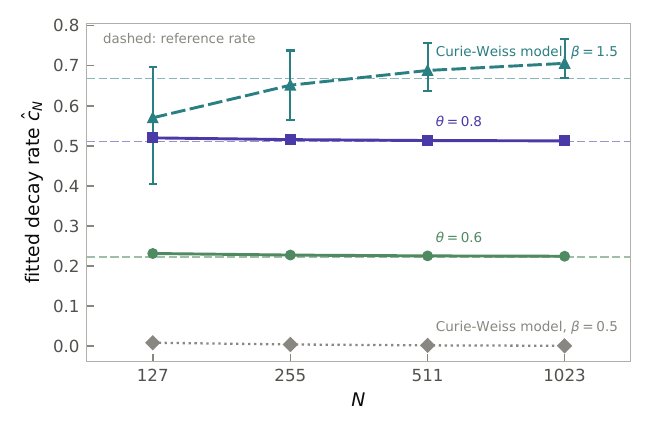}
        \subcaption{Fitted and reference decay rates.}
        \label{fig:exp-rates-measured}
    \end{subfigure}\hfill
    \begin{subfigure}[t]{0.4\linewidth}
        \centering
        \includegraphics[width=\linewidth]{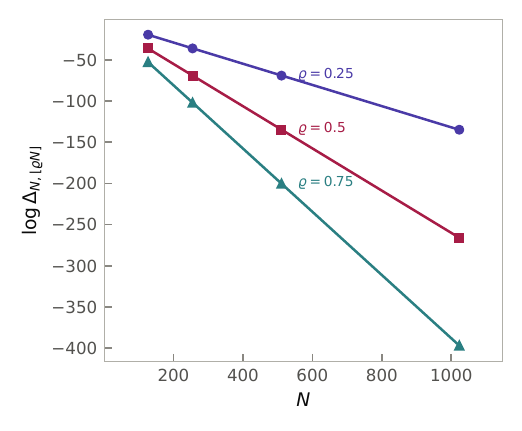}
        \subcaption{Decay at fixed visible fractions.}
        \label{fig:exp-rates-fixed-fraction}
    \end{subfigure}
    \caption{Decay rates across system sizes. (a) Fitted rates for the conditioned-product mixture and both Curie--Weiss settings. The conditioned-product reference is given in closed form; the interacting Curie--Weiss reference is a heuristic visual guide. (b) Log discrepancy $\log\Delta_{N,\lfloor\varrho N\rfloor}$ versus $N$ at fixed visible fractions $\varrho\in\{1/4,1/2,3/4\}$.}
    \label{fig:exp-rates}
\end{figure}

\paragraph{Residual uncertainty governs reweighting sensitivity.}
The normalized ratio \(\rho_{N,m}(\lambda)\) in \Cref{eq:exp-ratio} separates the effects of mode-weight divergence and residual uncertainty. \Cref{lem:sensitivity} bounds this ratio uniformly over a fixed weight interval, whereas \Cref{eq:exact-sensitivity} identifies its limiting value only as \(\lambda\to w\). For \(w=1/2\), across the 98 evaluated \((N,m)\) settings—49 for each value of \(\theta\)—and all retained \(\lambda\)-grid points with \(|\lambda-w|\ge0.01\), we obtain \(\rho_{N,m}(\lambda)\in[4.000,4.408]\). Thus, the normalization remains informative across finite reweightings, beyond the local limit \(1/(w(1-w))=4\). The asymmetric setting \(\theta=0.8\), \(w=0.9\) gives ratios in \([6.550,13.825]\), with local reference \(1/(w(1-w))\simeq11.1\). Residual uncertainty therefore accounts for the large visibility-dependent changes in discrepancy in both symmetric and asymmetric mixtures.

\subsection{Low-Visibility Schedule Intervention}
\label{app:exp-e2}

\paragraph{Mask schedule and estimands.}
We fix the sequence length at $N=1023$. Let $\delta_m$ denote a schedule supported on masks leaving exactly $m$ coordinates visible. For $s<m_{\mathrm{base}}$, we evaluate
\begin{equation}
\mu
=
(1-\pi)\delta_{m_{\mathrm{base}}}+\pi\delta_s,
\qquad
\pi_s(\mu)=\pi.
\label{eq:exp-two-point}
\end{equation}
A fraction \(\pi\) of examples leave \(s\) coordinates visible, while the remainder use the baseline visibility \(m_{\mathrm{base}}\). We test $s\in\{0,1,2,4\}$, including both the full mask ($s=0$) and non-full masks.

We compute the curvature with respect to $\lambda$ and the admissible mode-weight recovery radius:
\begin{align}
\varkappa(\pi,s)
&:=
\left.
\partial_\lambda^2
\mathfrak D_\mu(p\Vert q_\lambda)
\right|_{\lambda=w}
=
\frac{
\E_{K\sim\mu}
\left[
\mmse_p(Z\mid X_{K^\complement})
\right]
}{
\bigl(w(1-w)\bigr)^2
},
\label{eq:exp-curvature}\\
r(\varepsilon)
&:=
\sup\Big\{
|w-\lambda|:
\lambda\in I,\ 
\mathfrak D_\mu(p\Vert q_\lambda)
\le \varepsilon
\Big\}.
\label{eq:exp-radius}
\end{align}
The discrepancy and curvature are exact weighted averages of their fixed-visibility values. Hollow markers indicate interval saturation: in the symmetric setting, \(r(\varepsilon)=1/4\) when every weight in \(I\) satisfies the tolerance. The resulting plateau is therefore imposed by the reweighting interval, not by optimization.

\begin{figure}[tbp]
    \centering
    \begin{subfigure}[t]{0.32\linewidth}
        \centering
        \includegraphics[width=\linewidth]{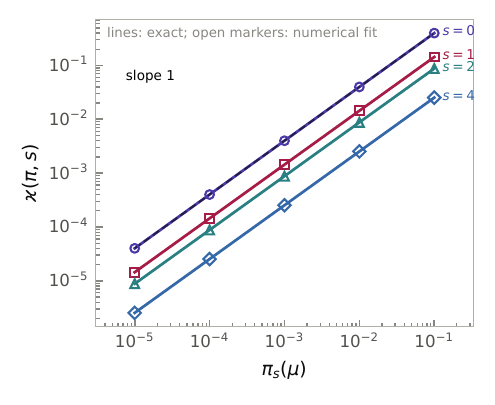}
        \subcaption{Curvature vs. $\pi_s(\mu)$.}
        \label{fig:exp-boost-curvature}
    \end{subfigure}\hfill
    \begin{subfigure}[t]{0.32\linewidth}
        \centering
        \includegraphics[width=\linewidth]{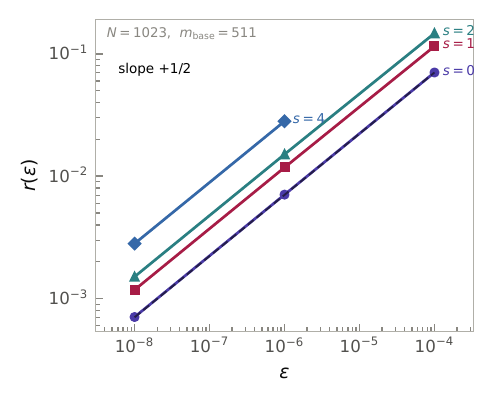}
        \subcaption{Recovery radius vs. tolerance.}
        \label{fig:exp-boost-radius}
    \end{subfigure}\hfill
    \begin{subfigure}[t]{0.32\linewidth}
        \centering
        \includegraphics[width=\linewidth]{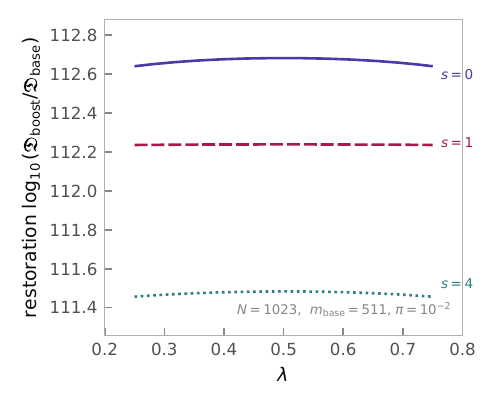}
        \subcaption{Discrepancy under intervention.}
        \label{fig:exp-boost-restoration}
    \end{subfigure}
    \caption{Low-visibility interventions at sequence length $N=1023$ and the baseline number of visible coordinates $m_{\mathrm{base}}=511$, with $s\in\{0,1,2,4\}$. (a) Curvature versus intervention mass $\pi$; lines show exact values and open markers show numerical fits. (b) Recovery radius versus tolerance at $\pi=10^{-2}$. (c) Discrepancy profiles at $\pi=10^{-2}$ compared with the blind baseline.}
    \label{fig:exp-boost}
\end{figure}

\paragraph{Recovery scaling.}
When the low-visibility contribution dominates the baseline, the curvature has a fitted log--log slope of $1.0000$ against $\pi$ (\Cref{fig:exp-boost-curvature}), consistent with linear schedule averaging. For fixed $s$ and sufficiently small tolerances with the interval boundary inactive, the local quadratic approximation gives
\[
r(\varepsilon)
\simeq
\sqrt{\frac{2\varepsilon}{\varkappa(\pi,s)}}
\asymp
\sqrt{\frac{\varepsilon}{\pi}}.
\]
At $\pi=10^{-2}$, the radius has a fitted log--log slope of $0.4995$ against $\varepsilon$ (\Cref{fig:exp-boost-radius}). At $\varepsilon=10^{-4}$, its slope against full-mask mass $\pi_0(\mu)$ is $-0.4894$ (\Cref{fig:exp-recovery}) and approaches $-1/2$ as the tolerance decreases.

\paragraph{Restored mode-weight sensitivity.}
The \(s=1\) and \(s=4\) interventions strengthen the discrepancy without assigning any probability to the full mask, directly illustrating directional recovery from nonempty visible contexts. The \(s=0\) intervention supplies the joint-KL endpoint: at \(\pi_0(\mu)=10^{-2}\), it increases the nonzero reweighting discrepancy by a factor of approximately \(4.8\times10^{112}\) relative to the blind baseline (\Cref{fig:exp-boost-restoration}). These measurements concern recovery along \(q_\lambda\) with fixed within-mode distributions; for \(s=0\), \Cref{prop:full-mask} provides the stronger guarantee over all admissible joint distributions.

\subsection{Direct Optimization within \texorpdfstring{$q_\lambda$}{q-lambda}}
\label{app:exp-training}

\paragraph{Likelihood and curvature.}
We optimize only \(a\in\mathbb R\), with \(\lambda=\logit^{-1}(a)\) ranging over \((0,1)\), while keeping the within-mode distributions fixed. The binary label \(Z=\indic\{X\in\cX_+\}\) is determined by the same complete sample \(X\) used for masked prediction; it is not an externally supplied annotation. For \(V=K^\complement\), the shared-component structure gives 
\[
-\log q_{\lambda,K}(X_K\,|\,X_V) = - Z\log \beta_{\lambda, V}(X_V) - (1-Z)\log\big(1- \beta_{\lambda,V}(X_V)\big) - \log p(X_K\,|\, X_V,Z).
\]
The final term is independent of \(a\). Consequently, optimizing mode-label cross-entropy gives the same parameter gradients as the joint masked-block negative log-likelihood. Because the visible context enters through the posterior \(\beta_{\lambda,V}\), the experiment does not reduce masked prediction to unconditional fitting of mode frequencies.

Let $R_\mu(a)$ denote the population masked-prediction risk of $q_\lambda$, and let $a^\star=\operatorname{logit}(w)$ denote the true parameter. The masked discrepancy is the excess risk above $R_\mu(a^\star)$:
\begin{equation}
\mathfrak D_\mu(p\Vert q_\lambda)
=
R_\mu(a)-R_\mu(a^\star)
=
\E_{K\sim\mu,\,X\sim p}
\kl\!\left(
\beta_{K^\complement}
\Vert
\beta_{\lambda,K^\complement}
\right).
\label{eq:exp-optimization-risk}
\end{equation}
The curvature at the truth is therefore
\(
H_\mu
:=
\left.\partial_a^2R_\mu(a)\right|_{a=a^\star}
=
\E_{K\sim\mu}
\left[
\mmse_p(Z\mid X_{K^\complement})
\right].
\)
Near $a^\star$, population gradient descent with learning rate $\eta>0$ contracts the logit error by approximately $1-\eta H_\mu$ per step. For $\eta H_\mu\ll1$, the corresponding half-time is $t_{1/2}\simeq \log2/(\eta H_\mu)$. Since $\lambda-w\simeq w(1-w)(a-a^\star)$ locally, this also predicts the time to halve the mode-weight error. At each checkpoint, we record $\widehat\lambda_t=\logit^{-1}(a_t)$ and the excess risk.

\begin{wrapfigure}{r}{0.44\textwidth}
    \vspace{-15pt}
    \centering
    \includegraphics[width=\linewidth]{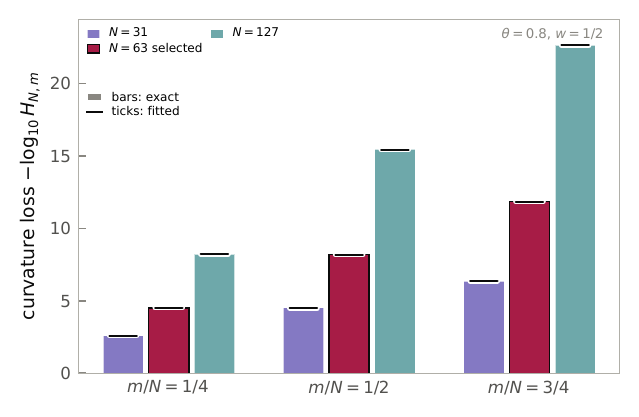}
    \vspace{-15pt}
    \caption{Curvature $H_{N,m}=U_{N,m}$ across visible fractions and system sizes, shown as $-\log_{10}H_{N,m}$. Filled bars give exact values; horizontal ticks show log-linear fits over the middle half of the visibility range. Outlined bars identify $N=63$, selected for optimization before inspecting trajectories.}
    \label{fig:exp-optimization-curvature}
\end{wrapfigure}

\paragraph{Calibration and optimizer.}
Before running gradient descent, we perform an exact population calibration over $N\in\{31,63,127\}$. For each $N$, we set $m_{\mathrm{base}}=\lfloor N/2\rfloor$ and evaluate the same two-point schedule family used in the low-visibility intervention. We use $s\in\{0,1,4\}$ and the schedule-mass grid
\(
\pi\in
\{0,10^{-4},3\times10^{-4},10^{-3},3\times10^{-3},
10^{-2},3\times10^{-2},10^{-1}\},
\)
where $\pi_s(\mu)=\pi$ is the probability of leaving $s$ coordinates visible. For each setting, we compute the exact population curvature and corresponding local half-time prediction. The case $\pi=0$ is the blind baseline $\delta_{m_{\mathrm{base}}}$ and is independent of $s$. 

Before inspecting any optimization trajectory, a prespecified rule selects the largest numerically resolvable dimension for which the blind predicted half-time is at least ten times the $10^5$-step budget. The rule also requires at least three $s=1$ interventions with predicted half-times between $500$ and $80{,}000$ steps and spanning at least a factor of ten in curvature. It selects $N=63$, with $m_{\mathrm{base}}=31$ visible coordinates under the blind baseline (\Cref{fig:exp-optimization-curvature}).

We choose the learning rate from $\{0.025,0.05,0.1\}$ using fixed-seed, $1000$-step pilot runs on the globally highest-curvature configuration, initialized at $\lambda_0\in\{0.1,0.9\}$. We retain the largest rate for which all pilot runs remain finite, the logit parameter satisfies $|a|\le 20$, and the exact excess risk does not exceed the prespecified $25\%$ risk-growth bound. The resulting rate, $\eta=0.1$, is fixed across all reported runs. Population gradient descent uses $\lambda_0\in\{0.1,0.3,0.4,0.6,0.7,0.9\}$. Minibatch SGD uses $\lambda_0\in\{0.4,0.6\}$, batch size $512$, and five independent seeds. Each run has a maximum budget of $10^5$ steps, with evaluation checkpoints every $100$ steps. A stochastic crossing is accepted only when three consecutive checkpoints satisfy $|\widehat{\lambda}_t-w|\le \tfrac12|\lambda_0-w|$. The reported half-time is the first checkpoint in this sustained three-checkpoint crossing. Runs without a sustained crossing are reported as right-censored.

\paragraph{Recovery times.}
\Cref{tab:exp-training} compares predicted and observed half-times at $\lambda_0=0.4$. For the displayed interventions, population half-times differ from the local prediction by at most \(2.3\%\), while stochastic medians follow the same ordering. Under the blind schedule, neither population gradient descent nor any of the five stochastic runs reaches the threshold within the budget. Calibration confirms that the blind gradient remains numerically resolvable, so the absence of recovery is not a floating-point underflow artifact.

The primary intervention uses $s=1$, while the $s=0$, $\pi=3\times10^{-2}$ schedule provides a full-mask boundary control. Additional runs with $s=4$ exhibit the same inverse-curvature ordering. Over the full $10^5$-step budget, the $s=1$, $\pi=3\times10^{-3}$ population run moves from $\lambda_0=0.4$ to $\lambda=0.493$, whereas the blind population run ends at $\lambda=0.400006$. Uncensored stochastic runs terminate once the sustained half-error crossing has been confirmed. Their terminal iterates therefore indicate recovery to the prescribed threshold, not convergence to $w$. Stochastic runs that never satisfy the criterion continue to the full budget and remain right-censored.

\begin{table}[t]
    \centering
    \small
    \caption{Error-halving times in steps at $N=63$, $\lambda_0=0.4$, and $\eta=0.1$. Predictions use $\log 2/(\eta H_\mu)$. SGD entries report medians over five seeds, with bootstrap intervals in brackets. Entries marked $>10^5$ indicate that the half-error criterion was not met within the step budget.}
    \label{tab:exp-training}
    \begin{tabular}{@{}lrrrr@{}}
        \toprule
        \textbf{Schedule} & $H_\mu$ & \textbf{Predicted} & \textbf{Population GD} & \textbf{SGD} median [interval] \\
        \midrule
        Blind
        & $6.23\!\times\!10^{-9}$
        & $1.11\!\times\!10^9$
        & $>10^5$
        & $>10^5$ \\
        $s=1,\ \pi=3\times10^{-3}$
        & $2.70\!\times\!10^{-4}$
        & $2.57\!\times\!10^4$
        & $2.59\!\times\!10^4$
        & $2.68\,[2.47,3.06]\!\times\!10^4$ \\
        $s=0,\ \pi=3\times10^{-2}$
        & $7.50\!\times\!10^{-3}$
        & $9.24\!\times\!10^2$
        & $9.45\!\times\!10^2$
        & $1.00\,[1.00,1.10]\!\times\!10^3$ \\
        $s=1,\ \pi=10^{-1}$
        & $9.00\!\times\!10^{-3}$
        & $7.70\!\times\!10^2$
        & $7.76\!\times\!10^2$
        & $9.00\,[8.00,9.00]\!\times\!10^2$ \\
        \bottomrule
    \end{tabular}
\end{table}


\subsection{Sampling Cost of Mode-Weight Errors}
\label{app:mask-cost}

\paragraph{Reference sample size.}
To express the strength of the restored objective signal in sampling units, define the per-example excess loss
\[
S: =
\log p_K(X_K\mid X_{K^\complement})
-
\log q_{\lambda,K}(X_K\mid X_{K^\complement}),
\qquad
X\sim p,\quad K\sim\mu.
\]
Its expectation is the schedule-averaged masked discrepancy $\E[S]=\mathfrak D_\mu(p\Vert q_\lambda)$. We define the reference sample size
\begin{equation}
n^\star
:=
\frac{\operatorname{Var}(S)}{\E[S]^2}.
\label{eq:exp-nstar}
\end{equation}
At the reference sample size $n^\star$, the mean excess loss equals one standard error of its sample average. Because \(n^\star\) depends only on the signal-to-noise ratio of \(S\), it provides a common scale for comparing the sampling cost of a fixed mode-weight perturbation across schedules.

\begin{figure}[t]
    \centering
    \begin{subfigure}[t]{0.32\linewidth}
        \centering
        \includegraphics[width=\linewidth]{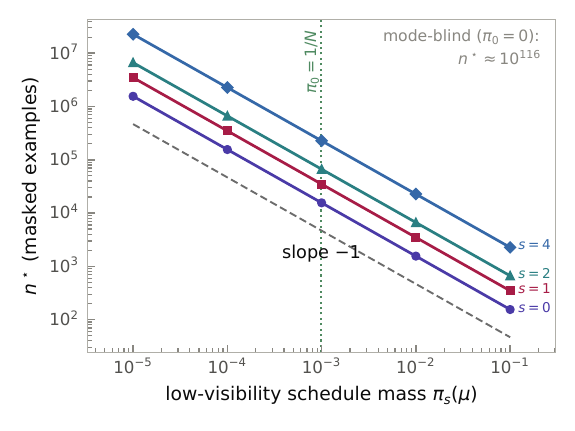}
        \subcaption{Cost versus low-visibility mass.}
        \label{fig:exp-cost-mass}
    \end{subfigure}\hfill
    \begin{subfigure}[t]{0.32\linewidth}
        \centering
        \includegraphics[width=\linewidth]{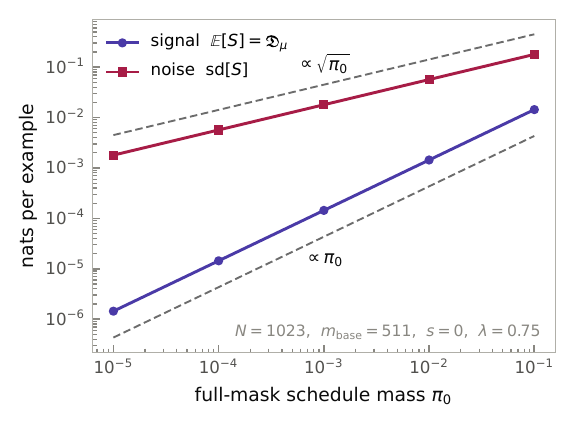}
        \subcaption{Signal and noise scaling.}
        \label{fig:exp-cost-signal}
    \end{subfigure}\hfill
    \begin{subfigure}[t]{0.32\linewidth}
        \centering
        \includegraphics[width=\linewidth]{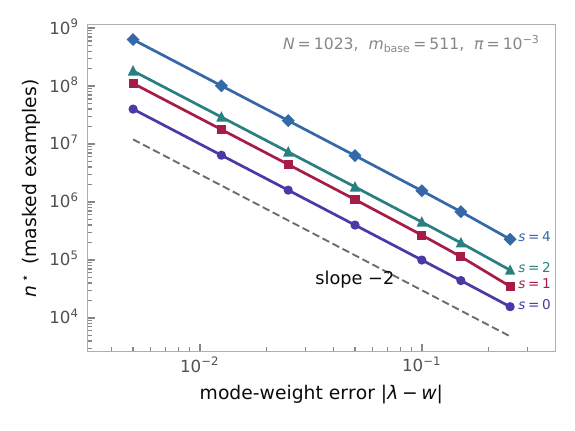}
        \subcaption{Cost versus mode-weight error.}
        \label{fig:exp-cost-resolution}
    \end{subfigure}
    \caption{Reference sample size for mode-weight errors in the conditioned-product mixture at $N=1023$, $m_{\mathrm{base}}=511$, $\theta=0.8$, and $w=1/2$. (a) At the endpoint $\lambda=3/4$, the reference sample size scales as $n^\star\propto\pi_s(\mu)^{-1}$ in the regime dominated by the low-visibility contribution. The vertical reference marks $\pi_0=1/N$. (b) For the $s=0$ channel, the mean excess loss is proportional to $\pi_0$, whereas its standard deviation is proportional to $\sqrt{\pi_0}$. (c) At $\pi_s(\mu)=10^{-3}$, the reference sample size scales as $n^\star\propto|\lambda-w|^{-2}$ near the truth.}
    \label{fig:exp-cost}
\end{figure}

\paragraph{Sampling cost and per-example loss.}
\Cref{fig:exp-cost-mass} shows that, at $\pi_0=10^{-3}\approx 1/N$ and $\lambda=3/4$, the reference sample size is $n^\star=1.56\times10^4$. Under the blind baseline $\pi_0=0$, it is $5.94\times10^{115}$. Assigning probability approximately $1/N$ to full-mask examples thus reduces $n^\star$ by more than $111$ orders of magnitude for this conditioned-product mixture. In this range, the mean excess loss is proportional to $\pi_s(\mu)$, while its standard deviation is proportional to $\sqrt{\pi_s(\mu)}$ (\Cref{fig:exp-cost-signal}). Substitution into \Cref{eq:exp-nstar} gives $n^\star\propto\pi_s(\mu)^{-1}$. At fixed low-visibility mass, the reference sample size increases as $n^\star\propto|\lambda-w|^{-2}$ when the mode-weight error decreases (\Cref{fig:exp-cost-resolution}).

We also compute the Bayes conditional log loss \(H(X_K\mid X_{K^\complement})\) for each mask size. At \(N=1023\), masks leaving \(s=1,2,4\) coordinates visible have conditional entropies only \(0.21\%\), \(0.35\%\), and \(0.58\%\) below the full-mask value \(H(X)\), respectively. Thus, these non-full low-visibility masks retain nearly the full-mask prediction difficulty while preserving mode-weight sensitivity. The comparison separates overall conditional entropy from sensitivity to a particular global-distributional error. The \(s=0\) endpoint has the distribution-free guarantee of \Cref{prop:full-mask}, whereas the measured \(s>0\) behavior concerns the conditioned-product mixture.

\subsection{Exploratory Natural-Text Data Geometry}
\label{app:exp-corpora}

We examine how estimated residual label uncertainty changes with context visibility in two binary classification settings: code versus prose and German versus English. Here $Z$ records dataset membership; its interpretation as an intrinsic partition into modes is not assumed. 

\paragraph{Natural-text data.}
For code--prose, we compare permissively licensed Python files from \texttt{github-code-clean}~\citep{codeparrot_githubcodeclean} with English C4 documents~\citep{raffel2020exploring}. For German--English, we use the two sides of OPUS-100~\citep{tiedemann-2012-parallel,zhang-etal-2020-improving}. We use GPT-2 tokenization~\citep{radford2019language} for both classes and pack documents to $1{,}152$ tokens, so the packed-document length cannot reveal the label. The realized datasets contain $16{,}000$ packed sequences per class for code--prose and $14{,}718$ per class for German--English. We divide the packed sequences into disjoint $80/10/10$ fitting, calibration, and evaluation splits. Each class contributes $15{,}000$ sampled windows at every visibility and split used by its estimator.

\paragraph{Estimators and residual uncertainty.}
At visible size $|V|\in\{1,2,4,8\}$, we construct a posterior proxy from the ratio of class-specific Stupid Backoff $n$-gram scores. At $|V|\in\{12,16,24,32,64,128,256,512,1024\}$, we fit logistic classifiers with unigram or unigram-plus-hashed-bigram features and calibrate their scores by isotonic regression on the calibration split. The bigram features allow us to assess the effect of including local token order.

For each estimator and visibility, evaluation uses $n_{\mathrm{ev}}=30{,}000$ windows, with $15{,}000$ from each class. Each window is sampled by choosing a packed sequence uniformly from its class and then choosing a valid starting position uniformly within that sequence. Let $\widehat\beta_j$ denote the predicted positive-class probability for evaluation window $j$, and let $z_j$ be its binary label. We report the plug-in uncertainty estimate and held-out Brier score
\[
\widehat U_V
:=
\frac{1}{n_{\mathrm{ev}}}
\sum\nolimits_{j=1}^{n_{\mathrm{ev}}}
\widehat\beta_j(1-\widehat\beta_j),
\qquad
\widehat B_V
:=
\frac{1}{n_{\mathrm{ev}}}
\sum\nolimits_{j=1}^{n_{\mathrm{ev}}}
(z_j-\widehat\beta_j)^2.
\]
The plug-in quantity $\widehat U_V$ provides a posterior-based proxy for residual label uncertainty, while the held-out Brier score $\widehat B_V$ estimates the population Brier risk, which upper-bounds Bayes MMSE. We obtain $95\%$ percentile intervals for $\widehat U_V$ from $1{,}000$ bootstrap replicates. Each replicate resamples the represented packed sequences with replacement, retains all sampled windows belonging to each selected sequence, and recomputes the window average, with the fitted predictors and calibration maps held fixed. The shaded regions between the \(\widehat U_V\) and \(\widehat B_V\) curves show the gap between the posterior-based uncertainty proxy and held-out Brier error.

\begin{figure}[t]
    \centering
    \begin{subfigure}[t]{0.48\linewidth}
        \centering
        \includegraphics[width=\linewidth]{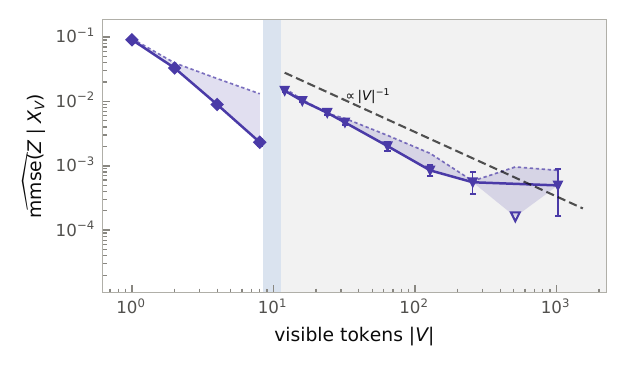}
        \vspace{-2em}
        \subcaption{Code--prose.}
        \label{fig:exp-corpora-code-prose}
    \end{subfigure}\hfill
    \begin{subfigure}[t]{0.48\linewidth}
        \centering
        \includegraphics[width=\linewidth]{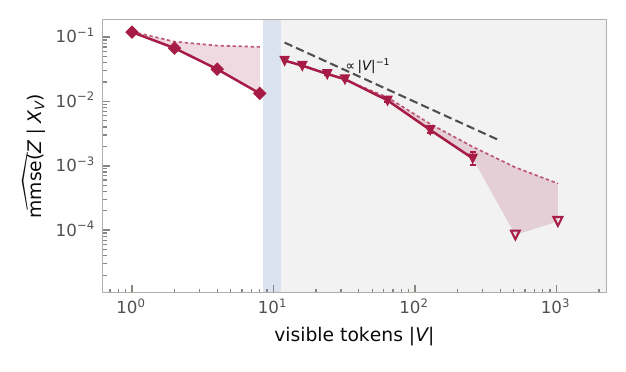}
        \vspace{-2em}
        \subcaption{German--English.}
        \label{fig:exp-corpora-german-english}
    \end{subfigure}
    \caption{
        Estimated residual label uncertainty in natural-text data. (a) Code--prose. (b) German--English. In both panels, short- and long-context measurements use separate estimators. Error bars show bootstrap intervals obtained by resampling packed sequences. Shaded bands extend from the estimated residual uncertainty to the held-out Brier score. Hollow markers indicate censored cells.
    }
    \label{fig:exp-corpora-result}
\end{figure}

\begin{table}[t]
    \centering
    \small
    \caption{Exponential and power-function fits over the same uncensored classifier cells. Parentheses report the coefficient of determination, $R^2$.}
    \label{tab:exp-corpora-form}
    \begin{tabular}{@{}llrr@{}}
        \toprule
        \textbf{Corpus} & \textbf{Features} & \textbf{Exponential rate} & \textbf{Power exponent} \\
        \midrule
        Code--prose
        & unigram
        & $0.0110\ (0.861)$
        & $-0.942\ (0.9998)$ \\
        Code--prose
        & unigram$+$bigram
        & $0.0126\ (0.809)$
        & $-1.109\ (0.993)$ \\
        German--English
        & unigram
        & $0.0143\ (0.943)$
        & $-1.150\ (0.974)$ \\
        German--English
        & unigram$+$bigram
        & $0.0142\ (0.939)$
        & $-1.149\ (0.977)$ \\
        \bottomrule
    \end{tabular}
\end{table}

\paragraph{Calibration and censoring.}
For the classifier estimates, isotonic predictions are clipped to $[f,1-f]$, where $f:=\frac{1}{2n_{\mathrm{cal}}}$ and $u_{\mathrm{clip}}:=f(1-f)$. Here $n_{\mathrm{cal}}=30{,}000$ is the number of calibration windows, so $u_{\mathrm{clip}}\approx1.66664\times10^{-5}$. We censor a classifier cell if $\widehat U_V\le1.05\,u_{\mathrm{clip}}$ or $\widehat B_V>2\widehat U_V$. The first criterion detects estimates near the imposed clipping floor; the second detects a large discrepancy between posterior variance and observed prediction error. The short-context posterior proxies are not calibrated and are not subject to this censoring rule. The study contains $44$ corpus--estimator--visibility cells: eight short-context cells and $36$ classifier cells. We analyze the two estimator ranges separately and compare exponential and power-function fits over the common uncensored classifier range.

\paragraph{Uncertainty decay with visibility.}
The reported uncertainty decreases by approximately $183\times$ on code--prose between $|V|=1$ and $1024$, and by $91\times$ on German--English between $|V|=1$ and its largest uncensored context, $256$ (\Cref{fig:exp-corpora-result}). These ratios compare short-context backoff estimates with long-context estimates from the unigram-plus-hashed-bigram classifier. Within the common uncensored classifier range, $12\le|V|\le256$, power-function fits have higher $R^2$ than exponential fits, with exponents ranging from $-1.15$ to $-0.94$ (\Cref{tab:exp-corpora-form}). Adding hashed bigrams lowers Brier error on most cells while leaving the fitted decay exponents broadly similar. Seven of the $36$ classifier cells are censored, all at $|V|\in\{512,1024\}$ because their held-out Brier scores exceed twice their plug-in uncertainty estimates. No cell is censored by the clipping-floor criterion.

\paragraph{Relation to the controlled experiments.}
These measurements quantify dataset-membership uncertainty under fitted predictors rather than exact mode posteriors or joint-distribution recovery. The observed preference for power-function fits is descriptive over the uncensored classifier range, neither verifying nor precluding exponential decay in Bayes uncertainty. The plug-in estimate is not a certified bound on Bayes MMSE (unlike the population Brier risk, which yields an upper bound). Limited predictor capacity, calibration errors, and censoring at the largest context sizes prevent strict asymptotic interpretations. The controlled experiments instead compute Bayes uncertainty exactly and test its relation to masked discrepancy and mode-weight recovery.

\bigskip
\section{Discussion}
\label{app:discussion}

\subsection{Implications for Mask Schedule Design}
\label{app:implications}

Mask schedules determine both the conditional prediction problems presented during training and which aspects of the joint distribution remain visible to the objective. Within the two-mode reweighting family, \Cref{thm:mode-sensitivity} identifies schedule-averaged residual mode uncertainty as the governing sensitivity scale. Under mode pinning, this uncertainty decays exponentially with visibility (\Cref{thm:mode-blind}), but the sensitivity characterization itself does not require exponential decay. More generally, the rate at which residual mode uncertainty decreases with visibility determines how quickly sensitivity to global mode weights is lost.

This suggests a simple schedule-level intervention: retain positive probability mass on contexts for which the mode is not yet resolved. Under \Cref{ass:low-vis-uncertainty}, assigning mass $\eta$ to masks leaving at most a fixed $s\le s_{\mathrm{ov}}$ visible coordinates yields an $O(\sqrt{\varepsilon/\eta})$ mode-weight recovery guarantee within the reweighting family. Low-visibility mass therefore provides a controllable source of supervision for global frequencies. The trade-off is that reducing visibility strengthens sensitivity to global mode weights while making the conditional prediction problem potentially harder. \Cref{app:mask-cost} further quantifies how this restored signal affects the sampling cost of detecting a fixed mode-weight perturbation.

\subsection{Limitations and Open Directions}
\label{app:limitations}

\paragraph{Scope of the recovery guarantees.}
Our results concern population-level, objective-level identifiability. The low-visibility guarantees apply to a structured family that varies only the global mode weights; extending them to the full identifiability modulus requires controlling arbitrary directions in distribution space. Positive full-mask probability provides one distribution-free route (\Cref{prop:full-mask}), but stable recovery without full masking may still be possible under additional structure.

\paragraph{Objectives and prediction heads.}
The positive guarantees also rely on masked-block log loss with conditionals induced by a common joint distribution. Mode blindness extends to token-factorized log loss, but the full-mask recovery argument does not because a factorized full-mask loss scores only coordinate marginals (\Cref{app:factorized}). Whether richer masked-block heads can recover these missing joint constraints at scale remains open. Continuous reconstruction and latent-regression objectives raise the analogous question of which discrepancies over the data or latent distribution they control.

\paragraph{Data geometry and empirical estimation.}
The mode-pinning assumption is deliberately schedule-independent and can be restrictive when mode information is unevenly distributed across coordinates. By contrast, \Cref{thm:mode-sensitivity} depends directly on residual uncertainty under the chosen schedule, suggesting an average-case route to broader data geometries. Beyond two modes, sensitivity is directional and requires control of posterior covariance rather than scalar residual uncertainty (\Cref{sec:multimode-sensitivity}). Estimating this geometry in natural data is also nontrivial: our fitted uncertainty proxies are better described by power-function fits over the observed range, but do not determine the asymptotic Bayes rate. Stronger empirical conclusions will require better posterior estimators and broader semantic partitions.

\paragraph{Model class and optimization.}
The theory is formulated over distributions rather than parameterized models. Restricted model classes may exclude blind alternatives or fail to represent the true distribution, while our optimization experiments hold the within-mode distributions fixed to isolate mode-weight curvature. Extending the analysis to jointly learned components, approximation error, and realistic optimizer dynamics is needed to connect objective-level identifiability to end-to-end pretraining.

\end{document}